\documentclass{article}
\usepackage{arxiv}
\usepackage[utf8]{inputenc}        
\usepackage[T1]{fontenc}           
\usepackage{algorithmic}
\usepackage{algorithm}
\usepackage{array}
\usepackage{textcomp}
\usepackage{stfloats}
\usepackage{url}                    
\usepackage{booktabs}              
\usepackage{amsfonts}
\usepackage{amsthm}                 
\usepackage{amsmath,amssymb}
\usepackage{nicefrac}               
\usepackage{microtype}             
\usepackage{xcolor} 
\usepackage{mathtools}
\usepackage[export]{adjustbox}
\usepackage{enumitem}
\usepackage{chngcntr}
\usepackage{apptools}
\usepackage{aliascnt}
\usepackage{sidecap}
\newtheorem{theorem}{\bf \emph{Theorem}}
\newtheorem{lemma}{\bf \emph{Lemma}}

\newtheorem{corollary}{Corollary}

\newtheorem{remark}{Remark}

\newcommand{\vertiii}[1]{\left\vert\kern-0.25ex\left\vert\kern-0.25ex\left\vert #1 \right\vert\kern-0.25ex\right\vert\kern-0.25ex\right\vert}
\newcommand{\geqsim}{\underset{\smash{\raisebox{0.6ex}{$\scriptstyle\sim$}\!}}{\succ}}
\newcommand{\leqsim}{\underset{\smash{\raisebox{0.6ex}{$\scriptstyle\sim$}\!}}{\prec}}

\DeclareMathOperator*{\argmin}{arg\,min}
\DeclareMathOperator{\Tr}{Tr}
\DeclareMathOperator{\im}{i}
\newcommand{\transpose}{\mathsf{T}}
\DeclareMathOperator*{\mvec}{\text{vec}}
\DeclareMathOperator{\ellinf}{\ell_{\infty}}

\newcommand*\xbar[1]{%
  \hbox{%
    \vbox{%
      \hrule height 0.5pt 
      \kern0.4ex%         
      \hbox{%
        \kern-0.15em%      
        \ensuremath{#1}%
        \kern-0.15em%     
      }%
    }%
  }%
} 

\newaliascnt{appendixlemma}{lemma}
\newtheorem{appendixlemma}[appendixlemma]{Lemma}
\aliascntresetthe{appendixlemma}

\newaliascnt{appendixtheorem}{theorem}
\newtheorem{appendixtheorem}[appendixtheorem]{Theorem}
\aliascntresetthe{appendixtheorem}

\newaliascnt{appendixcorollary}{corollary}
\newtheorem{appendixcorollary}[appendixcorollary]{Corollary}
\aliascntresetthe{appendixcorollary}

\usepackage[numbers,sort&compress]{natbib}
\usepackage[sectionbib]{bibunits} 
\defaultbibliographystyle{IEEEtran}
\usepackage[colorlinks=true, linkcolor=blue, citecolor=magenta, urlcolor=blue]{hyperref} 
\usepackage{etoolbox} 

\AtBeginEnvironment{thebibliography}{\small}
\usepackage{authblk}

\begin{document}

\title{Learning Networks from Wide-Sense Stationary Stochastic Processes} 
\date{}

\author[1]{Anirudh Rayas\thanks{Email: ahrayas@asu.edu}}
\author[1]{Jiajun Cheng\thanks{Email: ccheng@asu.edu}}
\author[2]{Rajasekhar Anguluri\thanks{Email: rajangul@umbc.edu}}
\author[3]{Deepjyoti Deka\thanks{Email: deepj87@mit.edu}}
\author[1]{Gautam Dasarathy\thanks{Email: gautamd@asu.edu}}

\affil[1]{\textit{Arizona State University}}
\affil[2]{\textit{University of Maryland, Baltimore County}}
\affil[3]{\textit{MIT Energy Initiative}}
\maketitle

\begin{abstract}
Complex networked systems driven by latent inputs are common in fields like neuroscience, finance, and engineering. A key inference problem here is to learn edge connectivity from node outputs (potentials). We focus on systems governed by steady-state linear conservation laws: $X_t = {L^{\ast}}Y_{t}$, where $X_t, Y_t \in \mathbb{R}^p$ denote inputs and potentials, respectively, and the sparsity pattern of the $p \times p$ Laplacian $L^{\ast}$ encodes the edge structure. Assuming $X_t$ to be a wide-sense stationary stochastic process with a known spectral density matrix, we learn the support of $L^{\ast}$ from temporally correlated samples of $Y_t$ via an $\ell_1$-regularized Whittle's maximum likelihood estimator (MLE). The regularization is particularly useful for learning large-scale networks in the high-dimensional setting where the network size $p$ significantly exceeds the number of samples $n$.

We show that the MLE problem is strictly convex, admitting a unique solution. Under a novel mutual incoherence condition and certain sufficient conditions on $(n, p, d)$, we show that the ML estimate recovers the sparsity pattern of $L^\ast$ with high probability, where $d$ is the maximum degree of the graph underlying $L^{\ast}$. We provide recovery guarantees for $L^\ast$ in element-wise maximum, Frobenius, and operator norms. Finally, we complement our theoretical results with several simulation studies on synthetic and benchmark datasets, including engineered systems (power and water networks), and real-world datasets from neural systems (such as the human brain).
\end{abstract}


\begin{bibunit}[IEEEtran]
\section{Introduction}\label{section: Introduction}
 Complex networked systems, composed of nodes and edges that connect them are commonly used to model real-world systems in fields such as neuroscience, engineering, climate, and finance~\cite{strogatz2001exploring, boccaletti2006complex}. We study networks governed by conservation laws that control edge flows; examples include current in electrical grids, fluids in pipelines, and traffic in transportation systems~\cite{van2017modeling, bressan2014flows}. In neuroscience, there is growing interest in identifying and understanding conservation laws~\cite{voss2014searching, podobnik2017biological}.

Networked systems driven by latent inputs (i.e., nodal injections) generate edge flows that are proportional to differences in node potentials. For example, in electrical networks, nodal current injections induce current flows that are proportional to potential differences between nodes. The overall dynamics of these edge flows are governed by conservation laws. Formally, for a network of size $p$, these dynamics are described by the balance equation $X = L^{\ast} Y$, where $L^{\ast} \in \mathbb{R}^{p \times p}$ is a weighted symmetric Laplacian matrix \cite{chung1997spectral}. The off-diagonal entries of $L^{\ast}$ capture the edge connectivity structure of the network. Vectors $X, Y \in \mathbb{R}^{p}$ represent nodal injections and potentials respectively, and in this paper, we treat them as random vectors. Further details on the balance equation are in Section \ref{section: Problem Setup}.

In various practical situations, the network's connectivity is typically not known and needs to be estimated for modeling, management, and control tasks. This involves determining the non-zero elements of the associated Laplacian matrix $L^{\ast}$. Previous methods such as \cite{shafipour2017network} estimate $L^{\ast}$ given observations of node injection-potential pairs $\{X, Y\}$ by minimizing an appropriate least squares objective. Such methods critically rely on the ability to observe both injections and potentials simultaneously. However, in various scenarios {\em node injections are often unobservable}. For instance, in financial or brain networks,  nodal injections correspond to economic shocks or unknown stimuli, and these are not observable by the measurement system in place. In these settings, the goal is to estimate $L^{\ast}$ with only samples of $Y$. Indeed, this problem is ill-posed as multiple solutions of $X$ and $L^{\ast}$ can satisfy the equation $X = L^{\ast} Y$. To address the ill-posedness, we assume we have access to some information about the distribution of $X$. The challenge of estimating $L^{\ast}$ from $Y$ under such assumptions have been previously studied in \cite{rayas_learning_2022, DekaTSG2020, anguluri2021grid}.

This line of work relies on the observations of the potentials being independent and identically distributed (i.i.d.). When temporal dependencies exist in the data, such methods are insufficient. In this paper, we adopt a more realistic data model and suppose that the nodal injections ($X_t$)  and potentials ($Y_t$) are wide-sense stationary processes (WSS). This generalization allows for a more flexible framework for network learning while posing some interesting technical challenges. Before we outline our major contributions, we will first state the problem more formally and outline the challenges it presents. 

\smallskip

{\bf Structure learning problem: }\emph{Given finite samples of node potentials $\{Y_{t}\}_{t=1}^{n}$ and assuming the node injections $X_t$ are generated from a WSS process with known spectral density matrix, the goal is to recover the matrix $L^{\ast} \in \mathbb{R}^{p \times p}$ such that the estimate $\widehat{L}$ approximately satisfies the balance equation $X_{t} \approx \widehat{L}Y_{t}$.}
\smallskip

The structure learning problem stated above assumes that the spectral density matrix for the latent process $X_t$ is known. As discussed earlier, estimating a sparse matrix $L^\ast$ from observations $\{Y_t\}_{t=1}^n$ alone is fundamentally ill-posed (see Remark \ref{rem: unknown spectral density} for further discussion).

A common approach in related work is to assume access to samples of the latent process $X_t$ \cite{shafipour2017network, segarra2018network}. In such a scenario, the spectral density matrix of $X_{t}$ can be estimated and subsequently $L^{\ast}$. However, access to samples from $X_{t}$ is unreasonable in many 
domains such as neuroscience, finance, and biology, where $X_t$ represents unobservable external inputs (e.g., latent external stimuli or economic shocks). An alternative assumption used in latent factor and structural equation models (SEMs) is to assume that the spectral density of $X_t$ is diagonal \cite{park2021learning, pruttiakaravanich2020convex}. However, this assumption is overly restrictive, as real-world exogenous inputs typically exhibit temporal and cross-sectional correlation \cite{doddi2022efficient}.

To address these limitations, we assume access to the full spectral density matrix of $X_t$, without imposing diagonality. This standard assumption \cite{shafipour2021identifying, mateos2019connecting} accommodates correlated latent inputs while still ensuring identifiability of $L^\ast$.

Its practical relevance is illustrated in two scenarios. In social networks, $Y_t$ may represent individuals' opinions and $X_t$ their latent beliefs. Though $X_t$ is unobserved, its second-order statistics can be modeled by exploiting homophily (i.e., individuals with similar attributes hold correlated beliefs) \cite{mcpherson2001birds}. In financial networks, $Y_t$ reflects stock prices driven by investor activity $X_t$, which are typically unobservable due to privacy concerns. However, many companies release second-order statistical summary information $\mathbb{E}[X_t X_t^\transpose]$ \cite{shen2017tensor}. 

Although the structure learning problem can be addressed through a two-step process---first estimating the spectral density of $Y_t$ from $\{Y_t\}_{t=1}^n$, and then estimating $L^{\ast}$ from the spectral density of $X_t$---this approach is statistically inefficient, even when $Y_t$ is i.i.d., this is elaborated in Remark \ref{rem: Limitation of Spectral precision estimation} of \cite{rayas_learning_2022}. To overcome these limitations, we propose a novel single-step estimator for $L^{\ast}$ that integrates finite time-series data with constraints imposed by conservation laws. Our method also ensures consistent estimation of $L^{\ast}$ in the high-dimensional setting where the number of samples $n$ is significantly smaller than the network size $p$ (i.e., $n \ll p$). This requires that $L^\ast$ is sparse, which is natural in all of our motivating examples: power grids, social networks, and brain connectivity graphs are inherently sparse, with nodes connected to only a small subset of others. 
 We now provide a high-level overview of our methodology.

Suppose that $\{X_t\}_{t\in \mathbb{Z}}$ is a WSS process with a complex-valued power spectral density matrix $f_X(\omega)$ with $\omega \in [-\pi, \pi]$ (see \eqref{eq: density X} for a formal definition). The conservation law dictates the spectral density $f_Y(\omega)$ of $\{Y_t\}_{t\in \mathbb{Z}}$ to satisfy $f_X(\omega)=L^{\ast}f_Y(\omega)(L^{\ast})^\transpose$. Given samples from the node potential process $\{Y_t\}_{t=1}^n$ and assuming that $f_X(\omega)$ is known (this is all we know about $X$), consider the optimization problem: 
\begin{equation}\label{eq: intro_problem}
\begin{aligned}
& \underset{L\in \mathbb{R}^{p\times p}}{\text{maximize}}
& & \mathcal{L}[\{Y_t\}_{t=1}^n;f_X(\omega)]+\lambda_{n}\|L\|_{1} \\
& \text{subject to}
& & f_X(\omega)=Lf_Y(\omega)L^\transpose, \quad \omega \in [-\pi, \pi], 
\end{aligned}
\end{equation}
where $\mathcal{L}[\cdot]$ is an appropriate log-likelihood that measures the fit to observed data, and $\lambda_n\geq 0$ is a regularization parameter. The $\ell_1$-norm $\|\cdot\|_1$ (which is the entry-wise absolute sum) helps promote sparsity in our estimate of $L^\ast$. Full details of \eqref{eq: intro_problem} are in Section \ref{section: Problem Setup}. While such optimization problems that target sparse matrix estimation have received considerable attention in the literature (see Sections~\ref{sec: connections} and \ref{sec: lit-rev} for a brief overview), \eqref{eq: intro_problem} presents some unique challenges:
\begin{enumerate}[label=\roman*)] 
\item $\{Y_t\}_{t=1}^n$ is not i.i.d., making standard sample covariance matrix style analyses inapplicable; \item it involves a continuum of constraints since $\omega \in [-\pi,\pi]$, rendering \eqref{eq: intro_problem} an infinite-dimensional optimization problem; and \item the constraint is non-convex for arbitrary matrices $L$, even when considering the symmetry of the Laplacian matrix. 
\end{enumerate}

Although a line of work \cite{deb2024regularized, dallakyan2022time, basu2015regularized, doddi2021learning} addresses challenges of the form (\romannumeral 1) and (\romannumeral 2) separately in the context of learning Gaussian graphical models from time-series data, and \cite{rayas_learning_2022} tackles challenge (\romannumeral 3), no prior work, to the best of our knowledge addresses all three challenges simultaneously. The goal of this paper is to show that despite these challenges, the optimizer of \eqref{eq: intro_problem} captures the sparsity pattern of $L^{\ast}$ with high probability. Thus, the optimizer of \eqref{eq: intro_problem} is the estimator we seek to recover the sparse matrix $L^{\ast}$. This problem formulation is motivated by several applications where it plays a natural role; here we briefly outline two.

\textit{ 1) Topology learning in power distribution networks}: Knowledge of network topology (or structure) enables better fault detection, efficient resource allocation, and better integration of decentralized energy resources, ensuring reliable operation of the power system. However, system operators may lack access to real-time topology information and use nodal voltages or current injections to learn the network topology. A balance equation of the form $X_t = {L^{\ast}}Y_t$, where $L^{\ast}$ is the network admittance matrix and injected currents $X_t$ modeled by a WSS process, has been considered in this context \cite{doddi2019exact}. 

\smallskip 
 
\textit{ 2) Learning sensor to source mapping in the human brain}: Learning the mapping from source signals to EEG electrodes is crucial for analyzing brain connections. Many studies \cite{ranciati2021fused,monti2014estimating} suggest a model of the form in \eqref{eq: flow systems}. Specifically, the Laplacian matrix plays the role of lead-field matrix and the potentials $Y_{t}$ are the EEG signals. The injections $X_t$ model the latent source signals and are thought to be generated by a vector auto-regressive process (VAR$(m)$): $X_t=\sum_{k=1}^mA_kx_{t-k}+\epsilon_t$, where $\epsilon_t$ could be non-Gaussian; and the integer $m$ and matrices $A_k$ could be known or unknown. Thus, learning the source mapping involves learning $L^{\ast}$ from WSS data.

\subsection{Main contributions} \label{subsec: main contributions}

\smallskip  
\textit{1) A novel convex estimator:} We propose an $\ell_1$-regularized log-likelihood estimator of the form \eqref{eq: intro_problem} to estimate $L^{\ast}$ from finite samples of WSS data $\{Y_t\}_{t=1}^n$. This estimator builds on the Whittle log-likelihood approximation (details in Section \ref{subsec: Modified Whittle's likelihood approximation}). Our first theoretical result establishes that the proposed $\ell_1$-regularized estimator is convex in $L$ and under standard conditions, admits a unique minimum even in the high-dimensional regime ($n \ll p$).

Since the Whittle likelihood is closely tied to the likelihood of Gaussian WSS processes, our estimator maximizes an approximate Gaussian likelihood. However, the estimator remains meaningful even for non-Gaussian injections $\{X_t\}_{t \in \mathbb{Z}}$, including stationary linear processes with sub-exponential or finite fourth-moment error distributions (see the remark on Bregman divergence in Section \ref{subsec: Modified Whittle's likelihood approximation}).

\smallskip  
\textit{2) Sample complexity and estimation consistency:} We provide sufficient conditions on the sample size $n$ of the data $\{Y_t\}_{t=1}^n$ for the estimator to achieve two key properties: \emph{sparsistency}, ensuring the recovery of the sparsity pattern of $L^{\ast}$, and \emph{norm consistency}, providing error bounds in terms of element-wise maximum, Frobenius, and operator norms. Pivotal to our analysis is a novel irrepresentability-like condition on $L^{\ast}$, inspired by similar conditions commonly used in high-dimensional statistics \cite{wainwright2009sharp, van2008high}. The sample complexity results are derived for both Gaussian and linear non-Gaussian WSS processes (see Theorem \ref{thm: Gaussian time series} and \ref{thm: Linear Process}).
 
 \smallskip 
{\textit{3) Experimental validation:}} We validate our theoretical results with extensive numerical experiments using synthetic and quasi-synthetic data from many benchmark networked systems, as well as a real-world dataset involving the brain network (see Section \ref{section: Experiments}).

\subsection{Related work}\label{sec: lit-rev}

\smallskip 

\subsubsection{Structure learning in Gaussian graphical models (GGMs)} \label{subsubsec: structure learning of GGM}The graph underlying a GGM can be inferred from the sparsity pattern of the inverse covariance matrix, and numerous papers have focused on learning this pattern from i.i.d. data (see \cite{drton2017structure} for an overview). Pioneering works like ~\cite{yuan2007model, ravikumar2011high} have developed key theoretical concepts for analyzing $\ell_1$-regularized likelihood estimators, and our analysis builds on these concepts. Other works like ~\cite{dallakyan2023fused, chang2010estimation} focus on learning Cholesky factors of the inverse covariance matrix, but they lack theoretical guarantees. Survey papers like \cite{tsai2022joint} provide a comprehensive overview of estimators for GGMs in various scenarios, including dynamic and grouped networks, while \cite{chen2024estimation} presents detailed analyses of theoretical frameworks and sample complexity results for these models. However, these approaches face two significant limitations in our context. First, they are primarily designed for i.i.d. data, whereas the problem we address involves time-series data. Second, these methods aim to estimate the inverse covariance matrix, whereas our focus is to estimate the Laplacian $L^{\ast}$ directly, bypassing the need to first estimate the inverse covariance matrix.

\subsubsection{Graph signal processing (GSP)} Recent research in GSP studied sparse inverse covariance estimation problems in GGMs by imposing Laplacian constraints. Both the regularized likelihood and spectral template-based (i.e., using eigenvectors of the sample covariance matrix) techniques are used to learn the Laplacian-constrained inverse covariance matrix \cite{ying2020does, kumar2019structured, ying2023network}. However, many papers in this area focus only on estimation consistency or algorithmic convergence, but not on sample complexity. In our problem, the inverse covariance (or spectral density) matrix is represented as a quadratic matrix equation involving products of Laplacian matrices (see \eqref{eq: intro_problem}), making existing methods in the cited works unsuitable for direct application. In addition, we provide sample complexity guarantees and establish precise rates of convergence for our proposed estimator.

\subsubsection{Learning network structure from WSS process} Dahlhaus \cite{dahlhaus2000graphical} showed that the sparsity pattern of the inverse spectral density (ISD) matrix represents the structure of the graphical model for a Gaussian WSS. Subsequently, many papers (see e.g., \cite{deb2024regularized,baek2021local}) have focused on estimating a sparse ISD matrix. Finally, a few more (see \cite{dallakyan2022time, basu2015regularized, doddi2021learning}) have focused on estimating parameter matrices of latent models (e.g., VAR or state-space) generating the ISD matrix. Our research falls into the latter category, with a parameter matrix that is a Laplacian of a conservation law. However, directly applying these methods often leads to a two-stage approach: first estimating the parameter matrix, followed by a refinement step to identify non-zero entries in $L^{\ast}$. In contrast, our estimator of the form in \eqref{eq: intro_problem} directly estimates the Laplacian matrix $L^{\ast}$, thus avoiding the statistical inefficiencies inherent in the two-stage approach (see Section \ref{subsubsec: structure learning of GGM}). Related streams of work have addressed latent-variable autoregressive graphical models using sparse + low-rank decompositions of the inverse spectral density \cite{zorzi2015ar, zorzi2019empirical, crescente2020learning}, ARMA factor models using diagonal + low-rank structures \cite{falconi2023robust,zorzi2024identification}, and sparse reciprocal graphical models that impose block-circulant patterns \cite{alpago2018identification}.

While these approaches provide valuable insights, our problem setting is fundamentally different. We focus on estimating a general sparse Laplacian matrix associated with a conservation law constraint, using a single-step likelihood-based approach in the frequency domain. We do not assume latent-variable factorizations or additional structural constraints such as low-rankness or block-circulant structures. Importantly, we provide theoretical guarantees on the sample complexity required to achieve support recovery and to bound estimation error in matrix norms for this general setting. To the best of our knowledge, these guarantees have not been established in the aforementioned literature.

\subsubsection{Electric power networks} While there are many motivating examples for this framework, the authors were specifically motivated by the problem of topology learning in power networks. For i.i.d. data, works like~\cite{DekaTCNS2018, deka2020graphical} infer the sparsity pattern of the Laplacian (associated with a conservation law under linear power flow) by learning the inverse covariance of node potentials and applying algebraic rules. This approach requires minimum cycle length conditions on the network, which we do not need (see Remark \ref{rem: Limitation of Spectral precision estimation}). Survey papers like \cite{deka2023learning} provide a good overview of state-of-the-art methods, including the likelihood approaches in \cite{grotas2019power}.

We now contrast this work with a related paper by a subset of the authors~\cite{rayas_learning_2022}. First, the estimator in ~\cite{rayas_learning_2022} assumes i.i.d. Gaussian injections $X_t$, whereas the current work addresses non-i.i.d. $X_t$ and considers a broader class of Gaussian and non-Gaussian WSS processes; we outlined the unique challenges in the discussion following equation \eqref{eq: intro_problem}. Second, our analysis requires a comprehensive examination of Hermitian matrices in the optimization problem, which is more complex than dealing solely with symmetric matrices, as in \cite{rayas_learning_2022}. Third, we empirically validate the performance of our estimator, particularly regarding sample complexity and error consistency, across a wide range of networked systems, and compare it directly with the estimator proposed in \cite{rayas_learning_2022}.

\smallskip  

\textit{Notation}: Let $\mathbb{Z}, \mathbb{R}$, and $\mathbb{C}$ denote sets of integers, reals, and complex numbers, respectively. For sets $T_1, T_2 \subset [p]\times [p]$, denote by $A_{T_1T_2}$ the submatrix of $A$ with rows and columns indexed by $T_1$ and $T_2$. If $T_1=T_2$, we denote the submatrix by $A_{T_1}$. For a matrix $A=[A_{i,j}]$, $\|A\|_F$ and $\Vert A\Vert_{2}$ denote the Frobenius and the operator norm; $\Vert A\Vert_{\infty} \triangleq \max_{i,j}\vert A_{ij}\vert$ and $\Vert A\Vert_{1,\text{off}} = \sum_{i\neq j}\vert A_{ij}\vert$. The $\ellinf$-matrix norm of $A$ is defined as
    $\nu_{A} = \vertiii{A}_{\infty} \triangleq \max_{j=1,\ldots,p}\sum_{j=1}^{p}\vert A_{ij}\vert.$ 
We use $\mvec(A)$ to denote the $p^2$-vector formed by stacking the columns of $A$ and $\Gamma(A)=(I\otimes A)$ to denote the Kronecker product of $A$ with the identity matrix $I$. For two symmetric positive definite matrices $A_1$ and $A_2$, $A_1\succ A_2$ means $A_1-A_2$ is positive definite.  We define $\text{sign}(A_{ij}) = +1$ if $A_{ij}>0$ and $\text{sign}(A_{ij}) = -1$ if $A_{ij}<0$. For two-real valued functions $f(\cdot)$ and $g(\cdot)$, we write $f(n) = \mathcal{O}(g(n))$ if $f(n)\leq cg(n)$ and $f(n) = \Omega(g(n))$ if $f(n)\geq c^{\prime}g(n)$ for constants $c,c^{\prime}>0.$ 

\smallskip 
\noindent\textbf{Organization of the paper}: In Section \ref{section: Problem Setup}, we define the structure learning problem and propose the modified $\ell_{1}$-regularized Whittle likelihood estimator for learning a network structure from WSS data. Section \ref{section: Convex estimator} establishes the convexity of the proposed estimator and provides guarantees for support recovery and norm consistency for both Gaussian and non-Gaussian node injections $X_t$. In Section \ref{section: Experiments}, we evaluate the performance of our estimator on synthetic, benchmark, and real-world datasets. Section \ref{sec: connections} emphasizes the parallels that our structure learning framework shares by drawing connections to other learning problems in the literature. Finally, Section \ref{section: Discussion and Future Work} concludes with a summary and outlines future directions. Proofs of theoretical results and additional experimental details are provided in the supplementary material. Throughout, we use \emph{estimation} and \emph{learning} interchangeably, as well as \emph{network} and \emph{graph.}

\section{Preliminaries and Problem Setup}\label{section: Problem Setup}
For directed graph $\mathcal{G} = ([p], E)$, where the node set is defined as $[p] \triangleq \{1, 2, \ldots, p\}$ and the edge set is $E \subseteq [p] \times [p]$, let $\mathcal{D}$ denote the $p \times \left|E\right|$ \emph{incidence matrix}. Each column of \(\mathcal{D}\) corresponds to an edge $(i, j)$ and is populated with zeros except at the $i$-th and $j$-th positions, where it takes the values $-1$ and $+1$, respectively. Suppose $X\in \mathbb{R}^p$ denotes the vector of {\em node injections}. The basic \emph{conservation law} is given by: $\mathcal{D}f+X=0$, where $f\in \mathbb{R}^{\left| E \right|}$ is the vector of \emph{edge flows}. This law states that the sum of flows over the edges incident to a vertex equals the injected flow at that vertex. In other words, edge and injected flows are conserved. 

In physical systems, edge flows are determined by \emph{potentials} $Y\in \mathbb{R}^{p}$ at the vertices. Under natural linearity assumptions, the edge flow on the $(i,j)$-th edge is proportional to $Y_j - Y_i$. For all edges,
$f=-\mathcal{D}^\transpose Y$. Substituting this edge flow relation in the basic conservation law yields the {\em balance equation}:
\begin{align}\label{eq: flow systems}
    X-L^{\ast}Y = 0, 
\end{align}
 where $L^{\ast}\triangleq \mathcal{D}\mathcal{D}^\transpose$ is the $p\times p$ real-valued symmetric Laplacian matrix. A typical system satisfying \eqref{eq: flow systems} is an electrical network with unit resistances, where $Y$ represents voltage potentials, $f$  edge currents, and $X$ injected currents. For examples involving hydraulic, social, and transportation systems, see \cite{van2017modeling, bressan2014flows}.

\subsection{Structure learning problem}
The sparsity pattern (locations of zero and non-zero entries) of $L^{\ast}$ reflects the edge connectivity of the underlying network. Specifically, $(i,j)\in E$ if and only if $L^{\ast}_{ij}\ne 0$. Our goal is to learn the unknown edge set $E$ (or the sparsity pattern of $L^{\ast}$) from data collected at the nodes of the graph.

Let $\{X_{t}\}_{t\in\mathbb{Z}}$ be a zero-mean $p$-dimensional vector-valued WSS process, where, for each $t\in \mathbb{Z}$, $X_{t} = (X_{t1},\ldots,X_{tp})^{\transpose}\in \mathbb{R}^{p}$. The auto-covariance function of this process is $\Phi_{X}(l) \triangleq \mathbb{E}[X_{t}X_{t-l}^{\transpose}], \text{ for all } t \in \mathbb{Z}$ and $l\in \mathbb{Z}$ is the lag parameter. We assume that $\Phi_X(l)\succ 0$. Because $\{X_t\}_{t\in \mathbb{Z}}$ is WSS, it holds that $\|\Phi_{X}(l)\|_2<\infty$. Hence, the power spectral density (PSD) function of $\{X_{t}\}_{t\in\mathbb{Z}}$ exists and is defined via the discrete-time Fourier transform of $\Phi_{X}(l)$: 
\begin{align}\label{eq: density X}
    f_X(\omega) &\triangleq  \frac{1}{2\pi}\sum_{l=-\infty}^{\infty}\Phi_{X}(l)e^{-\im l\omega}, \quad \omega \in [-\pi, \pi], 
\end{align}
where $\im=\sqrt{-1}$ and $f_X(\omega)\in \mathbb{C}^{p\times p}$ is a Hermitian positive definite matrix. Let $\Theta_{X}(\omega)\!\triangleq\! f_{X}^{-1}(\omega)$ be the inverse PSD. 

Let $\{Y_{t}\}_{t\in\mathbb{Z}}$ be generated per the balance equation in \eqref{eq: flow systems}. We want to obtain a sparse estimate of $L^{\ast}$ using the finite time-series potential data $\{Y_t\}_{t=1}^n$ and only the nodal injection's inverse PSD matrix $\Theta_X(\omega)$; see Remark \ref{rem: unknown spectral density}. We emphasize that our processes need not be Gaussian. A major challenge in developing maximum-likelihood parameter estimates from time-series data is obtaining tractable likelihood formulas. Whittle \cite{whittle1953estimation} developed a good approximation for the Gaussian case, and the later work extended this approach to other cases. Following \cite{deb2024regularized}, we provide likelihood approximations for $\{Y_t\}_{t=1}^n$.

\subsection{Modified Whittle's likelihood approximation}\label{subsec: Modified Whittle's likelihood approximation}

Suppose that $L^{\ast}$ is invertible (see Remark \ref{rem: invertibility assumption}), the equation in \eqref{eq: flow systems} simplifies to $Y_t = {(L^{\ast})}^{-1}X_t$. Due to this linear relationship, $\{Y_t\}_{t\in \mathbb{Z}}$ is also a WSS process with the auto-covariance matrix: 
\begin{align*}
    \Phi_{Y}(l) \triangleq \mathbb{E}[Y_{t},Y_{t-l}^\transpose]={(L^{\ast})}^{-1}\Phi_{X}(l){(L^{\ast})}^{-1}, 
\end{align*}

and the PSD matrix:  
\begin{align}\label{eq: density Y}
  \hspace{-2.0mm}  f_Y(\omega) \triangleq  \frac{1}{2\pi}\sum_{l=-\infty}^{\infty}\Phi_{Y}(l)e^{-\im l\omega} = {(L^{\ast})}^{-1}f_X(\omega){(L^{\ast})}^{-1}, 
\end{align}
where $\omega \in [-\pi, \pi]$. Finally, define the inverse PSD matrix: 
\begin{align}\label{eq: defn of spectral density of Y}
    \Theta_{Y}(\omega)\triangleq f^{-1}_Y(\omega)={L^{\ast}}\Theta_{X}(\omega){L^{\ast}}.
\end{align}

For now assume that $\{Y_t\}_{t\in \mathbb{Z}}$ is a WSS Gaussian process. We will relax this assumption later. Define $\omega_{j} = 2\pi j/n$ and denote $\mathcal{F}_{n} = \{\omega_0,\ldots,\omega_{n-1}\}$ to be the set of Fourier frequencies. The discrete Fourier transform (DFT) of $\{Y_t\}_{t=1}^n$ is then given by $d_{j}= \frac{1}{\sqrt{n}}\sum_{t=1}^{n}Y_{t}e^{-\im t\omega_{j}} \in \mathbb{C}^{p}$. Observe that DFT is a linear transformation; hence, $d_j$s are complex-valued multivariate Gaussian with the inverse covariance $\Theta_Y(\omega_j)\in \mathbb{C}^{p\times p}$.

The log-likelihood of the finite-time series data $\{Y_t\}_{t=1}^n$ as per the \emph{Whittle approximation} \cite{whittle1953estimation} (see Remark~\ref{rem: freq domain formulation} for justification and benefits of the frequency-domain formulation) is given by
\begin{align}\label{eq: approximate whittle likelihood}
    \frac{1}{2} \sum_{j\in \mathcal{F}_{n}}\left[\log\det(\Theta_{Y}(\omega_j))-\Tr(\Theta_{Y}(\omega_j)d_{j}d^{\dagger}_{j})\right], 
\end{align}
where $\dagger$ is the conjugate transpose and we dropped the constants in the approximation that do not depend on $L^{\ast}$. Expression in \eqref{eq: approximate whittle likelihood} resembles the log-likelihood formula for i.i.d. $\{Y_t\}_{t=1}^n$. Thus, we can view $\hat{f}_j\triangleq \hat{f}(\omega_j)=d_jd_j^\dagger$ as playing the role of sample covariance for the spectral density matrix $f_Y(\omega_j)$. 

The log-likelihood in \eqref{eq: approximate whittle likelihood} requires modifications to serve as a suitable objective function in $\mathcal{L}[\cdot]$ in \eqref{eq: intro_problem}. First, for $\widehat{L}$ to have better statistical performance, the spectral density estimate \(\hat{f}_j\), which has a high variance (see \cite[Proposition 10.3.2]{brockwell2009time}), needs to be smoothed. 

We use the \emph{averaged periodogram} \cite{brockwell2009time}: 
\begin{align}\label{eq: Averaged periodogram}
    P_j\triangleq P(\omega_j)=\frac{1}{2\pi(2m+1)}\sum_{\vert k\vert\leq m}d(\omega_{j+k})d^\dagger(\omega_{j+k}), 
\end{align}
where $\omega_j\in \mathcal{F}_n$ and $P_j \in \mathbb{C}^{p\times p}$. The bandwidth $m$ regulates the bias and variance of $P_j$ \cite{brockwell2009time}, which in turn impacts the estimation consistency results for $L^{\ast}$ in Theorem \ref{thm: Gaussian time series} and \ref{thm: Linear Process}. 
For a theoretical discussion on {\color{black} periodograms consult \cite{brockwell2009time}}.

Second, substituting $P_j$ given by \eqref{eq: Averaged periodogram} in \eqref{eq: approximate whittle likelihood} results in an approximate likelihood that is analytically intractable because of the double summation that appears within the $\Tr[\cdot]$ operator. We address this by further approximating the likelihood in \eqref{eq: approximate whittle likelihood} as suggested by \cite{deb2024regularized}. The idea here is to consider the likelihood in the neighborhood of a frequency $\omega_j$, where $j\in \mathcal{F}_n$. Thus, for $j-m\leq l \leq j+m$, a reasonable likelihood near $\omega_j$ is 
\begin{align}\label{eq: second approximate}
   &\frac{1}{2}\sum_{l=j-m}^{j+m}\left[\log\det(\Theta_{Y}(\omega_l))-\Tr(\Theta_{Y}(\omega_l)d_{l}d^{\dagger}_{l})\right]. 
\end{align}

This local likelihood could be simplified by assuming $\Theta_X(\omega)$ is a smooth function of $\omega \in [-\pi,\pi]$. Thus, $\Theta_X(\omega_l)$ is constant for the frequencies neighboring $\omega_j$. This smoothness assumption along with the relationship in \eqref{eq: defn of spectral density of Y} implies $\Theta_Y(\omega_j)=\Theta_Y(\omega_l)$, for all $j-m\leq l \leq j+m$. Consequently, \eqref{eq: second approximate} simplifies to 
%\rajmargin{Should we supplement the smoothness assumption with any power network related motivation?}
\begin{align}
    \frac{(2m+1)}{2}\left[\log\det(\Theta_Y(\omega_j))-\text{Tr}(\Theta_Y(\omega_j)P_j)\right],\label{eq: whittle2} 
\end{align}
which we call the modified Whittle's approximate likelihood for the Gaussian node potentials $\{Y_t\}_{t=1}^n$. 
 
The modified (per frequency) likelihood in \eqref{eq: whittle2} is valid even if $\{Y_t\}_{t=1}^n$ is non-Gaussian. This is because as $n\to \infty$, the DFT vectors $d_j$ converge to a complex-valued multivariate Gaussian with inverse covariance $\Theta_Y(\omega_j)$, per \cite[Propositions 11.7.4 and 11.7.3]{brockwell2009time}. Thus, the likelihood either in \eqref{eq: approximate whittle likelihood} or in \eqref{eq: whittle2} remains applicable for non-Gaussian $\{Y_t\}_{t\in Z}$. However, this standard justification relies on $n$ being large and might not be appropriate for smaller $n$. A more robust theoretical justification can be given using Bregman divergences, which we discuss next.

The Bregman divergence between $p \times p$ Hermitian matrices $A$ and $B$ is $D_{\phi}(A;B)\triangleq \phi(A)-\phi(B)-\langle \nabla\phi(B), A-B\rangle$, where  
$\phi(\cdot)$ is a differentiable, strictly convex function mapping matrices to reals \cite{ravikumar2011high, dhillon2008matrix}. The log-det Bregman divergence is a special case for $\phi(\cdot) = \log\det[\cdot]$. Thus, for $A\succ 0$ and $B\succ0$ (either real or complex-valued matrices), we have, 
\begin{align*}
  D_{\phi}(A;B) &= -\log\det(A)+\log\det(B) +\Tr(B^{-1}(A-B)).
\end{align*}
Let $A=\Theta_{Y}(\omega)$; and $B=\Theta^{\ast}_{Y}(\omega)$ be the true inverse spectral density matrix with $f^{\ast}_{Y}={\Theta_{Y}^{\ast}}^{-1}$. We drop terms that do not depend on $\Theta_Y(\omega)$ in $D_{\phi}(A;B)$ and note that $D_{\phi}(A;B)$ is proportional to $-\log|\Theta_{Y}(\omega)|+\Tr(f^{\ast}_{Y}(\omega)\Theta_{Y}(\omega))$. Finally, replacing $f^{\ast}_{Y}(\omega)$ in this expression with the periodogram estimator $P(\omega)$ gives us the negative of the modified likelihood given in \eqref{eq: whittle2}.

In view of the foregoing discussion, we see that our modified approximate likelihood function in \eqref{eq: whittle2} is a good candidate for the loss function $\mathcal{L}[\cdot]$ in \eqref{eq: intro_problem} even for non-Gaussian $\{Y_t\}_{t\in \mathbb{Z}}$.

\begin{remark}(Inverse of $L^{\ast}$)\label{rem: invertibility assumption}. The invertibility assumption is necessary for identifying $L^\ast$ from the time series data $\{Y_t\}_{t=1}^n$. However, $L^\ast$ is not invertible because it has single or multiple zero eigenvalues. A workaround is to use the reduced-order Laplacian, which is obtained by removing $k$ rows and columns from $L^\ast$ (see \cite{dorfler2012kron}), or to perturb the diagonal of $L^\ast$ with a small positive quantity. In power networks, this perturbation corresponds to adding shunt impedance (self-loops in graph theory) at the nodes. We assume that one of the approaches is in place and that $L^{\ast}$ is invertible. 
\end{remark}   
\begin{remark}\label{rem: freq domain formulation}(Frequency-domain approach): Frequency-domain methods are increasingly used for multivariate time series due to their computational efficiency \cite{sun2018large,fiecas2019spectral,dallakyan2022time,baek2023local,basu2023graphical,deb2024regularized,krampe2025frequency}. For a stationary univariate process with $n$ samples, the Whittle approximation reduces the $O(n^3)$ cost of likelihood evaluation to $O(n \log n)$ via fast Fourier transforms \cite{hurvich2002whittle}. In the multivariate case, with $n$ samples and a $p \times p$ spectral density matrix, this computational advantage becomes even more critical, thus justifying the choice of a frequency-domain formulation.
    
\end{remark}

\section{Convexity and Statistical Guarantees}\label{section: Convex estimator}

Using the modified Whittle's approximate likelihood in \eqref{eq: whittle2}, we first introduce our $\ell_1$-regularized estimator as a convex optimization problem. We then present our main results that theoretically characterize the performance of this estimator when $\{X_{t}\}_{t\in \mathbb{Z}}$ is Gaussian and more generally a linear process. Complete proofs are in the Appendix. 

The invertibility assumption (see Remark \ref{rem: invertibility assumption}) and the diagonal dominance property of $L^{\ast}$ imply that $L^{\ast}$ is a symmetric positive definite matrix. Recall that $f^{-1}(\omega)=\Theta(\omega)$, for $\omega \in [-\pi,\pi]$. Given these conditions and the likelihood formula in \eqref{eq: whittle2}, the optimization problem in \eqref{eq: intro_problem} modifies to:
\begin{align}\label{eq: whittle likelihood_SB}
 \hat{L}_j & =\underset{L\succ 0}{\operatorname{arg\, min}} \,\,\Tr(\Theta_Y(\omega_j) P_{j})\!-\!\log \det(\Theta_Y(\omega_j))\!+\!\lambda_{n}\|L\|_{1,\text{off}}\nonumber \\
 & \quad \,\, \text{subject to}\,\,\,\, \Theta_Y(\omega_j)=L\Theta_X(\omega_j)L^\transpose,
\end{align}
where $j=\{0,\ldots,n-1\}$, $\lambda_n>0$, and $\Vert L\Vert_{1,\text{off}}=\sum_{i\neq j}\vert L_{ij}\vert$ is the $\ell_1$-norm (see Remark \ref{rem: Choice behind l_1-regularization} for more discussion on this choice) applied to the off-diagonals of $L\in \mathbb{R}^{p\times p}$. Note that the constraint in \eqref{eq: intro_problem} is stated in terms of the density matrix $f(\omega)$. But note that the constraint in \eqref{eq: whittle likelihood_SB} is in terms of the inverse matrix $f^{-1}(\omega)=\Theta(\omega)$. 

Let $D_{j}\in \mathbb{C}^{p \times p}$ be the unique Hermitian positive-definite square root
of $\Theta_{X}(\omega_{j})$ satisfying $D^2_{j} = \Theta_{X}(\omega_{j})$. Then substituting $\Theta_Y(\omega_{j})=LD^{2}_{j}L^{\transpose}$ and $L=L^{\transpose}$ in the cost function of \eqref{eq: whittle likelihood_SB}, followed by an application of the cyclic property of the trace, results in the following unconstrained estimator: 
\begin{align}\label{eq: l_1 regularized Whittle likelihood}
    \hspace{-2.5mm}\widehat{L}_j \!= \!\argmin_{L\succ 0}\Tr(D_{j}LP_{j}LD_{j})\!-\!\log \det(L^{2})\!+\!\lambda_{n}\|L\|_{1,\text{off}}.
\end{align}

We dropped constants that bear no effect on the optimization problem. In summary, for $\omega_j\in \mathcal{F}_n$, we propose a point-wise estimator $\widehat{L}_j$ via \eqref{eq: l_1 regularized Whittle likelihood}. While the true Laplacian $L^{\ast}$ is fixed and does not vary with frequency, our estimator $\widehat{L}_j$ is defined at each $\omega_j$. Theorems~\ref{thm: Gaussian time series} and \ref{thm: Linear Process} show that $\widehat{L}_j$ satisfies the same statistical guarantees with respect to $L^{\ast}$ for all $\omega_j \in \mathcal{F}_n$. Therefore, any $\widehat{L}_j$ can be chosen as a candidate estimator for $L^{\ast}$. This per-frequency formulation aligns with recent methods such as \cite{fiecas2019spectral, krampe2025frequency, deb2024regularized}, which also estimate spectral quantities locally at each frequency, in contrast to approaches that penalize across all frequencies \cite{jung2015graphical, dallakyan2022time, baek2021thresholding}. Hereafter, we refer to $P_j$ and $\widehat{L}_{j}$ as $P$ and $\widehat{L}$, respectively, since our results hold for all $\omega_j \in \mathcal{F}_n$. 
Finally, we use $P_1 = \mathfrak{R}(P)$ and $P_{2} = \mathfrak{I}(P)$ to denote the real and imaginary parts of the periodogram $P$ and $\Psi_{1},\Psi_{2}$ to denote the real and imaginary parts of $D^{2}$ respectively. 

\smallskip 

The following lemma establishes two crucial properties of \eqref{eq: l_1 regularized Whittle likelihood}: (i) the objective function is strictly convex in $L$ and (ii) $\widehat{L}$ is unique. The proof of this lemma is in Appendix \ref{app: convexity and uniqueness}.

\begin{lemma}\label{lma: convexity of objective}
For any $\lambda_{n}\!>\!0$ and $L\!\succ\! 0$, if all the diagonals of the averaged periodogram $P_{ii}>0$, then (i) the $\ell_{1}$-regularized Whittle likelihood estimator in \eqref{eq: l_1 regularized Whittle likelihood} is strictly convex and (ii) $\widehat{L}$ in \eqref{eq: l_1 regularized Whittle likelihood} is the unique minima satisfying the sub-gradient condition $2\Psi_{1}\widehat{L}P_{1} - 2\Psi_{2}\widehat{L}P_{2}- 2\widehat{L}^{-1}\!+\!\lambda_{n}\widehat{Z}\!=\!0$, where $\widehat{Z}$ belong to the sub-gradient $\partial\Vert L \Vert_{1,\text{off}}$ evaluated at $\widehat{L}$.
\end{lemma}

Establishing strict convexity of the objective function in \eqref{eq: l_1 regularized Whittle likelihood} is non-trivial and crucial to derive sample complexity and estimation consistency results discussed in Section \ref{subsection: Supporting lemmas}. Furthermore, this strict convexity enforces the existence of unique minima even in the high-dimensional regime ($n\ll p$), where the Hessian of the objective function is rank deficient. The key ingredient in establishing such minima is the coercivity of the objective function (discussed later). The combination of convexity, coercivity, and separable property of the $\ell_1$-regularizer also facilitates the development of efficient coordinate descent algorithms, which we leave for future research.

 \begin{remark}(Identifiability of $L^{\ast}$)\label{rem: unknown spectral density} The matrix $L^{\ast}$ is identifiable under two conditions: (i) the spectral density matrix $\Phi_X$ or its inverse $\Theta_{X}$ is known, and (ii) $L^{\ast}$ is constrained to be symmetric and positive definite (PD). Under these assumptions, $L^{\ast}$ has a unique closed-form expression in terms of $\Phi_X$ and $\Phi_Y$, since the relation $\Phi_X = L^{\ast} \Phi_Y {L^{\ast}}^\top$ admits a unique PD factorization. However, identifiability fails when these assumptions are relaxed. Suppose $L^{\ast}$ is symmetric but not PD. Then, multiple symmetric square roots of $\Phi_X$ may exist, and therefore $L^{\ast}$ may not have a unique representation in terms of $\Phi_X$ and $\Phi_Y$, leading to a loss of identifiability. Now, if $L^{\ast}$ is non-symmetric, and $\Phi_{X}$ is diagonal, then $L^{\ast}$ is indistinguishable from $L^{\ast}U$ for any orthogonal matrix $U$. Lastly, if $\Phi_X$ is unknown, then multiple pairs of $L^{\ast}$ and $\Phi_{X}$ can yield the same $\Phi_Y$, and therefore $L^{\ast}$ is not identifiable.
\end{remark}

\begin{remark}(Advantage of directly estimating $L^{\ast}$)\label{rem: Limitation of Spectral precision estimation} The estimator in \eqref{eq: l_1 regularized Whittle likelihood} directly estimates $L^{\ast}$ subject to the constraint $\Theta_{Y} = L^{\ast}\Theta_{X}L^{\ast}$. In contrast, prior methods (see for e.g., \cite{deka2020graphical}) learn the network structure by first estimating the ISD matrix $\Theta_{Y}$ corresponding to $\{Y_t\}_{t=1}^n$ and then perform a post-processing step of applying algebraic rules to recover the support of $L^{\ast}$. Ref.~\cite{rayas_learning_2022} explains in great detail why this top-stage procedure is inferior to direct estimation in terms of sample complexity for the i.i.d. setting (see Fig.~1 in Ref.~\cite{rayas_learning_2022}). Mutatis mutandis, the same reasoning applies to our problem setup.
\end{remark}
\begin{remark}(Choosing $\ell_{1}$-regularization)\label{rem: Choice behind l_1-regularization}
The $\ell_{1}$-regularization is used to estimate a sparse matrix $\widehat{L}_{j}$. Popular applications include sparse linear regression, where it achieves both asymptotic support recovery \cite{bunea2007sparsity, bickel2009simultaneous} and finite-sample recovery under conditions such as mutual incoherence \cite{wainwright2009sharp, negahban2012unified}. In contrast, convex alternatives such as ridge regression do not induce sparsity \cite{hastie2009elements}. Iterative $\ell_2$-based methods like broken adaptive ridge (BAR) regression \cite{dai2018broken} can recover support asymptotically only when both the number of samples and iterations tend to infinity. Non-convex penalties such as the smoothly clipped absolute deviation (SCAD) and minimax concave penalty (MCP) relax mutual incoherence assumptions \cite{fan2001variable, loh2017support}, but are difficult to optimize due to non-convexity, sensitivity to tuning, and initialization. Given these trade-offs, we choose the $\ell_{1}$-penalty for its balance of theoretical guarantees and computational tractability. 
\end{remark}

\subsection{Statement of main results}\label{subsection: Statement of main results}
This section features two main results. The first one concerns the theoretical characterization of the convex estimator in \eqref{eq: l_1 regularized Whittle likelihood} when $\{X_{t}\}_{t\in\mathbb{Z}}$ is a Gaussian time series. And the second one gives such a characterization when $\{X_{t}\}_{t\in\mathbb{Z}}$ is a non-Gaussian linear process. At a high level our result for the Gaussian setting states that as long as the time domain samples $n$ scales as $\Omega(d^3\log p)$, the estimate $\widehat{L}$ correctly recovers the true support and is close to $L^\ast$ (measured in Frobenius and operator norms) with high probability. Here $d$ is the maximum degree of the graph underlying $L^{\ast}$. In the linear process setting, such a performance is guaranteed if $n$ scales as $\Omega(d^{3}(\log p)^{4+\rho})$ for sub-exponential families with parameter $\rho$ and $\Omega(d^{3}p^{2})$ for distributions with finite fourth moment, respectively.

Our main results rely on three assumptions. These type of assumptions, but not identical, appeared in the literature of $\ell_{1}$-constrained least squares problem \cite{wainwright2009sharp,zhao2006model} and in the literature of $\ell_{1}$-regularized inverse-covariance and spectral density estimation \cite{ravikumar2011high,deb2024regularized}. Define the edge set $\mathcal{E}(L^{\ast}) = \{(i,j): L^{\ast}_{ij}\neq 0, \text{for all} \hspace{3px} i\neq j\}$. Let $E=\{\mathcal{E}(L^{\ast})\cup(1,1)\ldots\cup (p,p)\}$ be the augmented edge set including edges for the diagonal elements of $L^{\ast}$. Let $E^{c}$ be the set complement of $E$. 

\smallskip 
\noindent\textbf{[A1] Mutual incoherence condition:} Let $\Gamma^{\ast}$ be the Hessian of the log-determinant in \eqref{eq: l_1 regularized Whittle likelihood}:  
\begin{align}\label{assn: Mutual incoherence}
    \Gamma^{\ast} \triangleq \nabla_{L}^{2}\log\det(L)\vert_{L=L^{\ast}} = {L^{\ast}}^{-1}\otimes {L^{\ast}}^{-1}. 
\end{align}
We say that $L^{\ast}$ satisfies the mutual incoherence condition if $\vertiii{\Gamma^{\ast}_{E^{c}E}{\Gamma^{\ast}_{EE}}^{-1}}_{\infty}\leq 1-\alpha$, for some $\alpha\in (0,1]$.
 
The incoherence condition on $L^\ast$ controls the influence of irrelevant variables (elements of the Hessian matrix restricted to $E^c\times E$ on relevant ones (elements restricted to $E\times E$). The $\alpha$-incoherence assumption, commonly used in the literature, has been validated for various graphs like chain and grid graphs \cite{ravikumar2011high}. While $\alpha$-incoherence in \cite{ravikumar2011high,deb2024regularized} is imposed on the inverse covariance or spectral density matrix, we enforce it on \(L^\ast\). A similar condition has also been explored in \cite{rayas_learning_2022}. We note that mutual incoherence is sufficient but not strictly necessary for support recovery \footnote{It is nearly necessary for sign selection consistency, but not for support recovery; see \cite{zhao2006model}}. Non-convex penalties such as SCAD and MCP achieve support recovery without requiring incoherence \cite{fan2001variable, loh2017support}. Although these non-convex regularizers introduce challenges related to optimization (see Remark~\ref{rem: Choice behind l_1-regularization}), we view them as a promising direction for future work.

\smallskip

\noindent\textbf{[A2] Bounding temporal dependence:} $\{Y_{t}\}_{t\in\mathbb{Z}}$ has short range dependence: $\sum_{l=-\infty}^{\infty}\|\Phi_{Y}(l)\|_{\infty}<\infty$.
Thus, the autocorrelation function $\Phi_Y(l)$ decreases quickly as the time lag $l$ increases, leading to negligible temporal dependence between samples that are far apart in time.

This mild assumption holds if the nodal injections $\{X_{t}\}_{t\in\mathbb{Z}}$ exhibits short range dependence: $\sum_{l=-\infty}^{\infty}\|\Phi_{X}(l)\|_{\infty}<\infty$. In fact, $\sum_{l=-\infty}^{\infty}\|\Phi_{Y}(l)\|_{\infty} = \sum_{l=-\infty}^{\infty}\|{L^\ast}^{-1}\Phi_{X}(l){L^\ast}^{-1}\|_{\infty}\leq\nu_{{L^\ast}^{-1}}^{2}\sum_{l=-\infty}^{\infty}\|\Phi_{X}(l)\|_{\infty}<\infty$, where $\nu_{{L^\ast}^{-1}}$ is the $\ellinf$-matrix norm of ${L^\ast}^{-1}$. Notice that in real systems like power networks, injections typically are short-range dependent {\color{black} processes \cite{doddi2022efficient}.}

\smallskip 

\noindent \textbf{[A3] Condition number bound on the Hessian:}\label{assn: hessian regularity} The condition number $\kappa(\Gamma^{\ast})$ of the Hessian matrix in \eqref{assn: Mutual incoherence} satisfies: 
\begin{align}\label{eq: Bounded hessian}
    \kappa(\Gamma^{\ast})\triangleq \vertiii{{\Gamma^{\ast}}}_{\infty}\!\vertiii{{\Gamma^{\ast}}^{-1}}_{\infty}\!\leq \!\frac{1}{4d\nu_{D^{2}_{j}}\|\Theta^{-1}_{Y}(\omega_{j})\|_{\infty}C_{\alpha}}, 
\end{align}
where $C_{\alpha} = 1+\frac{24}{\alpha}$, $\alpha\in (0,1]$, $\omega_{j}\in \mathcal{F}_{n}$, and $d$ is the maximum degree of the graph underlying $L^{\ast}$. Bounding $\kappa(\Gamma^{\ast})$ to derive estimation consistency results is standard in the high-dimensional graphical model literature \cite{cai2011constrained, rothman2008sparse}.

\smallskip 

\subsubsection{Structure learning with Gaussian injections}\label{subsection: Gaussian time series}
Let $\{X_t\}_{t\in \mathbb{Z}}$ in \eqref{eq: flow systems} be a WSS Gaussian process. Consequently, $\{Y_t\}_{t\in \mathbb{Z}}$, a linear transformation of $X_t$, is also a WSS Gaussian process. Under this assumption, Theorem \ref{thm: Gaussian time series} provides sufficient conditions on the number of samples $n$ of $Y_t$ required so that the estimator $\widehat{L}$ in \eqref{eq: l_1 regularized Whittle likelihood} exactly recovers the sparsity structure of $L^\ast$ and achieves norm and sign consistency. Here, sign consistency is defined as $\text{sign}(\widehat{L}_{ij}) = \text{sign}(L^{\ast}_{ij})$, for all $(i,j)\in E$. We recall that  $\nu_{A} = \vertiii{A}_{\infty} \triangleq \max_{j=1,\ldots,p}\sum_{j=1}^{p}\vert A_{ij}\vert$.

Define the two model-dependent quantities: 
\begin{align}
    \Omega_{n}(\Theta_{Y}^{-1}) &= \max_{r\geq 1,s\leq p}\sum_{|l|<n}|l||\Phi_{Y,rs}(l)| \label{eq: model dependent quantities Omega_n} \\
    L_{n}(\Theta_{Y}^{-1}) &=\max_{r\geq 1,s\leq p}\sum_{|l|>n}|\Phi_{Y,rs}(l)|\label{eq: model dependent quantities L_n}.
\end{align}
These quantities play a crucial role in the norm consistency bounds presented in Theorem \ref{thm: Gaussian time series} and Theorem \ref{thm: Linear Process} (see Remark \ref{rem: control on bias of periodogram}).

Below is an informal version of the main theorem. A formal statement and a proof with all numerical and model-dependent constants are in Appendix \ref{app: Gaussian time series}. We define $|L^{\ast}_{\min}|\triangleq \min_{(i,j)\in E}|L^{\ast}_{ij}|$ to be the minimum absolute value of the non-zero entries in $L^{\ast}$. We use $x\geqsim y$ to denote $x\geq cy$, where the constant $c$ is independent of model parameters and dimensions. 

\begin{theorem}\label{thm: Gaussian time series}
Let the node injections $X_{t}$ be a WSS Gaussian time series. Consider any Fourier frequency $\omega_{j}\in[-\pi,\pi]$. Suppose that assumptions in $\textbf{[A1-A3]}$ hold. Define $\alpha>0$ and $C_{\alpha} =  1+24/\alpha$. Let 
%the regularization parameter be chosen as 
$\lambda_{n} = 96\nu_{D^2}\nu_{L^{\ast}}\delta_{\Theta_{Y}^{-1}}(m,n,p)/\alpha$ and the bandwidth parameter $m\geqsim \vertiii{\Theta_{Y}^{-1}}_{\infty}^{2}\zeta^{2}d^{2}\log p$, where 
$\zeta = \max\{\nu_{{\Gamma^{\ast}}^{-1}}\nu_{{L^{\ast}}^{-1}}\nu_{L^{\ast}}\nu_{D^{2}}C^{2}_{\alpha},\nu^{2}_{{\Gamma^{\ast}}^{-1}}\nu^{3}_{{L^{\ast}}^{-1}}\nu_{L^{\ast}}\nu_{D^{2}}C^{2}_{\alpha}\}$. 

If the sample size $n\geqsim \Omega_{n}(\Theta_{Y}^{-1})\zeta md$. Then  with probability greater than $1-1/p^{\tau-2}$, for some $\tau>2$, we have 

    \begin{enumerate}[label=(\alph*)]
        \item $\widehat{L}_{j}$ exactly recovers the sparsity structure i.e., $[\widehat{L}_j]_{E^{c}}=0$.
        \item The estimate $\widehat{L}_{j}$ which is the solution of \eqref{eq: l_1 regularized Whittle likelihood} satisfies 
        \begin{align}
            \|\widehat{L}_{j}-L^{\ast}\|_{\infty} \leq 8\nu^{\prime}\delta_{\Theta_{Y}^{-1}}(m,n,p).
        \end{align}
        \item $\widehat{L}_{j}$ satisfies sign consistency if:
        \begin{align}
            |L^{\ast}_{\min}(E)|\geq 8\nu^{\prime}\delta_{\Theta_{Y}^{-1}}(m,n,p),
        \end{align}
    \end{enumerate}
where, $\nu^{\prime} = \nu_{{\Gamma^{\ast}}^{-1}}\nu_{D^{2}}\nu_{L^{\ast}}C_{\alpha}$ and
\begin{align*}
    \delta_{\Theta_{Y}^{-1}}(m,n,p) \!=\!\sqrt{\frac{\tau\log p}{m}} +\frac{m+\frac{1}{2\pi}}{n}\Omega_{n}(\Theta_{Y}^{-1})+\frac{1}{2\pi}L_{n}(\Theta_{Y}^{-1}). 
\end{align*} 
\end{theorem} 
Some remarks are in order. Assume that $\zeta$ and $\vertiii{\Theta_{Y}^{-1}}_{\infty}$ are independent of $(n,p,d)$ and that we are in the high-dimensional regime where $\log p/n\rightarrow 0$ as $(n,p)\rightarrow \infty$. Under assumptions in Theorem \ref{thm: Gaussian time series}, and when $n=\Omega(d^3\log p)$, with high probability: (a) The support of $\widehat{L}$ is contained within $L^{\ast}$; meaning there are no false negatives. Furthermore, when $(m/n)\Omega_{n}(\Theta_{Y}^{-1})\rightarrow 0$ as $(m,n)\rightarrow \infty$, part (b) asserts that the element-wise $\ellinf$-norm, $\|\widehat{L}-L^{\ast}\|_{\infty}$, vanishes asymptotically (see Remark \ref{rem: control on bias of periodogram} for further discussion on the asymptotic decay of the error norm). Finally, part (c) establishes the sign consistency of $\widehat{L}$. Crucial is the requirement of $\vert L^{\ast}_{\min}\vert = \Omega\left(\delta_{\Theta_{Y}^{-1}}(m,n,p)\right)$, which limits the minimum value (in absolute) of the nonzero entries in $L^\ast$. This condition parallels the familiar \emph{beta-min} condition in the LASSO literature (see \cite{wainwright2009sharp, ravikumar2011high,deb2024regularized}). Finally, since each estimate $\widehat{L}_j$ for $j = 1, \ldots, n-1$ satisfies the same statistical guarantees with high probability, any $\widehat{L}_j$ can be selected as a candidate estimator for $L^{\ast}$.

The error bound $\delta_{\Theta_Y^{-1}}$ in Theorem \ref{thm: Gaussian time series} quantifies the deviation of the estimator $\widehat{L}_j$ from the true Laplacian $L^{\ast}$ in the element-wise $\ell_\infty$-norm. It has two components: the first term, $\sqrt{\log p / m}$, captures the leading statistical error, while the second, involving $\Omega_n$ and $L_n$, accounts for temporal and contemporaneous dependencies in the data. As defined in equations \eqref{eq: model dependent quantities Omega_n} and \eqref{eq: model dependent quantities L_n}, these terms vanish under i.i.d. data and increase with stronger temporal dependence in the data.

We also emphasize the strength of the above result. Although $\widehat{L}_j$ is derived from the Whittle approximation, Theorem \ref{thm: Gaussian time series} ensures support recovery and norm consistency. Prior works such as \cite{rao2021reconciling} have studied the discrepancy between the Gaussian and Whittle likelihoods. While formally quantifying this approximation error is beyond the scope of the present work, we view it as a valuable direction for future research.

We state a corollary to Theorem \ref{thm: Gaussian time series} that gives error-consistency rates for $\widehat{L}$ in the Frobenius and operator norms. Let $\mathcal{E}(L^{\ast}) = \{(i,j): L^{\ast}_{ij}\neq 0, \text{for all} \hspace{3px} i\neq j\}$ be the edge set. 
\begin{corollary}
    Let $s=|\mathcal{E}(L^{\ast})|$ be the cardinality of the edge set $\mathcal{E}(L^{\ast})$. Under the hypothesis as in Theorem \ref{thm: Gaussian time series}, with probability greater than $1-\frac{1}{p^{\tau-2}}$, the estimator $\widehat{L}$ defined in \eqref{eq: l_1 regularized Whittle likelihood} satisfies
    \begin{align*}
    \Vert \widehat{L}\!-\!L^{\ast}\Vert_{F} & \leq 8\nu^{\prime} (\sqrt{s+p})\delta_{\Theta_{Y}^{-1}}(m,n,p) \quad \text{ and }\\
    \Vert{\widehat{L}\!-\!L^{\ast}}\Vert_{2} &\leq 8\nu^{\prime}\min\{d,\sqrt{s+p}\}\delta_{\Theta_{Y}^{-1}}(m,n,p),  
\end{align*}
where $\nu^{\prime}$ and $\delta_{\Theta_{Y}^{-1}}(m,n,p)$ are defined in Theorem \ref{thm: Gaussian time series}. 

\end{corollary}
\emph{\textbf{Proof sketch:}}. Both the Frobenius and operator norm bounds follow by applying standard matrix norm inequalities to the $\ellinf$ consistency bound in part (b) of Theorem \ref{thm: Gaussian time series}. Importantly, $s+p$ is the bound on the maximum number of non-zero entries in $L^\ast$, where $s$ is the total number of off-diagonal non-zeros in $L^\ast$. Complete details are in Appendix \ref{app: frobenius norm bounds for gaussian time series}.

\subsubsection{Structure learning for non-Gaussian injections }\label{subsection: non-Gaussian time series}
We consider a class of WSS processes $\{X_t\}_{t\in \mathbb{Z}}$ that are not necessarily Gaussian. Examples include Vector Auto Regressive (VAR($p$)) and Vector Auto Regressive Moving Average (VARMA ($p,q$)) models with non-Gaussian noise terms. Such models, and many others, belong to a family of linear WSS processes with absolute summable coefficients:
\begin{align}\label{eq: linear process}
    X_{t} = \sum_{l=0}^{\infty}A_{l}\epsilon_{t-l},
\end{align}
where $A_{l}\in\mathbb{R}^{p\times p}$ is known and $\epsilon_{t}\in\mathbb{R}^{p}$ is a zero mean i.i.d. process with tails possibly heavier than Gaussian. The absolute summability $\sum_{l=0}^{\infty}|A_{l}(i,j)|<\infty$ ensures {\color{black} stationarity} for all $i,j\in \{1,\ldots,p\}$ \cite{rosenblatt2012stationary}. {\color{black} We assume that $\epsilon_{kl}$ (for all $k\in [p]$), the $k$-th component of $\epsilon_{l} \in \mathbb{R}^p$, is given by one of the distributions below:}

\smallskip 
\noindent\textbf{[B1] Sub-Gaussian:} There exists $\sigma>0$ such that for all $\eta>0$, we have $\mathbb{P}[|\epsilon_{kl}|>\eta]\leq 2\exp(-\frac{\eta^{2}}{2\sigma^{2}})$. 

\smallskip 

\noindent\textbf{[B2] Generalized sub-exponential with parameter $\rho>0$:} There exists constants $a$ and $b$ such that for all $\eta>0$: $ \mathbb{P}[|\epsilon_{kl}|>\eta^{\rho}]\leq a\exp(-b\eta)$.

\smallskip 
\noindent\textbf{[B3] Distributions with finite $4^{\text{th}}$ moment:} There exists a constant $M>0$ such that $\mathbb{E}[\epsilon_{kl}^{4}]\leq M <\infty$.

\medskip 

We need additional notation. Let $n_{k}=\Omega(d^{3}\mathcal{T}_{k})$ represent the family of sample sizes indexed by $k=\{1,2,3\}$, where $\mathcal{T}_{1}=\log p$ correspond to the distribution in [\textbf{B1}], $\mathcal{T}_{2}=(\log p)^{4+4\rho}$ in [\textbf{B2}], and $\mathcal{T}_{3}=p^{2}$ in [\textbf{B3}].  
\begin{theorem}\label{thm: Linear Process}
    Let $X_{t}$ be given by \eqref{eq: linear process} and $Y_{t} = {L^{\ast}}^{-1}X_{t}$. Fix $\omega_{j}\in[-\pi,\pi]$. Let $n_{k}=\Omega(d^{3}\mathcal{T}_{k})$, where $k=\{1,2,3\}$. Then for some $\tau>2$, with probability greater than $1-1/p^{\tau-2}$:
    \begin{enumerate}[label=(\alph*), itemsep=0pt]
    \item $\widehat{L}$ exactly recovers the sparsity structure i.e., $\widehat{L}_{E^{c}}=0$.
    \item The $\ellinf$ bound of the error satisfies:
    \begin{align}
        \|\widehat{L}-L^{\ast}\|_{\infty} = \mathcal{O}(\delta_{\Theta_{Y}^{-1}}^{(k)}(n,m,p)).
    \end{align}
    \item $\widehat{L}$ satisfies sign consistency if:
    \begin{align}
        |L^{\ast}_{\min}(E)| = \Omega(\delta_{\Theta_{Y}^{-1}}^{(k)}(n,m,p)),
    \end{align}
    \end{enumerate}
    where $\delta_{\Theta_{Y}^{-1}}^{(k)}(n,m,p)$ for $k=\{1,2,3\}$ is given by
        \begin{align*}
            \delta_{\Theta_{Y}^{-1}}^{(1)}(n,m,p) &= \vertiii{\Theta_{Y}^{-1}}_{\infty}\frac{(\tau\log p)^{1/2}}{\sqrt{m}}+\bigtriangleup(n,m,\Theta_{Y}^{-1})\\
            \delta_{\Theta_{Y}^{-1}}^{(2)}(n,m,p) &= \vertiii{\Theta_{Y}^{-1}}_{\infty}\frac{(\tau\log p)^{2+2\rho}}{\sqrt{m}}+\bigtriangleup(n,m,\Theta_{Y}^{-1})\\
            \delta_{\Theta_{Y}^{-1}}^{(3)}(n,m,p) &= \vertiii{\Theta^{-1}_{Y}}_{\infty}\frac{p^{1+\tau}}{\sqrt{m}}+\bigtriangleup(n,m,\Theta_{Y}^{-1}),
        \end{align*}
        where $\bigtriangleup(n,m,\Theta_{Y}^{-1})=\frac{m+\frac{1}{2\pi}}{n}\Omega_{n}(\Theta_{Y}^{-1})+\frac{1}{2\pi}L_{n}(\Theta_{Y}^{-1})$. 
\end{theorem} 

\begin{remark}\label{rem: control on bias of periodogram}
(Asymptotic decay rate of the error $\|\widehat{L}-L^{\ast}\|_{\infty}$) The model-dependent quantities $\Omega_{n}(\Theta^{-1}_{Y})$ and $L_{n}(\Theta^{-1}_{Y})$, as defined in \eqref{eq: model dependent quantities Omega_n} and \eqref{eq: model dependent quantities L_n}, are critical for bounding the element-wise $\ell_{\infty}$-norm of the error $\|\widehat{L} - L^{\ast}\|_{\infty}$ in Theorems \ref{thm: Gaussian time series} and \ref{thm: Linear Process}. We examine conditions under which this error vanishes asymptotically. Specifically, by definition in \eqref{eq: model dependent quantities L_n}, the quantity $(\sqrt{\log p / m},L_{n}(\Theta^{-1}_{Y})) \rightarrow 0$ as $(m,n) \rightarrow \infty$. Furthermore, if $(m/n)\Omega_{n}(\Theta^{-1}_{Y}) \rightarrow 0$ as $(m,n) \rightarrow \infty$, then the error norm vanishes asymptotically. This condition holds in scenarios where the autocovariance function $\Phi_{Y}(l)$ exhibits a geometric decay rate or if $\{Y_t\}_{t \in \mathbb{Z}}$ is a VAR(d) process or other stationary processes with strong mixing conditions (see Proposition 3.4 in \cite{sun2018large}). As a consequence, the condition $(m/n)\Omega_{n}(\Theta^{-1}_{Y}) \rightarrow 0$ as $(m,n)\rightarrow \infty$ holds for a wide range of stationary processes, leading to asymptotic decay of the error norm.
\end{remark}

\subsection{Outline of technical analysis for main results}\label{subsection: Outline of main analysis}

We summarize the key techniques used to prove Theorems \ref{thm: Gaussian time series} and \ref{thm: Linear Process}. Complete details are in Appendix \ref{subsection: Proofs of all technical results}. We leverage the primal-dual witness (PDW) method---a general technique used to derive statistical guarantees for sparse convex estimators \cite{ravikumar2011high,wainwright2009sharp}. Before detailing the PDW method, we state differences in our proof approach compared to the cited literature. First, our analysis is in the frequency domain, this accounts for temporal dependencies from the WSS process, requiring careful treatment of the Hermitian matrices $P_j$ and $D^2$ in \eqref{eq: l_1 regularized Whittle likelihood}. Second, unlike most literature where the objective function's dependence on the optimization variable \(L\) is linear, our objective function in \eqref{eq: l_1 regularized Whittle likelihood} has a quadratic dependence. This distinction in the frequency domain necessitates stricter control of the Hessian matrix $\Gamma^{\ast}$ via our assumption [\textbf{A3}].

In the PDW method, we construct an optimal primal-dual pair $(\widetilde{L},\widetilde{Z})$ that satisfies the zero sub-gradient condition of the problem in \eqref{eq: l_1 regularized Whittle likelihood}. (i) The primal 
$\widetilde{L}$ is constrained to have the correct signed support $E$ of the true Laplacian matrix $L^{\ast}$ and (ii)  The dual $\widetilde{Z}$ is the sub-gradient of $\|L\|_{1,\text{off}}$ evaluated at $\widetilde{L}$. If the dual $\widetilde{Z}$ satisfies the strict dual feasibility condition $\|\widetilde{Z}_{E^{c}}\|_{\infty}<1$. Then the dual acts as a witness to certify that $\widetilde{L} = \widehat{L}$ and $\widetilde{L}$ is indeed the unique global optimum.

\subsection{The primal-dual construction and supporting lemmata}\label{subsection: Supporting lemmas}

We construct an optimal primal-dual pair $(\widetilde{L},\widetilde{Z})$. Lemma \ref{lma: Sufficient conditions for strict dual feasibility} gives conditions under which this construction succeeds. First, we determine $\widetilde{L}$ by solving the restricted problem:
\begin{align}\label{eq: restricted l_1 whittle likelihood}
   \hspace{-2.5mm} \widetilde{L} \triangleq \argmin_{L\succ 0,L_{E^{c}}=0} \Tr(DLPLD)\!-\!\log \det(L^{2})\!+\!\lambda_{n}\|L\|_{1,\text{off}}.
\end{align}
Notice that $\widetilde{L}\succ 0$ and $\widetilde{L}_{E^{c}} = 0$. We choose the dual $\widetilde{Z}\in \|\widetilde{L}\|_{1,\text{off}}$ to satisfy the zero sub-gradient condition of \eqref{eq: restricted l_1 whittle likelihood} by setting $\lambda_{n}\widetilde{Z}_{ij}\!=\!-2[\Psi_{1}\widetilde{L}P_{1}]_{ij} + 2[\Psi_{2}\widetilde{L}P_{2}]_{ij}+2[\widetilde{L}^{-1}]_{ij}$, for all $(i,j)\in E^{c}$, where $P_1$ (resp. $\Psi_1$) and $P_2$ (resp. $\Psi_1$) are the real and imaginary parts of $P$ (resp. $D$). Therefore, the pair $(\widetilde{L},\widetilde{Z})$ satisfies the zero sub-gradient condition of the restricted problem in \eqref{eq: restricted l_1 whittle likelihood}.

We verify the strict dual feasibility condition: $|\widetilde{Z}_{ij}|<1$, for any $(i,j)\in E^{c}$. We introduce three quantities. First, $W\triangleq P-{\Theta}^{-1}_{Y}$ quantifies the error between the averaged periodogram $P$ and the true spectral density matrix ${\Theta}^{-1}_{Y}$. Second, let $\Delta \triangleq \widetilde{L}-L^{\ast}$ be the measure of distortion between $\widetilde{L}$ given by \eqref{eq: restricted l_1 whittle likelihood} and the true Laplacian matrix $L^{\ast}$. The final quantity $R(\Delta)$ captures higher order terms in the Taylor expansion of the gradient $\nabla \log\det(\widetilde{L})$ centered around $L^{\ast}$. In fact, expand $\nabla \log\det(\widetilde{L})= \widetilde{L}^{-1} =  {L^{\ast}}^{-1} + {L^{\ast}}^{-1}\Delta {L^{\ast}}^{-1} + {\widetilde{L}^{-1}-{L^{\ast}}^{-1}-{L^{\ast}}^{-1}\Delta {L^{\ast}}^{-1}}$, and then define $\widetilde{L}^{-1}-{L^{\ast}}^{-1}-{L^{\ast}}^{-1}\Delta {L^{\ast}}^{-1} = R(\Delta)$.

The following lemma establishes the sufficient conditions for ensuring strict dual feasibility.
\begin{lemma}(Conditions for strict-dual-feasibility)\label{lma: Sufficient conditions for strict dual feasibility}
Let $\lambda_{n}>0$ and $\alpha$ be defined as in \textbf{[A1]}. Suppose that $\max\{2\nu_{D^{2}}(d\|\Delta\|_{\infty}+\nu_{L^{\ast}})\|W\|_{\infty}, \|R(\Delta)\|_{\infty},2\nu_{D^{2}}d\|\Delta\|_{\infty}\|\Theta^{-1}_{Y}\|_{\infty}\}\leq \frac{\alpha\lambda_{n}}{24}$. Then the dual vector $\widetilde{Z}_{E^{c}}$ satisfies $\|\widetilde{Z}_{E^{c}}\|_{\infty}<1$, and hence, $\widetilde{L} = \widehat{L}$. 
\end{lemma}

\emph{\textbf{Proof sketch:}} Express the sub-gradient condition in Lemma \ref{lma: convexity of objective} in a vectorized form as a function of $R(\Delta)$, $W = P-\Theta_{Y}^{-1}$, and $\Theta^{-1}_{Y}$. We decompose the vectorized sub-gradient condition into two linear equations corresponding to the edge set $E$ and its complement $E^{c}$. An expression for $\widetilde{Z}_{E^{c}}$ is obtained as a function of $R(\Delta), W$ and $\Theta_{Y}^{-1}$. We finish the proof by utilizing the mutual incoherence condition stated in $\textbf{[A1]}$.

\smallskip 
The following results provide us with dimension and model complexity dependent bounds on the remainder term $R(\Delta)$.  The proof, adapted from \cite[lemma 5]{ravikumar2011high}, relies on matrix expansion techniques; see Appendix \ref{app: control remainder} for details.

%under bounded distortion $\Delta$. 
\begin{lemma}\label{lma: control remainder}
    Suppose that the $\ell_{\infty}$-norm $\|\Delta\|_{\infty}\leq {1}/{(3\nu_{{L^{\ast}}^{-1}}d)}$, then $\|R(\Delta)\|_{\infty}\leq \frac{3}{2}d\|\Delta\|_{\infty}^{2}\nu_{{L^{\ast}}^{-1}}^{3}$.
\end{lemma}

 The result below provides a sufficient condition under which the $\ellinf$-bound on $\Delta$ in Lemma \ref{lma: control remainder} holds. Full proof in Appendix \ref{app: control distortion}. 

\begin{lemma}\label{lma: Control of distortion} Define $ r\!\triangleq\! 8\nu_{{\Gamma^{\ast}}^{-1}}[\nu_{D^{2}}\nu_{L^{\ast}}\|W\|_{\infty}\!+\!\lambda_{n}/4]$ and suppose $r\leq\min\{{1}/{(3\nu_{{L^{\ast}}^{-1}}d)},{1}/{(6\nu_{{\Gamma^{\ast}}^{-1}}\nu_{{L^{\ast}}^{-1}}^{3}d)}\}$.  Then we have the element-wise $\ell_{\infty}$-bound: $\|\Delta\|_{\infty} = \|\widetilde{L}-L^{\ast}\|_{\infty}\leq r$.  

\end{lemma}

\emph{\textbf{Proof sketch:}} Since $\widetilde{L}_{E^{c}} = L^{\ast}_{E^{c}} = 0$, we note $\|\Delta\|_{\infty} = \|\Delta_{E}\|_{\infty}$, where $\Delta_{E} = \widetilde{L}_{E} - L^{\ast}_{E}$ and it is the solution of the sub-gradient associated with the restricted problem in \eqref{eq: restricted l_1 whittle likelihood}. We construct a continuous function $F: \mathbb{R}^{|E|}\rightarrow \mathbb{R}^{|E|}$ with two properties: (i) it has a unique fixed point $\Delta_{E}$ and (ii) On invoking assumption $\textbf{[A3]}$, $F$ is a contraction---specifically, $F(\mathbb{B}_{r})\subseteq \mathbb{B}_{r}$, where $\mathcal{B}_{r} = \{A\in \mathbb{R^{|E|}}: \|A\|_{\infty}\leq r\}$ and $r$. The proof follows by invoking Brower's fixed point theorem \cite{kellogg} and exploiting the unique fixed point property of $F$ to show that $\Delta_{E}\in \mathcal{B}_{r}$, and hence, $\|\Delta\|_{\infty}\leq r$. 
\begin{remark}\label{remark: Usefulness of bounded hessian}
    A consequence of assumption \textbf{[A3]} is the lower bound on the norm of the Hessian $\vertiii{\Gamma^{\ast}}_{\infty}$.  This implies that the curvature at the true minimum $L^{\ast}$ is lower bounded. This bound on the curvature is specific to our problem and helps in attaining control on the distortion parameter $\Delta$, as demonstrated in Lemma \ref{lma: Control of distortion}.
\end{remark}

\section{Simulations}\label{section: Experiments}
 We report the results of multiple simulations to validate our theoretical claims. The results in Theorems \ref{thm: Gaussian time series} and \ref{thm: Linear Process} involve several constants, along with the dimensional parameters $(n, m, d, p)$. Therefore, we do not expect the theoretical results to capture the nuanced behavior of the simulations in every detail. However, we observe that the learning performance of the estimator in \eqref{eq: l_1 regularized Whittle likelihood} improves as the rescaled sample size $n/({d^3 \log(p)})$ increases, and that the error norm decreases with increasing $n/\log p$. Additionally, the experimental results are also influenced by the choice of the regularization $\lambda_{n}$. 
%All numerical experiments in this section were implemented 
We ran the experiments using CVXPY 1.2, an open-source Python package. The reproducible code for generating simulation results in this paper is publicly available at \url{https://tinyurl.com/LNSWSSP}.

\subsection{Setup and accuracy evaluation metrics}\label{subsec: Setup and accuracy evaluation metrics}

Our experiments assess the finite-sample performance of the proposed estimator 
for two families of stochastic injections $\{X_{t}\}_{t\in \mathbb{Z}}$, namely, vector autoregressive (VAR (1)) and vector autoregressive moving average (VARMA (2,2)) processes. These processes not only satisfy our technical assumptions but are also widely used for empirical studies. 

\smallskip 
\textit{(i) VAR(1) process}: Here the injections $\{X_{t}\}_{t\in \mathbb{Z}}$ satisfy $X_{t} = AX_{t-1} + \epsilon_{t}$ where $\epsilon_t \overset{\text{i.i.d.}}{\sim} \mathcal{N}(0,1)$ and $A = 0.7 I_{p}$. The PSD matrix of this process, for $z=e^{-i\omega}$ and $\omega \in [-\pi, \pi]$, is 
\begin{align*}
    f_{X}(\omega) = \frac{1}{2\pi} \left(I_{p} - A z\right) \left(I_{p} - A z^{-1}\right)^{-1} . 
\end{align*}

\smallskip 
\textit{(ii) VARMA(2,2) process}: We let
$X_t = A_{1}X_{t-1} + A_{2}X_{t-2} + \epsilon_{t} + B_{1}\epsilon_{t-1} + B_{2}\epsilon_{t-2}$ where $\epsilon_t \overset{\text{i.i.d.}}{\sim} \mathcal{N}(0,1)$. The PSD matrix of this process, for $z=e^{-i\omega}$ and with $\omega \in [-\pi, \pi]$, is \cite{brockwell2009time}
\begin{align} \label{eq: Noise PSD}
\hspace{-2.5mm}f_{X}(\omega) = \frac{1}{2\pi} \mathcal{A}(z)\mathcal{B}(z)\mathcal{B}^{\dagger}(z)(\mathcal{A}^{-1}(z^{-1}))^{\dagger},
\end{align}
where $\mathcal{A}(z) = I_{p}-\sum_{t=1}^{2}A_{t}z^t$ and $\mathcal{B}(z) = I_{p}-\sum_{t=1}^{2}B_{t}z^t$. We set \(A_{1} = 0.4 I_{p}\) and \(A_{2} = 0.2 I_{p}\). Furthermore, \(B_{1} = 1.5(I_{5} + J_{5})\) and \(B_{2} = 0.75(I_{5} + J_{5})\), where $J_{k} \in \mathbb{R}^{k \times k}$ is the matrix of all ones.

For the above processes, we assume that the nodal observation data $
\{Y_{t}\}_{t\in \mathbb{Z}}$ satisfy $Y_{t} = {L^{\ast}}^{-1} X_{t}$, where we consider $L^{\ast}$ for synthetic, benchmark, and real-world networks (discussed later). The periodogram of $
\{Y_{t}\}_{t=1}^n$ at frequency $\omega_{j}$ is then computed as $P(\omega_j) = \frac{1}{2\pi(2m+1)} \sum_{|k|\leq m} d(\omega_{j+k}) d(\omega_{j+k})^{\dagger}$.
For simplicity, we set the centering frequency as $\omega_{j}=0$. However, our numerical and theoretical analysis applies to any non-zero Fourier frequency. Further, the bandwidth parameter $m = \sqrt{n}$, which is theoretically justified because we consider the regime $m/n \to 0$ as $(m,n) \to \infty$ where the periodogram is asymptotically unbiased (see Remark \ref{rem: control on bias of periodogram} and \cite{sun2018large}).

We consider sparsistency (the ability to recover the correct edge structure) and norm-consistency (the Frobenius norm of the deviation between $\widehat{L}$ and $L^{\ast}$) metrics to evaluate the estimation performance. We assess sparsistency via the F-score: $\text{F-score} = {2\text{TP}}/({2\text{TP} + \text{FP} + \text{FN})} \in [0,1]$,
        % \begin{align} 
        %     \text{F-score} = \frac{2\text{TP}}{2\text{TP} + \text{FP} + \text{FN}} \in [0,1],  
        % \end{align} 
where $\text{TP}$ (true positives) is the number of correctly detected edges, $\text{FP}$ (false positives) is the number of non-existent edges detected, and $\text{FN}$ (false negatives) is the number of actual edges not detected. The higher the $\text{F-score}$, the better the performance of the estimator in learning the true structure, with $\text{F-score}=1$ signifying perfect structure recovery. 

\subsection{Synthetic networks}
\begin{figure*}[!t]
    \centering
    \includegraphics[width=\textwidth]{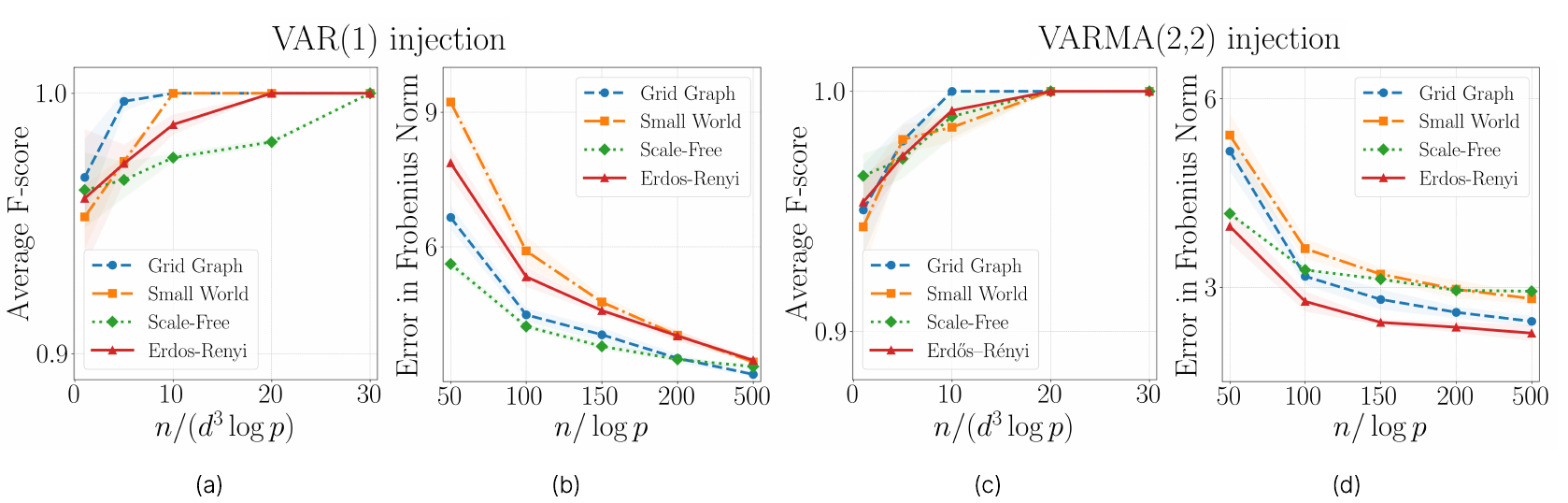}
    \caption{\small We evaluate the support recovery metric (F-score) and the Frobenius norm error for synthetic random networks under VAR(1) and VARMA(2,2) stochastic injections. Synthetic networks of size $ p=30$ are examined, with results averaged over 50 independent trials. Solid curves represent mean performance, while shaded regions around each curve indicate one-sigma standard deviations. The random networks analyzed include grid, small-world, scale-free, and Erdős-Rényi, with maximum degrees $d=\{4,3,9,4\}$, respectively. Panels (a,b) present the average F-score and Frobenius norm error versus rescaled sample size for VAR(1) injection, while panels (c,d) display the same metrics for VARMA(2,2) injection. The rescaled sample size for the F-score is  $n/(d^3 \log p)$, and for the Frobenius norm error, it is \( n/\log p \), based on asymptotic convergence rates in Theorem \ref{thm: Gaussian time series}. Notably, rescaling the sample size to \( n/(d^3 \log p) \) aligns all curves on top of each other as predicted by Theorem \ref{thm: Gaussian time series}.}
    \label{fig: Synthetic combined}
\end{figure*}
We present simulations evaluating the performance of the proposed estimator on synthetic random networks. All synthetic networks have a fixed size of $p=30$. The random networks examined in Figure~\ref{fig: Synthetic combined} include Erdős-Rényi, Small-World (Watts-Strogatz model), and Scale-Free (Barabási-Albert model) networks, with maximum degrees $d=\{4,3,9\}$, respectively. Additionally, a synthetic grid graph ($d=4$) is constructed by connecting each node to its fourth-nearest neighbor. 

For details on constructing the Laplacian matrix $L^{\ast}$ for the synthetic random networks, we refer the readers to \cite{narasimhan2019learning} and the GitHub repository\footnote{\url{https://github.com/psjayadev/Predicting-Links-Conserved-Networks}}. Once $L^{\ast}$ is obtained, we ensure its positive definiteness by adding a small diagonal perturbation of $0.1$ (positive definiteness by diagonal perturbation follows from the Gershgorin circle theorem). This perturbed matrix is no longer a Laplacian in the strict sense. However, this perturbation is acceptable since our estimation task focuses only on recovering the sparsity pattern of $L^{\ast}$ and not its spectral properties. In Figure~\ref{fig: Synthetic combined}, we plot the average F-score and the average Frobenius norm of the error (averaged over 50 independent trials) versus rescaled sample size under VAR(1) and VARMA(2,2) injections. Panels (a-b) depict these metrics for VAR(1) injection, while panels (c-d) show results for VARMA(2,2). The rescaled sample size is $n/(d^3 \log p)$ for F-score and $n/{\log p}$ for Frobenius norm error, based on asymptotic convergence rates in Theorem~\ref{thm: Gaussian time series}. As shown in panels (a) and (c), the F-score increases with $n/(d^3\log p)$, achieving perfect structure recovery, as predicted by Theorem~\ref{thm: Gaussian time series}. This causes all plots in panels (a) and (c) to align on top of each other. Panels (b) and (d) demonstrate similar behavior for the Frobenius norm error metric, where the error norm decreases with an increase in $n/\log p$.

In Figure~\ref{fig: Synthetic comparison}, we compare F-scores for i.i.d., VAR(1), and VARMA(2,2) injections on an Erdős-Rényi network with size $p=30$ and maximum degree $d=4$. The results indicate that fewer samples are needed to achieve perfect structure recovery (that is, $\text{F-score}=1$) with i.i.d. injections compared to injections of VAR (1) and VARMA (2,2). This trend aligns with theoretical expectations: structure recovery under i.i.d. injections requires $n=\mathcal{O}(d^2\log p)$ samples (see \cite{rayas_learning_2022}), compared to the higher sample complexity of $n=\mathcal{O}(d^3\log p)$ for VAR(1) and VARMA(2,2) (see Theorem~\ref{thm: Gaussian time series}).

\begin{figure}
    \centering
    \includegraphics[width=0.35\textwidth]{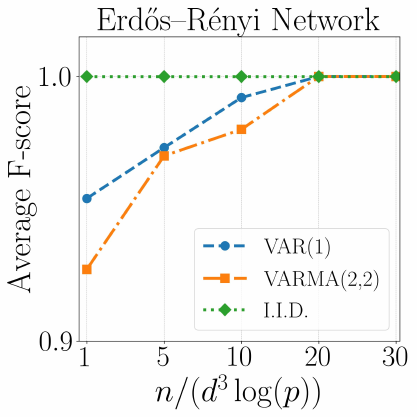}  
    \caption{Average F-score comparison for $\{X_{t}\}_{t \in \mathbb{Z}}$ governed by i.i.d., VAR(1), and VARMA(2,2) processes versus rescaled sample size for an Erdős-Rényi network ($p=30$, $d=4$). Perfect structure recovery under VAR(1) and VARMA(2,2) injections requires more samples than under i.i.d. injections.}
    \label{fig: Synthetic comparison}
\end{figure}

Finally, we comment on obtaining the regularization parameter $\lambda_{n}$ for experiments in Figure~\ref{fig: Synthetic combined} and \ref{fig: Synthetic comparison}. 
We apply the extended Bayesian information criterion (EBIC) \cite{chen2008extended} to select $\lambda_{n}$. The EBIC is given by:
\begin{align}
    \text{EBIC}_{\gamma}(\widehat{L}) = -2\mathcal{L}_{n}(\widehat{L}) + |\widehat{E}|\log n + 4\gamma |\widehat{E}|\log p, 
\end{align}
where $\mathcal{L}_{n}(\widehat{L})$ is the log-likelihood in \eqref{eq: l_1 regularized Whittle likelihood}, $\widehat{E} = E(\widehat{L})$ represents the edge set of the candidate graph $\widehat{L}$, and $\gamma \in [0,1]$ is a tuning parameter that influences the penalization. Higher values of $\gamma$ lead to sparser networks. The optimal regularization parameter is $\lambda_{n} = \argmin_{\lambda > 0} \text{EBIC}_{\gamma}(\widehat{L})$. 

 The results in Figure \ref{fig: Synthetic combined} and Figure \ref{fig: Synthetic comparison} are for $\gamma = 0.4$. In Figure \ref{fig: Tuning regularization}, we fix a sample size $n=1000$ and plot the regularization path for both the F-score and Frobenius norm error across various network types. Notably, we observe that for a class of random networks, and the fixed sample size $n=1000$ the value $\log(\lambda_{n})\approx -2$ simultaneously maximizes both the F-score and minimizes the Frobenius norm error.

\begin{figure}
    \centering
    \includegraphics[width=0.7\textwidth]{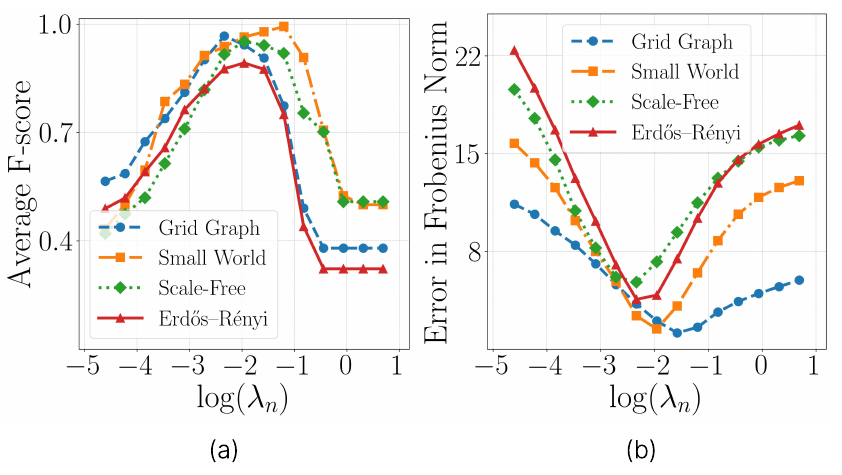} 
    \caption{For a fixed sample size $n=1000$, we plot (a) Regularization path for F-score and (b) regularization path for Frobenius norm error, both on a linear-log scale. All networks have \( p = 30 \) nodes, with maximum degrees as follows: grid (\( d = 4 \)), small-world (\( d = 3 \)), scale-free (\( d = 9 \)), and Erdős–Rényi (\( d = 4 \)). }
    \label{fig: Tuning regularization}
\end{figure}

\subsection{Benchmark networks}
For $\{X_{t}\}_{t \in \mathbb{Z}}$ governed by the VARMA(2,2) process, we evaluate the performance of our estimator on three benchmark networks: the power distribution network, water network, and the brain network. Each network has an associated ground truth matrix $L^{\ast}=A+\epsilon I_{p}$, where $A$ is the adjacency matrix that defines the edge structure of the network, $\epsilon = \{2, 2, 3\}$ for the power, water, and brain networks, respectively, and $I_{p}$ is the $p$-dimensional identity matrix. This diagonal perturbation ensures that $L^{\ast}$ is positive definite while preserving its sparsity pattern and thus does not affect the structure learning objective.

\textit{1) Power distribution network:} We consider the IEEE 33-bus power distribution network whose raw data files are publicly available\footnote{\url{https://www.mathworks.com/matlabcentral/fileexchange/73127-ieee-33-bus-system}}. An adjacency matrix $A$ can be constructed from this dataset. The network corresponding to $A$ consists of 33 buses and 32 branches (edges) with maximum degree $d=3$. 
\smallskip 

\textit{2) Water distribution network:} We examine the Bellingham water distribution network, using data sourced from the database described in \cite{hernadez2016water}. The raw data files are publicly accessible\footnote{\url{https://www.uky.edu/WDST/index.html}}. The ground truth adjacency matrix $A$, containing 121 nodes and 162 edges with maximum degree $d=6$, is generated by loading the raw data files into the WNTR simulator\footnote{\url{https://github.com/USEPA/WNTR}}. Complete details on obtaining the adjacency matrix are provided in \cite{seccamonte2023bilevel}. 

\smallskip 
\textit{3) Brain network:}  The ground truth adjacency matrix $A$ for this study is publicly accessible\footnote{\url{https://osf.io/yw5vf/}}, with the detailed methodology regarding its construction described in \cite{vskoch2022human}. The matrix $A$ is a $90 \times 90$ 
matrix (i.e., $90$ nodes), where each row and column corresponds to a specific region of interest (ROI) in the brain, as defined by the Automated Anatomical Labeling (AAL) atlas. From 88 patient-derived connectivity matrices found in the database, one was selected (filename: S001.csv) for numerical analyses. The selected network consists of $90$ nodes, $141$ edges and maximum degree $d=7$.

Figure~\ref{fig:Comparison} compares the performance of the proposed single-step Whittle likelihood estimator with a two-step baseline method (square root). The matrix $L^{\ast}$ is an IEEE 33-bus power distribution network and $X_{t}$ is a Gaussian VAR(1) stochastic injection with diagonal auto-covariance: $\Phi_X(l) = \rho^{|l|} I$, with $\rho = 0.1$ and $l = \{1, \ldots, n-1\}$. The single-step approach estimates $L^{\ast}$ from samples of $Y_{t}$ as described in earlier experiments.

In contrast, the two-step procedure first estimates the inverse spectral density matrix $\Theta_Y(\omega)$ from samples of $Y_{t}$ and then computes its positive definite square root to estimate $L^{\ast}$. In this experiment, we fix the frequency at $\omega = 0$, where $\Theta_Y(0) = {L^{\ast}}^2 K.I$, where $K$ is some constant and $I$ is the identity matrix. In more general settings where $\Theta_X(\omega)$ is non-diagonal, the baseline would compute $\widehat{L} = \widehat{\Theta}_Y \Theta_X^{-1/2}$.

Panels (a) and (b) show the average F-score and Frobenius norm error, respectively, as functions of sample size $n$, averaged over 50 trials. The single-step estimator recovers the structure with fewer samples and achieves lower error compared to the two-step approach, thereby highlighting its superior performance over the baseline approach. As $\Theta_Y$ has degree $d^2$ (presence of two-hop neighbors) versus $d$ for $L^{\ast}$, Theorem~\ref{thm: Gaussian time series} implies that the two-step method requires $O(d^6 \log p)$ samples as compared to $O(d^3 \log p)$ for the proposed approach.

Figure \ref{fig: Benchmark networks} shows the F-score and element-wise $\ell_{\infty}$-norm of the error versus the rescaled sample size. For benchmark networks with varying sizes $p$ and maximum degrees $d$, there is a sharp increase in the F-score when the sample size is $n/(d^3\log p)\approx 1$, thus validating the sample complexity of $n = \mathcal{O}(d^3 \log p)$ as suggested by Theorem \ref{thm: Gaussian time series}. This sharp increase in F-score is consistent across different benchmark networks with differing size $p$ and maximum degree $d$. Similarly, across the benchmark networks, the element-wise $\ell_{\infty}$-norm of the error decreases sharply at $n/(\log p)\approx 1$.

\begin{figure}
    \centering
    \includegraphics[width=0.7\linewidth]{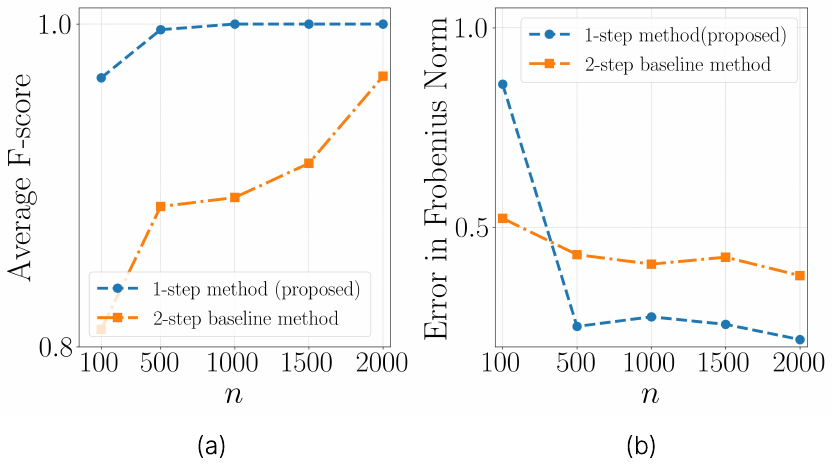} 
    \caption{Performance comparison between the proposed single-step Whittle likelihood estimator and a two-step baseline method on the IEEE 33-bus power distribution network under VAR(1) stochastic injection with diagonal auto-covariance structure $\Phi_{X}(l) = \rho^{|l|}I$ ($\rho = 0.1$). Panel (a) shows the average F-score versus sample size $n$, and panel (b) shows the average Frobenius norm error versus $n$. The single-step estimator achieves perfect structure recovery with fewer samples and exhibits faster error decay compared to the two-step approach, thereby signifying better performance. All results are averaged over $50$ independent trials. }
    \label{fig:Comparison}
\end{figure}

\begin{figure}
    \centering
    \includegraphics[width=0.7\linewidth]{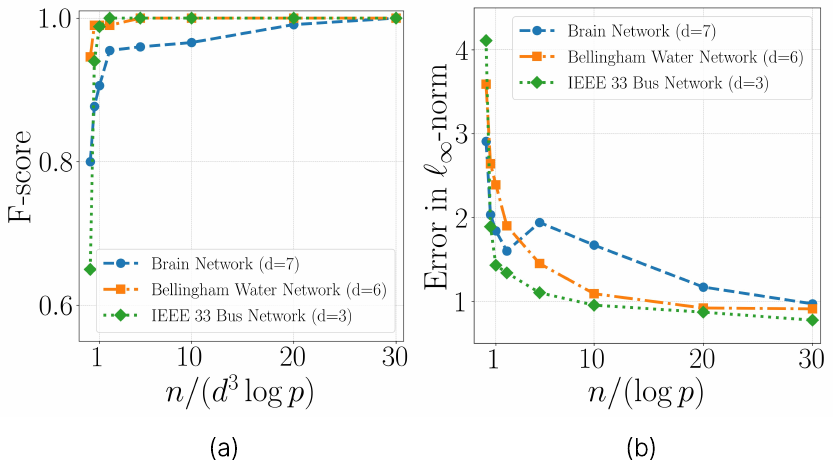} 
    \caption{(a) F-score versus rescaled sample size $(n/({d^3\log p}))$ across different benchmark networks. (b) Element-wise $\ell_{\infty}$-norm of the error versus rescaled sample size $(n/\log p)$ for the same networks. Both panels compare the human brain structural connectivity network (size $p=90$), Bellingham water network ($p=120$), and IEEE 33 bus power distribution network ($p=33$).
}
    \label{fig: Benchmark networks}
\end{figure}

\subsection{Real world brain network} \label{subsec: real world brain network}

We aim to estimate the brain networks for the control and autism groups using fMRI data (obtained under resting-state conditions) from the Autism Brain Imaging Data Exchange (ABIDE) dataset\footnote{\url{https://fcon_1000.projects.nitrc.org/indi/abide/}}. The pre-processed dataset is accessible\footnote{\url{http://preprocessed-connectomes-project.org/abide/}}, we refer to \cite{pruttiakaravanich2020convex} for more details. For each subject, we have access to $249$ samples of time series measurements across 90 anatomical regions of interest (ROIs) that result in a data matrix, $\{Y_t\}_{t=1}^{249} \in \mathbb{R}^{90}$. We collect such measurements for 86 subjects (46 from the autism group and 40 from the control group), from \url{https://github.com/jitkomut/cvxsem}. 

\begin{figure*}[!t]
    \centering
    \includegraphics[width=\textwidth]{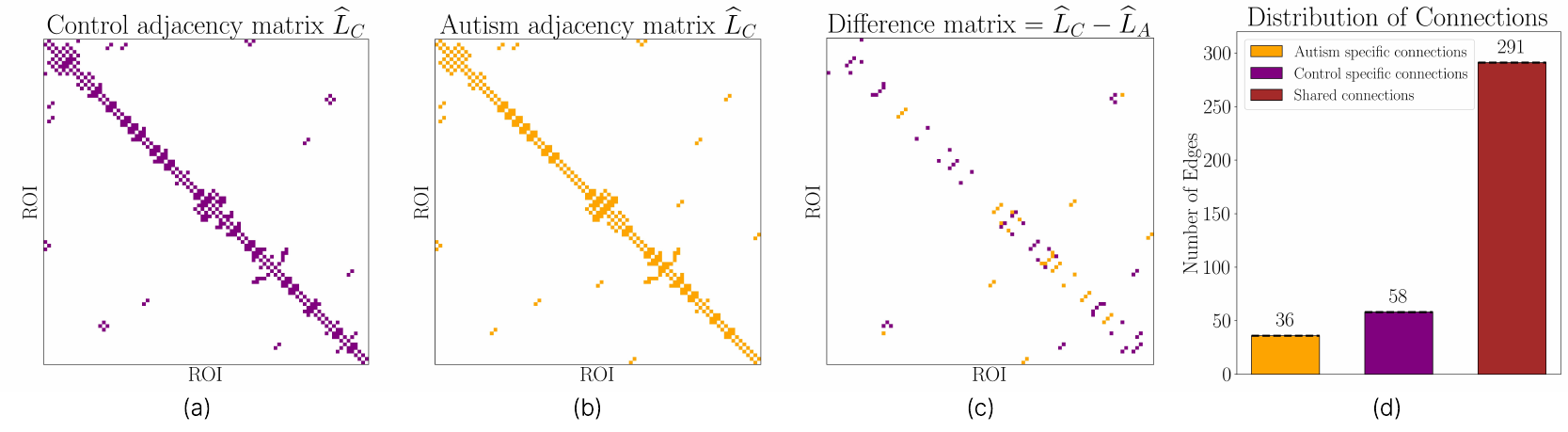}
    \caption{The results here are obtained using a fixed regularization parameter of $\lambda_{n} = 0.23$. Each dot in the heatmaps represents a statistically significant edge, i.e., an edge present in more than $90\%$ of the subjects. Panels (a) and (b) display the heatmaps of the estimated common adjacency matrices for the control group ($\widehat{L}_{C}$) and autism group ($\widehat{L}_{A}$), respectively, while panel (c) illustrates the difference matrix, $\widehat{L}_{C} - \widehat{L}_{A}$. This difference matrix captures both control-specific and autism-specific connections. Panel (d) provides a bar plot representing the distribution of connections, detailing the number of group-specific and shared connections. The bar plot indicates that the control adjacency matrix is denser than that of the autism group.}
    \label{fig: Brain Heatmap}
\end{figure*}

Using this dataset, we estimate a common brain network for each group: one for the control group (among $40$ subjects) and one for the autism group (among $46$ subjects). The common networks are constructed by identifying the \textit{statistically significant edges} (to be defined later) present across subjects in each group. While our goal is to evaluate the common brain network estimates against the ground truth using metrics like the F-score and Frobenius norm, this is not possible since the true network $L^{\ast}$ is unknown for both groups. Instead, we analyze the relative similarities and differences between the estimated common networks for the control and autism groups.

We begin the experiment by modeling the autocovariance matrix of the noise $\{X_t\}$ as $\Phi_X(l) = \rho^{|l|}I_{p}$ with $\rho = 0.1$, $l = \{1, \ldots, 248\}$ and $I_{p}$ is the p-dimensional identity matrix. The noise $\{X_t\}$ is therefore a WSS process. The PSD matrix $f_X(\omega) = D^2$ is computed as the Fourier transform of the autocovariance function $\Phi_{X}(l)$ at $\omega = 0$. Our estimator is then applied with regularization $\lambda_{n} = 0.23$  (tuned via grid search) across all 86 subjects. The common brain networks for each group are then constructed by retaining the statistically significant edges, that is, the edges that appear in over 90\% of the subjects.
 
Figure \ref{fig: Brain Heatmap} (a,b) illustrates the sparsity pattern of the estimated common adjacency matrix for the control group ($\widehat{L}_{C}$) and the autism group ($\widehat{L}_{A}$) brain networks. Each colored point in Figure \ref{fig: Brain Heatmap} (a,b) represents a statistically significant edge. We observe that the estimated adjacency matrix for both groups exhibits sparsity as proposed in \cite{hagmann2008mapping, bassett2006small}. In Figure \ref{fig: Brain Heatmap} (c), we plot the difference matrix $\widehat{L}_{C} - \widehat{L}_{A}$ to highlight control-specific connections, indicating more connections in the control group than in the autism group. Furthermore, we identify connections that are unique to each group as well as shared across groups. Figure \ref{fig: Brain Heatmap} (d) displays a bar plot of the distribution of the group-specific and shared connections, showing that while both groups share numerous connections, the control group exhibits greater connectivity, suggesting a denser network compared to the autism group. This sparsity trend persists for values of $\lambda_n$ between $0.1$ and $0.23$. For values below $0.1$, the estimated networks become too dense to support any meaningful conclusions. Similarly, for values above $0.23$, the networks become overly sparse and lack interpretability. At $\lambda_n = 0.23$, the estimated network recovers several connections reported in the literature. 

In  Appendix~\ref{app:roi}, we list all estimated neural connections present only in the control group. Table~\ref{table: Control specific connections} links these control-specific connections to well-established cognitive functions, including social interaction, face and image recognition, working memory, and language comprehension. Each of these findings is supported by prior neuroscience literature cited in Table~\ref{table: Control specific connections}.

\section{Parallels with other structure learning problems}\label{sec: connections}

In this section, we loop back to emphasize the generality of the network learning framework considered in this paper. Towards this, we present four examples here that fit well into the framework presented in \eqref{eq: intro_problem}. It is worth noting that many of these assume that $\{Y_t\}_{t\in\mathbb{Z}}$ is i.i.d.; so $f_Y(\omega)$ is constant. However, we allow for ${\{Y_t\}_{t\in\mathbb{Z}}}$ to be a WSS process (which subsumes the i.i.d. case); that is, we do not require $f_Y(\omega)$ to be a constant.

\smallskip 
\textit{1) Graph signal processing (GSP)} extends classical signal processing by analyzing signals supported on a graph. For random signals, a simple generative model is $Y_t=H(\alpha)X_t$. Here $X_t$ is white noise and $H(\alpha)=\sum_{k=0}^{K-1}\alpha_k S^k$ is the graph filter for a given $\alpha_k$ and $K$. The shift matrix $S$ (e.g., adjacency or Laplacian) encodes the edge connectivity of the graph. \cite{mateos2019connecting} discusses several methods to infer sparsity pattern of $S$ from finitely many observations of $Y_t$ for a variety of loss functions $\mathcal{L}[\cdot]$. Note that when $K\to \infty$, $\alpha_k=1$, and $S=L-I$, we have\footnote{The invertibility of the Laplacian matrix $L$ is discussed in Remark \ref{rem: invertibility assumption}.} $H(\alpha)=(I-S)^{-1}=L^{-1}$. Thus, $f_Y(\omega)=H(\alpha)f_X(\omega)H(\alpha)^\transpose$ becomes the constraint in our learning problem in \eqref{eq: whittle likelihood_SB}. 
    
\smallskip  
\textit{2) Structural equation models (SEMs)} are used to model cause-and-effect relationships between variables, allowing us to infer the causal structure of systems in medicine, economics, and social sciences. Networks generated by SEMs, including directed acyclic graphs are of great interest \cite{drton2017structure}.

A random vector $Y_t \in \mathbb{R}^p$ follows linear SEM if $Y_t=B^\transpose Y_t+X_t$. The path (or autoregressive) matrix $B$ is upper triangular---a structure essential for modeling causal relationships. Therefore, we can take $L=I-B^\transpose$ in \eqref{eq: whittle likelihood_SB} to reproduce this problem setup. However, our theoretical results need to be suitably adapted to handle a non-symmetric matrix $L$ needed for SEMs, and we leave this for future work. 

\smallskip 
\textit{3) Cholesky decomposition for correlation networks}: Let $Y_t\sim \mathcal{N}(0,\Sigma)$. The sparsity pattern of $\Sigma$ or the inverse $\Omega=\Sigma^{-1}$ allows us to construct the correlation and partial correlation networks, respectively \cite{pourahmadi2011covariance}. Learning sparse covariance or inverse covariance matrices has been well-studied (see Section \ref{sec: lit-rev}).

However, for a clear statistical interpretation, one wants to learn the underlying Cholesky matrices $T$ or $W$, where $\Sigma = TD_1T^\transpose$ or $\Omega = WD_2W^\transpose$. The sparse triangular matrices $ T $ and $W$ can be learned using our framework in \eqref{eq: l_1 regularized Whittle likelihood} by letting $f_X(\omega) = D$ and $L^{\ast} = W^{-1}$. However, our approach is more general and does not constrain $L^{\ast}$ to be triangular.

\smallskip 
\textit{4) Factor analysis (FA)} is a statistical method that discovers latent structures within high-dimensional data and is used in Finance and Psychology. The fundamental FA equation is $\tilde{X}_t=\Lambda Y_t+\Phi U_t$. Here $Y_t$ and $U_t$ are called the common and unique factors; and $\Lambda$ (loading) and $\Phi$ (diagonal) are parametric matrices \cite[Chapter 5]{trendafilov2021multivariate}. Assuming the contribution from the unique factor is known, define $X_t\triangleq \tilde{X}_t-\Phi U_t=\Lambda Y_t$, where $\Lambda$ plays the role of $L^{\ast}$. Then by treating $\tilde{X}_t$ as a latent random signal, we can use the estimator in \eqref{eq: whittle likelihood_SB} to learn $\Lambda$.

\section{Conclusion and Future Work}\label{section: Discussion and Future Work}
 We study the structure learning problem in systems obeying conservation laws under wide-sense stationary (WSS) stochastic injections. This problem appears in domains like power, the human brain, finance, and social networks. We propose a novel $\ell_1$-regularized (approximate) Whittle likelihood estimator to solve the network learning problem for WSS injections that include Gaussian and a few classes of non-Gaussian processes. Our theoretical analysis demonstrates that the estimator is convex and has a unique minimum in the high-dimensional regime. We establish sample complexity guarantees for recovering the sparsity structure of $L^{\ast}$, along with norm-consistency bounds (that is, estimation error computed using element-wise maximum, Frobenius, and operator norms). We validate our theoretical results on synthetic, benchmark, and real-world networks under VAR(1) and VARMA(2,2) injections. 

\medskip  

We identify three significant future extensions. First, deriving minimax lower bounds to establish the statistical optimality of our estimator building upon the tools developed in \cite{ying2021minimax}. Second, the work in \cite{kumar2020unified} showed that incorporating diagonal dominance and non-positive off-diagonal constraints of Laplacian matrices could improve the estimation performance for precision matrices modeled as Laplacians. Thus, it would be interesting to exploit such constraints into the estimator in \eqref{eq: whittle likelihood_SB}, and also to relax the symmetry assumption. Non-symmetric Laplacian matrices model directional flows and appear in many fields like transportation, hydrodynamics, and neuronal networks; see \cite{van2017modeling}. 

Finally, we could broaden the class of distributions considered for the nodal injection process $X_{t}$. Although we model $X_{t}$ as a WSS process, non-stationarity often arises in applications such as task-based fMRI signals in neuroscience \cite{preti2017dynamic} and stock market data, which is frequently modeled by Brownian or Lévy processes \cite{peng2024financial, engelke2024evy}. Characterizing sample complexity results for non-stationary processes is challenging and much work needs to be done. 

\section*{Acknowledgment}
This work was supported in part by the National Science Foundation (NSF) award CCF-2048223 and the National Institutes of Health (NIH) under the award 1R01GM140468-01. D. Deka acknowledges the funding provided by LANL’s Directed Research and Development (LDRD) project: “High-Performance Artificial Intelligence” (20230771DI).

%%%%%%%%%%%%%%%%%%%%%%%%%%%%%%%%%%%%%%%%%%%%%%%%%%%%%%%%%%%%

{\small \putbib[references_main]} 
\end{bibunit}

\newpage
\appendix
\counterwithin{equation}{section}
\counterwithin{theorem}{section}

\begin{bibunit}[IEEEtran]
\renewcommand{\bibname}{References for Appendix}
\renewcommand{\thesection}{\Alph{section}} 
\setcounter{section}{0}  
\setcounter{subsection}{0}  

    \begin{center}
        {\Large \bfseries Learning Network Structures from Wide-Sense Stationary Processes: Supplemental Material}\\
        \vspace{0.3cm}  % Adjust the vertical space as needed
        Anirudh Rayas, Jiajun Cheng, Rajasekhar Anguluri, Deepjyoti Deka, and Gautam Dasarathy
        \vspace{0.5cm}
    \end{center}
    
We restate all theorems, lemmas, and corollaries with their original numbering consistent with the main text. For any new numbered environments introduced exclusively in the appendix, we prefix the labels with "A" (e.g., Lemma A.1). We use $\det(A)$ or $|A|$ to denote the determinant of matrix $A$. 

\section{Proofs of all technical results}\label{subsection: Proofs of all technical results}
 After giving a brief overview of the problem set-up and the necessary assumptions, we provide proof for all the technical results. Recall that our observation model is $Y_t = {L^{\ast}}^{-1}X_t$, where $L^{\ast}$ is a $p\times p$ Laplacian matrix (which encodes network structure, that is, $L^{\ast}_{ij} = 0$ for all $(i,j)\in E^{c}$); $X_t\in \mathbb{R}^{p}$ is a wide sense stationary stochastic (WSS) process with a spectral density matrix $f_{X}(\omega_{j})$ and $Y_t\in \mathbb{R}^{p}$ is a random vector of node potentials. Given $n$ samples of $\{Y_{t}\}_{t\in \mathbb{Z}}$ our goal is to learn the sparsity structure of the matrix $L^{\ast}$. We propose the following $\ell_{1}$-regularized Whittle likelihood estimator $\widehat{L}_{j}$ to obtain a sparse estimate of $L^{\ast}$: 
\begin{align}\label{appeq: l_1 regularized whittle likelihood}
    \hspace{-1.5mm}\widehat{L}_j = \argmin_{L\succ 0}\,\,\Tr(DLP_{j}LD)\!-\!\log |L^{2}|\!+\!\lambda_{n}\|L\|_{1,\text{off}},
\end{align}
where $D\in \mathbb{R}^{p\times p}$ is the unique Hermitian positive definite square root matrix of $\Theta_{X}$, and $P_{j}=P(\omega_{j})$ is the averaged periodogram. Hereafter, we refer to $P_j$ and $\widehat{L}_{j}$ as $P$ and $\widehat{L}$, respectively, since our results hold for all $\omega_j \in \mathcal{F}_n$. We recall the assumptions to prove our results.
\vspace{10px}

\noindent\textbf{[A1] Mutual incoherence condition:} Let $\Gamma^{\ast}$ be the Hessian of the log-determinant in \eqref{appeq: l_1 regularized whittle likelihood}:  
\begin{align}\label{appeq: Hessian}
    \Gamma^{\ast} \triangleq \nabla_{L}^{2}\log\det(L)\vert_{L=L^{\ast}} = {L^{\ast}}^{-1}\otimes {L^{\ast}}^{-1}. 
\end{align}
We say that the Laplacian $L^*$ satisfies the mutual incoherence condition if $\vertiii{\Gamma^{\ast}_{E^{c}E}{\Gamma^{\ast}_{EE}}^{-1}}_{\infty}\leq 1-\alpha$, for some $\alpha\in (0,1]$.

The incoherence condition on $L^\ast$ controls the influence of irrelevant variables (elements of the Hessian matrix restricted to $E^c\times E$ on relevant ones (elements restricted to $E\times E$). The $\alpha$-incoherence assumption, commonly used in the literature, has been validated for various graphs like chain and grid graphs \cite{ravikumar2011high_app}. While $\alpha$-incoherence in \cite{ravikumar2011high_app,deb2024regularized_app} is imposed on the inverse covariance or spectral density matrix, we enforce it on \(L^\ast\). 
%A similar condition has also been explored in \cite{rayas_learning_2022}.
\smallskip 

\noindent\textbf{[A2] Bounding temporal dependence:} $\{Y_{t}\}_{t\in\mathbb{Z}}$ has short range dependence: $\sum_{l=-\infty}^{\infty}\|\Phi_{Y}(l)\|_{\infty}<\infty$.
Thus, the autocorrelation function $\Phi_Y(l)$ decreases quickly as the time lag $l$ increases, leading to negligible temporal dependence between samples that are far apart in time.

This mild assumption holds if the nodal injections $\{X_{t}\}_{t\in\mathbb{Z}}$ exhibits short range dependence: $\sum_{l=-\infty}^{\infty}\|\Phi_{X}(l)\|_{\infty}<\infty$. In fact, $\sum_{l=-\infty}^{\infty}\|\Phi_{Y}(l)\|_{\infty} = \sum_{l=-\infty}^{\infty}\|{L^\ast}^{-1}\Phi_{X}(l){L^\ast}^{-1}\|_{\infty}\leq\nu_{{L^\ast}^{-1}}^{2}\sum_{l=-\infty}^{\infty}\|\Phi_{X}(l)\|_{\infty}<\infty$, where $\nu_{{L^\ast}^{-1}}$ is the $\ellinf$ matrix norm of ${L^\ast}^{-1}$. 

\smallskip

\noindent \textbf{[A3] Condition number bound on the Hessian:} The condition number $\kappa(\Gamma^{\ast})$ of the Hessian matrix in \ref{appeq: Hessian} satisfies: 
\begin{align}\label{appeq: Condition number bound}
    \kappa(\Gamma^{\ast})\triangleq \vertiii{{\Gamma^{\ast}}}_{\infty}\vertiii{{\Gamma^{\ast}}^{-1}}_{\infty}\leq \frac{1}{4d\nu_{D^{2}}\|\Theta^{-1}_{Y}\|_{\infty}C_{\alpha}}, 
\end{align}
where $C_{\alpha} = 1+\frac{24}{\alpha}$, $\alpha\in (0,1]$, and $d$ is the maximum number of non-zero entries across all rows in $L^{\ast}$ (or equivalently the maximum degree of the network underlying $L^*$). Bounding $\kappa(\Gamma^{\ast})$ to derive estimation consistency results is standard in the high-dimensional graphical model literature \cite{cai2011constrained_app, rothman2008sparse_app}.

\smallskip 
We employ the primal-dual witness (PDW) construction to validate the behavior of the estimator $\widehat{L}$. The PDW technique involves the construction of a primal-dual pair $(\widetilde{L},\widetilde{Z})$, where $\widetilde{L}$ represents the optimal primal solution defined as the minimum of the following restricted $\ell_1$-regularized problem:
\begin{align}\label{appeq: rest whittle}
\hspace{-2.5mm}\widetilde{L} \triangleq \argmin_{L\succ 0,L_{E^{c}}=0}\left[\Tr(DLPLD)\!-\!\log|L^{2}|\!+\!\lambda_{n}\|L\|_{1,\text{off}}\right],
\end{align}
where $\widetilde{Z}\in \partial \|\widetilde{L}\|_{1,\text{off}}$ denotes the optimal dual solution. By definition, the primal solution $\widetilde{L}$ satisfies $\widetilde{L}_{E^{c}} = L^{\ast}_{E^{c}}=0$. Further, $(\widetilde{L},\widetilde{Z})$ satisfies the zero gradient conditions of the restricted problem $\eqref{appeq: rest whittle}$. Therefore, when the PDW construction succeeds, the solution $\widehat{L}$ is equal to the primal solution $\widetilde{L}$, ensuring the support recovery property, i.e., $\widehat{L}_{E^{c}}=0$.

\smallskip 

\textbf{Key Technical Contributions:}  We show that the 

estimator in \eqref{appeq: l_1 regularized whittle likelihood} is convex and admits a unique solution $\widehat{L}$ (Lemma \ref{applma: convexity and uniqueness}). 
 We derive sufficient conditions under which the PDW construction succeeds (Lemma \ref{applma: Sufficiency of strict dual feasibility}).  
 We then guarantee that the remainder term $R(\Delta)$ is bounded if $\Delta$ is bounded (see Lemma \ref{applma: control remainder}). 
 Furthermore, for a specific choice of radius $r$ as a function of $\Vert W\Vert_{\infty}$, we show that $\Delta$ lies in a ball $\mathbb{B}_{r}$ of radius $r$ (see Lemma \ref{applma: Control of Delta}). 
 Using known concentration results on the averaged periodogram for Gaussian and linear processes, we derive sufficient conditions on the number of samples required for the proposed estimator $\widehat{L}$ to recover the exact sparsity structure of $L^{\ast}$. We also show that under these sufficient conditions $\widehat{L}$ is consistent with $L^{\ast}$ in the element-wise $\ellinf$-norm and achieves sign consistency if $\vert L^{\ast}_{\min}\vert$ (the minimum non-zero entries of $L^{\ast}$) is lower bounded (see Theorem \ref{appthm: Gaussian time series} and Theorem~\ref{appthm: Linear Process support recovery}). Finally, we show that $\widehat{L}$ is consistent in the Frobenius and spectral norm.

%\textbf{Numbering convention}: 

\setcounter{appendixlemma}{0}
\label{app: convexity and uniqueness}
\begin{appendixlemma}\label{applma: convexity and uniqueness}{{\emph(Convexity and uniqueness)}}:
For any $\lambda_{n}\!>\!0$ and $L\!\succ\! 0$, if all the diagonals of the averaged periodogram $P_{ii}>0$, then (i) the $\ell_{1}$-regularized Whittle likelihood estimator in \eqref{appeq: l_1 regularized whittle likelihood} is strictly convex and (ii) $\widehat{L}$ in \eqref{appeq: l_1 regularized whittle likelihood} is the unique minima satisfying the sub-gradient condition $2\Psi_{1}\widehat{L}P_{1} - 2\Psi_{2}\widehat{L}P_{2}- 2\widehat{L}^{-1}\!+\!\lambda_{n}\widehat{Z}\!=\!0$, where $\widehat{Z}\in \partial\Vert L \Vert_{1,\text{off}}$ is evaluated at $\widehat{L}$.
\end{appendixlemma}
\begin{proof}
%    Here, we present a sketch of the proof for the convexity and uniqueness of the estimator \eqref{eq: l_1 regularized Whittle likelihood}. 
    The proof follows the same argument as in \cite[Lemma 1]{rayas_learning_2022_app}, but needs to account for complex-valued matrices. To show convexity, we rewrite the objective in \eqref{appeq: l_1 regularized whittle likelihood} as
    \begin{align}\label{appeq: l_1 likelihood in norm}
      \mathcal{L}_\lambda(L)\triangleq\|DLM\|^{2}_{F} - 2\log\det(L) + \lambda_{n}\|L\|_{1,\text{off}},
    \end{align}
    where $M$ is the unique positive semidefinite square root of the averaged periodogram $P$.  Then the objective \eqref{appeq: l_1 likelihood in norm} is strictly convex since the Frobenius norm $\|DLM\|^{2}_{F}$ is strictly convex and $\log\det(L)$ is convex for any positive definite $L\succeq 0$. However, this does not guarantee that the estimator is unique since strictly convex functions have unique minima, \emph{if} attained \cite{boyd2004convex_app}. To show that the minima is attained it is sufficient if the convex objective \eqref{appeq: l_1 likelihood in norm} is coercive (see Def 11.10 and Proposition 11.14 in \cite{bauschke2011convex}). The proof of coercivity follows along the same lines as that provided in \cite{rayas_learning_2022_app} (see proof of Lemma 1), with the exception that the matrices $D$ and $M$ are complex-valued. It remains to show the sub-gradient condition of the $\ell_{1}$-regularized Whittle likelihood estimator in \eqref{appeq: l_1 regularized whittle likelihood}. The sub-gradient satisfies
    \begin{align}
        \frac{\partial}{\partial L}[\Tr(D^{2}LPL)-2\log\det(L)+\|L\|_{1,\text{off}}]|_{L=\widehat{L}} = 0.
    \end{align}
Let $D^{2} = \Psi_{1}+\im \Psi_{2}$, where $\Psi_{1} = \mathfrak{R}(D^{2})$ is the real part of $D^{2}$ and $\Psi_{2} = \mathfrak{I}(D^{2})$ is the imaginary part of $D^{2}$. Similarly, let $P = P_1+\im P_2$. Then
\begin{align*}
    \Tr((\Psi_{1}+\im \Psi_{2})L(P_1&+\im P_2)L) = \Tr(\Psi_{1}LP_{1}L - \Psi_{2}LP_{2}L)
    \\&+\im [Tr(\Psi_{2}LP_{1}L)+\Tr(\Psi_{1}LP_{2}L)].
\end{align*}
Since $D^2$ and $P$ are Hermitian, it follows that $\Psi_{1} = \mathfrak{R}(D^{2})$ and $\mathfrak{R}(P)$ are symmetric and the imaginary part $\Psi_{2}=\mathfrak{I}(D^{2})$ and $\mathfrak{I}(P)$ are skew-symmetric. As a result
\begin{align*}
    \hspace{-1.5mm}\frac{\partial}{\partial L}[\Tr(D^{2}LPL)] & = \frac{\partial}{\partial L}[\mathfrak{R}(\Tr(D^{2}LPL))\!+\!\im \mathfrak{I}(\Tr(D^{2}LPL))] \\ 
    & = 2\Psi_{1}LP_{1}-2\Psi_{2}LP_{2} + \im 0,
\end{align*}
where in the second equality we used the matrix trace derivative result in \cite{petersen2008matrix} and the fact that $D^{2}$ and $P$ are skew-symmetric. 

The derivative of $\log \det(L)$ with respect to $L$ is $L^{-1}$. Finally, the sub-gradient of $\|L\|_{1,\text{off}}$ is given by
\begin{align*}
    \hspace{-1.5mm}\frac{\partial}{\partial L}\|L\|_{1,\text{off}}= \frac{\partial}{\partial L} \sum_{i\neq j}|L_{ij}|=
    \begin{cases}
        0 &  i=j \\
    \text{sign}(L_{ij}) & i\neq j, L_{ij}\neq 0\\
    \in [-1,1] & i\neq j, L_{ij} = 0.
    \end{cases}
\end{align*}
Putting all these pieces together, the sub-gradient condition of \eqref{appeq: l_1 regularized whittle likelihood} evaluated at $\widehat{L}$ is then given by
\begin{align}\label{appeq: sub-gradient condition}
   \partial \mathcal{L}_\lambda(L)\triangleq 2\Psi_{1}\widehat{L}P_{1} - 2\Psi_{2}\widehat{L}P_{2}- 2\widehat{L}^{-1}\!+\!\lambda_{n}\widehat{Z}=0,
\end{align}
where $\widehat{Z}\in \partial \|L\|_{1,\text{off}}$ is evaluated at $\widehat{L}$.
\end{proof}

\begin{appendixlemma}\label{applma: Sufficiency of strict dual feasibility}(Conditions for strict dual feasibility) Let $\lambda_{n}>0$ and $\alpha$ be defined as in \textbf{[A1]}. Suppose that $\max\{2\nu_{D^{2}}(d\|\Delta\|_{\infty}+\nu_{L^{\ast}})\|W\|_{\infty}, \|R(\Delta)\|_{\infty},2\nu_{D^{2}}d\|\Delta\|_{\infty}\|\Theta^{-1}_{Y}\|_{\infty}\}\leq \frac{\alpha\lambda_{n}}{24}$. Then the dual vector $\widetilde{Z}_{E^{c}}$ satisfies $\|\widetilde{Z}_{E^{c}}\|_{\infty}<1$, and hence, $\widetilde{L} = \widehat{L}$. 

\end{appendixlemma}
\begin{proof}
   We start by deriving a suitable expression for the sub-gradient $\widetilde{Z}_{E^{c}}$ by using the optimality condition of the restricted $\ell_{1}$-regularized problem defined in \eqref{appeq: rest whittle}. From \eqref{appeq: sub-gradient condition} we have,
   \begin{align}\label{eq: appendix sub-grad pwd}
         \partial \mathcal{L}_\lambda(\tilde{L})\triangleq 2\Psi_{1}\widetilde{L}P_{1} - 2\Psi_{2}\widetilde{L}P_{2}- 2\widetilde{L}^{-1}\!+\!\lambda_{n}\widetilde{Z}\!=\!0.
   \end{align}
where $\widetilde{L}$ is the primal solution given by \eqref{appeq: rest whittle} and $\widetilde{Z}\in \|\widetilde{L}\|_{1,\text{off}}$ is the optimal dual. Recall that the measure of distortion is given by $\Delta = \widetilde{L}-L^{\ast}$ and the measure of noise is given by $W = P - \Theta_{Y}^{-1}$. We have the following chain of equations:
\begin{align*}
    \partial \mathcal{L}_\lambda(\tilde{L})&=2(\Psi_{1}\Delta P_{1}+\Psi_{1}L^{\ast}P_1)\\&\quad -2(\Psi_{2}\Delta P_{2}+\Psi_{2}L^{\ast}P_2)
    -2\widetilde{L}^{-1}+\lambda_{n}\widetilde{Z} \\
    &=2(\Psi_{1}\Delta W_{1} + \Psi_{1}L^{\ast} W_{1} + \Psi_{1}\Delta \mathfrak{R}(\Theta_{Y}^{-1})\\&\quad  +\Psi_{1}L^{\ast} \mathfrak{R}(\Theta_{Y}^{-1}))
    -2(\Psi_{2}\Delta W_{2} + \Psi_{2}L^{\ast} W_{2}\\& \quad + \Psi_{2}\Delta \mathfrak{I}(\Theta_{Y}^{-1}) + \Psi_{2}L^{\ast} \mathfrak{I}(\Theta_{Y}^{-1}))
    -2\widetilde{L}^{-1}+\lambda_{n}\widetilde{Z}.
\end{align*}
Define the following terms: 
\begin{align}
    T_1&=\Psi_{1}\Delta W_{1} + \Psi_{1}L^{\ast} W_{1} + \Psi_{1}\Delta \mathfrak{R}(\Theta_{Y}^{-1})\\
    T_2&=\Psi_{2}\Delta W_{2} + \Psi_{2}L^{\ast} W_{2} + \Psi_{2}\Delta \mathfrak{I}(\Theta_{Y}^{-1})\\
    T_3&=\Psi_{1}L^{\ast} \mathfrak{R}(\Theta_{Y}^{-1}) -\Psi_{2}L^{\ast} \mathfrak{I}(\Theta_{Y}^{-1}), 
\end{align}
and note that 
\begin{align}\label{eq: appendix simplified optimality}
\partial \mathcal{L}_\lambda(\tilde{L})=2T_1-2T_2+2T_3-2\widetilde{L}^{-1}\lambda_n\tilde{Z}. 
\end{align}
We now show that $T_{3} = {L^{\ast}}^{-1}$. By definition $\Theta_{Y} = {L^{\ast}}\Theta_{X}{L^{\ast}}$, where $\Theta_Y$ and $\Theta_X$ are Hermitian positive definite matrices. So
$L^{\ast}\mathfrak{R}(\Theta_{Y}^{-1}) = \mathfrak{R}(\Theta_{X}^{-1}){L^{\ast}}^{-1}$ and $L^{\ast}\mathfrak{I}(\Theta_{Y}^{-1}) = \mathfrak{I}(\Theta_{X}^{-1}){L^{\ast}}^{-1}$. From these two identities, we establish that
\begin{align*}
    &\Psi_{1}L^{\ast}\mathfrak{R}(\Theta_{Y}^{-1})-\Psi_{2}L^{\ast}\mathfrak{I}(\Theta_{Y}^{-1}) \\
    & \hspace{30.0mm} =[\Psi_{1}\mathfrak{R}(\Theta_{X}^{-1})-\Psi_{2}\mathfrak{I}(\Theta_{X}^{-1})]{L^{\ast}}^{-1}. 
\end{align*}

To show $\Psi_{1}\mathfrak{R}(\Theta_{X}^{-1})-\Psi_{2}\mathfrak{I}(\Theta_{X}^{-1})=I$ proceed as follows. Recall that $D^{2} \triangleq \Psi_{1}+\im \Psi_{2} = \Theta_{X}$. Thus, $\Psi_{1}\Theta_{X}^{-1}+\im \Psi_{2}\Theta_{X}^{-1} = I$.
%where $I$ is the $p\times p$ identity matrix. 
Decompose $\Theta_{X}^{-1}$ into real and imaginary parts to see that $\Psi_{1}\mathfrak{R}(\Theta_{X}^{-1})-\Psi_{2}\mathfrak{I}(\Theta_{X}^{-1}) = I_{p\times p}$. %\rajmargin{This is too vague to be in a technical paper!}

Substituting $T_3=(L^*)^{-1}$ in \eqref{eq: appendix simplified optimality}, followed by algebraic manipulations (Taylor series of ${L^{\ast}}^{-1}$ around $\tilde{L}$), yield us
\begin{align}
   \partial \mathcal{L}_\lambda(\tilde{L})=T_{1} - T_{2} - R(\Delta) - {L^{\ast}}^{-1}\Delta{L^{\ast}}^{-1}+\lambda_{n}^{\prime}\widetilde{Z}, 
\end{align}
where $\lambda_n'=0.5\lambda$ and $R(\Delta) = \tilde{L}^{-1}-{L^{\ast}}^{-1}-{L^{\ast}}^{-1}\Delta{L^{\ast}}^{-1}$ is the remainder term of the Taylor series. Apply vec operator\footnote{We use $\mvec(A)$ or $\xbar{A}$ to denote the $p^2$ vector formed by stacking the columns of the $p\times p$-dimensional matrix $A$.} on both sides and use the relation in \eqref{eq: appendix sub-grad pwd} to finally obtain

\begin{align}\label{appeq: vectorized optimality}
    \mvec(-{L^{\ast}}^{-1}\Delta{L^{\ast}}^{-1}+T_{1}-R(\Delta)-T_{2}+\lambda^{\prime}_{n}\widetilde{Z})=0.
\end{align}

From standard Kronecker product matrix rules \cite{laub2005matrix}, we have $\mvec({L^{\ast}}^{-1}\Delta{L^{\ast}}^{-1}) = \Gamma^{\ast}\xbar{\Delta}$, where $\Gamma^{\ast} = {L^{\ast}}^{-1}\otimes{L^{\ast}}^{-1}$ and $\xbar{\Delta} = \mvec(\Delta)$. For compatible matrices $ A,B,C$, we have $\mvec(ABC) = \Gamma(AB)\xbar{C}$, where $\Gamma(AB) = I\otimes AB$ and $I$ is the $p\times p$ identity matrix. Using the Kronecker product rules, \eqref{appeq: vectorized optimality} becomes
\begin{align}\label{appeq: vectorized stationary eq}
\xbar{T}_{1}-\xbar{T}_{2}-\Gamma^{\ast}\xbar{\Delta}-\xbar{R}(\Delta)+\lambda_{n}^{\prime}\xbar{\widetilde{Z}}=0,
\end{align}
where
\begin{align}
   \hspace{-2mm}\xbar{T}_{1}&= (\Gamma(\Psi_{1}\Delta)\xbar{W}_{1}+\Gamma(\Psi_{1}L^{\ast})\xbar{W}_{1}+\Gamma(\Psi_{1}\Delta)\xbar{\mathfrak{R}}(\Theta_{Y}^{-1}))\label{appeq: T1 bar}\\
   \hspace{-2mm}\xbar{T}_{2}&=(\Gamma(\Psi_{2}\Delta)\xbar{W}_{2}+\Gamma(\Psi_{2}L^{\ast})\xbar{W}_{2}+\Gamma(\Psi_{2}\Delta)\xbar{\mathfrak{I}}(\Theta_{Y}^{-1})).\label{appeq: T2 bar}
\end{align}

We partition equation \eqref{appeq: vectorized stationary eq} above into two separate equations corresponding to the sets $E$ and $E^{c}$. Recall that $E$ is the augmented edge set defined as $E = \{\mathcal{E}(L^{\ast})\cup(1,1),\ldots,\cup (p,p)\}$, where $\mathcal{E}(L^{\ast})$ is the edge set of $L^{\ast}$ and $E^{c}$ is the complement of the set $E$. Recall that we use the notation $A_{E}$ to denote the sub-matrix of $A$ containing all elements $A_{ij}$ such that $(i,j)\in E$. We partition the above linear equation into two separate linear equations corresponding to the sets $E$ and $E^{c}$:
\begin{align}
   -\Gamma^{\ast}_{EE}\xbar{\Delta}_{E} + \xbar{T_1}_{E}-\xbar{T_2}_{E}-\xbar{R}_{E}(\Delta)+\lambda^{\prime}_{n}\xbar{\widetilde{Z}}_{E} &= 0\label{appeq: vectorized optimality wrt E}\\
   -\Gamma^{\ast}_{E^{c}E}\xbar{\Delta}_{E} + \xbar{T_1}_{E^{c}}-\xbar{T_2}_{E^{c}}-\xbar{R}_{E^{c}}(\Delta)+\lambda^{\prime}_{n}\xbar{\widetilde{Z}}_{E^{c}} &= 0,\label{appeq: vectorized optimality wrt E^{c}}
\end{align}
where the latter equation follows by definition $\Delta_{E^{c}}=0$. From \eqref{appeq: vectorized optimality wrt E} solving for $\xbar{\Delta}_{E}$ gives us
\begin{align}\label{appeq: Delta_E}
    \xbar{\Delta}_{E} =  (\Gamma^{\ast}_{EE})^{-1}\underbrace{\left[\xbar{T_1}_{E}-\xbar{T_{2}}_{E}-\xbar{R}_{E}(\Delta)+\lambda^{\prime}_{n}\xbar{\widetilde{Z}}_{E}\right].}_{\triangleq M}
\end{align}
Substituting for $\xbar{\Delta}_{E}$ in \eqref{appeq: vectorized optimality wrt E^{c}} we have
\begin{align}\label{appeq: Z_{E^{c}}}
  \hspace{-3mm} \xbar{T_1}_{E^{c}}\!-\!\xbar{T_2}_{E^{c}}\!-\!\Gamma^{\ast}_{E^{c}E}{\Gamma^{\ast}_{EE}}^{-1}M\!-\!\xbar{R}_{E^{c}}(\Delta)\!-\!\lambda^{\prime}_{n}\xbar{\widetilde{Z}}_{E^{c}} = 0.
\end{align}
Solving for $\xbar{\widetilde{Z}}_{E^{c}}$ in \eqref{appeq: Z_{E^{c}}} we get
\begin{align}
   \hspace{-2mm} \xbar{\widetilde{Z}}_{E^{c}}\leq \frac{1}{\lambda_{n}^{\prime}}\left[ \Gamma^{\ast}_{E^{c}E}{\Gamma^{\ast}_{EE}}^{-1}M \!-\!\xbar{T_1}_{E^{c}}\!-\!\xbar{T_2}_{E^{c}}+\xbar{R}_{E^{c}}(\Delta)\right].
\end{align}
From this inequality the element-wise $\ell_{\infty}$ norm is bounded as
\begin{align}\label{appeq: element-wise of Z_{E^{c}}}
    \|\xbar{\widetilde{Z}}_{E^{c}}\|_{\infty}&\leq \frac{1}{\lambda_{n}^{\prime}}[\vertiii{\Gamma^{\ast}_{E^{c}E}(\Gamma^{\ast}_{EE})^{-1}}_{\infty}\|M\|_{\infty}+\|\xbar{T}_{1}\|_{\infty}\nonumber\\&+\|\xbar{T}_{2}\|_{\infty}+\|\xbar{R}(\Delta)\|_{\infty}].
\end{align}
The term $M$ in \eqref{appeq: Delta_E}, with the facts that $\|A_{E}\|_{\infty}\leq \|A\|_{\infty}$ and $\|\xbar{\widetilde{Z}}_{E}\|_{\infty}\leq 1$, satisfies: 
\begin{align}\label{appeq: bound M}
    \|M\|_{\infty}{\leq} \underbrace{\|\xbar{T_{1}}\|_{\infty}+ \|\xbar{T_{2}}\|_{\infty} + \|\xbar{R}(\Delta)\|_{\infty}}_{\triangleq H} + \lambda_{n}^{\prime}.
\end{align}
%(a) since $\|\xbar{\widetilde{Z}}_{E}\|_{\infty}\leq 1$, 
Finally, from Assumption \textbf{[A1]}, we have
\begin{align}
    \|\xbar{\widetilde{Z}}_{E^{c}}\|_{\infty}&\leq \frac{1}{\lambda_{n}^{\prime}}\left[(1-\alpha)(H+\lambda_{n}^{\prime}) +H \right]\\
    &=(1-\alpha)+\frac{2-\alpha}{\lambda_{n}^{\prime}}H. \label{appeq: bound Z in H}
\end{align}
We now upper bound $H$. Recall from \eqref{appeq: bound M} we have
\begin{align}\label{appeq: upper bound H}
    H = \|\xbar{T_{1}}\|_{\infty}+ \|\xbar{T_{2}}\|_{\infty} + \|\xbar{R}(\Delta)\|_{\infty}. 
\end{align}
Substituting for $\xbar{T_{1}}$ and $\xbar{T_{2}}$ from equation \eqref{appeq: T1 bar} and \eqref{appeq: T2 bar} in equation \eqref{appeq: upper bound H} we have,
\begin{align*}
    H &= \|\Gamma(\Psi_{1}\Delta)\xbar{W_{1}}+\Gamma(\Psi_{1}L^{\ast})\xbar{W_{1}}+\Gamma(\Psi_{1}\Delta)\xbar{\mathfrak{R}}(\Theta^{-1}_{Y})\|_{\infty}\\
    &+\|\Gamma(\Psi_{2}\Delta)\xbar{W_{2}}+\Gamma(\Psi_{2}L^{\ast})\xbar{W_{2}}+\Gamma(\Psi_{2}\Delta)\xbar{\mathfrak{I}}(\Theta^{-1}_{Y})\|_{\infty}\\
    &+\|\xbar{R}(\Delta)\|_{\infty}.
\end{align*}
From the sub-multiplicative property of $\|\cdot\|_\infty$-norm  
\begin{align*}
H&\leq \vertiii{\Gamma(\Psi_{1}\Delta)}_{\infty}\|W_{1}\|_{\infty}+ \vertiii{\Gamma(\Psi_{1}L^{\ast})}_{\infty}\|W_1\|_{\infty}\\&+\vertiii{\Gamma(\Psi_{1}\Delta)}_{\infty}\|\mathfrak{R}(\Theta^{-1}_{Y})\|_{\infty}
    +\vertiii{\Gamma(\Psi_{2}\Delta)}_{\infty}\|W_{2}\|_{\infty}\\&+ \vertiii{\Gamma(\Psi_{2}L^{\ast})}_{\infty}\|W_2\|_{\infty}+\vertiii{\Gamma(\Psi_{2}\Delta)}_{\infty}\|\mathfrak{I}(\Theta^{-1}_{Y})\|_{\infty}\\
    &+\|R(\Delta)\|_{\infty}.
\end{align*}
Further, $\max\{\vertiii{\mathfrak{R}(A)}_{\infty},\vertiii{\mathfrak{I}(A)}_{\infty}\}\leq \vertiii{A}_{\infty}$ for any $A\in \mathbb{C}^{p\times p}$. Thus, 
\begin{align}
    H &\leq 2\Big[\left(\vertiii{\Gamma(D^{2}\Delta)}_{\infty}+\vertiii{\Gamma(D^{2}L^{\ast})}_{\infty}\right)\|W\|_{\infty}\nonumber\\&+\vertiii{\Gamma(D^{2}\Delta)}_{\infty}\|\Theta^{-1}_{Y}\|_{\infty}]+\|R(\Delta)\|_{\infty} \Big]. 
\end{align}
Once again using the sub-multiplicative property of $\ell_{\infty}$-norm,
\begin{align}
    H&\leq 2\Big[(\nu_{D^{2}}\vertiii{\Delta}_{\infty}+\nu_{D^{2}}\nu_{L^{\ast}})\|W\|_{\infty}\nonumber\\&+\nu_{D^{2}}\vertiii{\Delta}_{\infty}\|\Theta^{-1}_{Y}\|_{\infty}+\|R(\Delta)\|_{\infty}\Big]\nonumber  \\
    &\overset{(a)}{\leq}2\Big[\nu_{D^{2}}(d\|\Delta\|_{\infty}+\nu_{L^{\ast}})\|W\|_{\infty}\nonumber\\&+\nu_{D^{2}}d\|\Delta\|_{\infty}\|\Theta^{-1}_{Y}\|_{\infty} + \|R(\Delta)\|_{\infty}  \Big]=H^{\prime}, \label{eq: appendix h prime}
\end{align}
where (a) follows from $\vertiii{\Delta}_{\infty}\leq d\|\Delta\|_{\infty}$ since there are at most $d$ non-zeros in every row in $\Delta$. Using the above bound in \eqref{eq: appendix h prime}, expression in \eqref{appeq: bound Z in H} becomes
\begin{align}
    \|\xbar{\widetilde{Z}}_{E^{c}}\|_{\infty}\leq (1-\alpha)+\frac{2-\alpha}{\lambda_{n}^{\prime}}H^{\prime}. 
\end{align}
Setting $H^{\prime}\leq \frac{\alpha\lambda_{n}^{\prime}}{4}$, we can conclude that $\|\xbar{\widetilde{Z}}_{E^{c}}\|_{\infty}<1$ (the strict dual feasibility condition) in the following way:
\begin{align*}
    \|\xbar{\widetilde{Z}}_{E^{c}}\|_{\infty}&\leq (1-\alpha)+\frac{2-\alpha}{\lambda_{n}^{\prime}}H^{\prime}\\
    &\leq (1-\alpha)+\frac{2-\alpha}{\lambda_{n}^{\prime}}\left(\frac{\alpha\lambda_{n}^{\prime}}{4}\right)\\
    &\leq (1-\alpha)+\frac{\alpha}{2}<1.
\end{align*}
This concludes the proof.
\end{proof}

\noindent The following lemma shows that the remainder term $R(\Delta)$ is bounded if $\Delta$ is bounded. The proof is adapted from \cite{ravikumar2011high_app, rayas_learning_2022_app}, where a similar result is derived using matrix expansion techniques. We omit the proof and refer the readers to \cite{ravikumar2011high_app, rayas_learning_2022_app}. This lemma is used in the proof of our main results (see Theorem \ref{appthm: Gaussian time series} and Theorem \ref{appthm: Linear Process support recovery}) to show that with a sufficient number of samples, $\|R(\Delta)\|_{\infty} \leq \alpha\lambda_{n}/24$.

\label{app: control remainder}
\begin{appendixlemma}\label{applma: control remainder}
    Suppose the $\ell_{\infty}$-norm $\|\Delta\|_{\infty}\leq \frac{1}{3\nu_{{L^{\ast}}^{-1}}d}$, then $\|R(\Delta)\|_{\infty}\leq \frac{3}{2}d\|\Delta\|_{\infty}^{2}\nu_{{L^{\ast}}^{-1}}^{3}$.
\end{appendixlemma}
We show that for a specific choice of radius $r$, the distortion defined as $\Delta = \widetilde{L}-L^{\ast}$ lies in a ball of radius $r$.

\label{app: control distortion}
\begin{appendixlemma}\label{applma: Control of Delta}(Control of $\Delta$) Let 
\begin{align}
    r&\triangleq 8\nu_{{\Gamma^{\ast}}^{-1}}[\nu_{D^{2}}\nu_{L^{\ast}}\|W\|_{\infty}+0.25\lambda_{n}] \quad\text{be such that}\nonumber \\ r&\leq\min\Big\{\frac{1}{3\nu_{{L^{\ast}}^{-1}}d},\frac{1}{6\nu_{{\Gamma^{\ast}}^{-1}}\nu_{{L^{\ast}}^{-1}}^{3}d}\Big\}.
\end{align}
 Then the element-wise $\ell_{\infty}$-bound $\|\Delta\|_{\infty} = \|\widetilde{L}-L^{\ast}\|_{\infty}\leq r$.   
\end{appendixlemma}
\begin{proof}
   We adopt the proof techniques from \cite{ravikumar2011high_app,rayas_learning_2022_app}. Let $G(\widetilde{L})$ be the zero sub-gradient condition of the restricted $\ell_{1}$-regularized Whittle likelihood estimator given in \eqref{appeq: rest whittle}:   
   \begin{align}\label{eq: G function}
       G(\widetilde{L}) = \Psi_{1}\widetilde{L}P_{1} - \Psi_{2}\widetilde{L}P_{2} -\widetilde{L}^{-1}+\lambda_{n}^{\prime}\widetilde{Z} = 0.
   \end{align}
where $\{\Psi_{1},P_1\}=\mathfrak{R}\{D^{2},P\}$ is the real part of $D^{2}$ and $P$ respectively. Similarly, $\{\Psi_{2},P_2\}=\mathfrak{I}\{D^{2},P\}$ is the imaginary part of $D^{2}$ and $P$ respectively. Recall that $\widetilde{L}$ is the primal solution of the $\ell_{1}$-regularized Whittle likelihood given in \eqref{appeq: rest whittle}, $\widetilde{Z}\in \partial\|\widetilde{L}\|_{1,\text{off}}$ is the sub-gradient and $\lambda_{n}^{\prime}=0.5\lambda_{n}$ is the regularization parameter.

For any matrix $A$, let $\xbar{A}$ or $\mvec(A)$ denote the vectorization of $A$ obtained by stacking the rows of $A$ and let $A_{E}$ or $[A]_{E}$ denote the sub-matrix of $A$ containing all elements $A_{ij}$ such that $(i,j)\in E$. Recall that the goal is to establish that $\|\Delta\|_{\infty}\leq r$, towards this it suffices to show that $\|\Delta_{E}\|_{\infty}\leq r$ since $\Delta_{E^{c}}=0$ from the primal dual witness construction. Equivalently we show $\xbar{\Delta}_{E}\in \mathcal{B}_{r}\triangleq\{\xbar{A}\in \mathbb{R}^{|E|}:\|A\|_{\infty}\leq r\}$. Towards this end we define a continuous vector valued map $F:\mathbb{R}^{|E|}\rightarrow \mathbb{R}^{{|E|}}$, given by
\begin{align}
    F(\xbar{\Delta}_{E}) = -(\Gamma^{\ast}_{EE})^{-1}[\xbar{G}(\Delta+L^{\ast})]_{E}+\xbar{\Delta}_{E},
\end{align}
where $G(\cdot)$ is given by \eqref{eq: G function}. We first outline the strategy to show $\|\Delta_{E}\|_{\infty}\in \mathcal{B}_{r}$, with radius $r$ specified in the lemma. We use the following two key properties (i) $\widetilde{L}$ that satisifes $\xbar{G}(\widetilde{L})=0$ is unique (see Lemma~\ref{applma: convexity and uniqueness}), since $\widetilde{L} = \Delta+L^{\ast}$, we have that $\Delta$ satisfies $\xbar{G}(\Delta+L^{\ast})=0$, (ii) $F(\xbar{\Delta}_{E}) = \xbar{\Delta}_{E}$ if and only if $\xbar{G}(\cdot) = 0$. From (i) and (ii) we conclude that $F$ has a unique fixed point $\xbar{\Delta}_{E}$. Now suppose $F(\mathcal{B}_{r})\subseteq\mathcal{B}_{r}$ is a contraction, then by Brower's fixed point theorem there exists a point $C\in \mathcal{B}_{r}$ such that $F(C)=C$ ie. $C$ is a fixed point. Since $\xbar{\Delta}_{E}$ is a unique fixed point of $F$, it follows that $C=\xbar{\Delta}_{E}\in \mathcal{B}_{r}$. It remains to show that $F$ is a contraction on $\mathcal{B}_{r}$. Let $\Delta^{\prime}\in \mathbb{R}^{p\times p}$ be a zero padded matrix on $E^{c}$ such that $\xbar{\Delta^{\prime}}_{E}\in \mathbb{B}_{r}$. We show that $\|F(\xbar{\Delta^{\prime}}_{E})\|_{\infty}\leq r$. In fact,  
\begin{align}
   \hspace{-2mm}& F(\xbar{\Delta^{\prime}}_{E}) = -(\Gamma^{\ast}_{EE})^{-1}[\xbar{G}(\Delta^{\prime}+L^{\ast})]_{E}+\xbar{\Delta^{\prime}}_{E}\nonumber\\
    &=-(\Gamma^{\ast}_{EE})^{-1}[\mvec(\Psi_{1}(\Delta^{\prime}+L^{\ast})P_1-\Psi_{2}(\Delta^{\prime}+L^{\ast})P_2\nonumber\\&-(\Delta^{\prime}+L^{\ast})^{-1}+\lambda^{\prime}Z)]_{E}+\xbar{\Delta^{\prime}}_{E}\nonumber\nonumber\nonumber\\
    &\overset{(a)}{=}-{\Gamma^{\ast}_{EE}}^{-1}[\mvec(\Psi_{1}(\Delta^{\prime}+L^{\ast})W_1-\Psi_{2}(\Delta^{\prime}+L^{\ast})W_2\nonumber\nonumber\\&+\Psi_{1}\Delta^{\prime}\mathfrak{R}(\Theta^{-1}_{Y})-\Psi_{2}\Delta^{\prime}\mathfrak{I}(\Theta^{-1}_{Y}))]_{E}\nonumber\\&
    -{\Gamma^{\ast}_{EE}}^{-1}[\xbar{\Psi_{1}L^{\ast}\mathfrak{R}(\Theta^{-1}_{Y})}-\xbar{\Psi_{2}L^{\ast}\mathfrak{I}(\Theta^{-1}_{Y})}-\xbar{(\Delta^{\prime}+L^{\ast})^{-1}}]_{E}\nonumber\\&+{\Gamma^{\ast}_{EE}}^{-1}\Big[\Gamma^{\ast}_{EE}\xbar{\Delta^{\prime}}_{E}+\lambda_{n}^{\prime}\xbar{Z}_{E}\Big],\label{appeq: contraction mapping}
\end{align}
where (a) follows from substituting $P_{1} = W_1+\mathfrak{R}(\Theta^{-1}_{Y})$ and $P_{2} = W_2+\mathfrak{I}(\Theta^{-1}_{Y})$. As shown in the proof of Lemma \ref{applma: Sufficiency of strict dual feasibility}, we have $\Psi_{1}L^{\ast}\mathfrak{R}(\Theta^{-1}_{Y})-\Psi_{2}L^{\ast}\mathfrak{I}(\Theta^{-1}_{Y}) = {L^{\ast}}^{-1}$, with this equation\eqref{appeq: contraction mapping} becomes,
\begin{align}
         &F(\xbar{\Delta^{\prime}}_{E}) =-{\Gamma^{\ast}_{EE}}^{-1}[\mvec(\Psi_{1}(\Delta^{\prime}+L^{\ast})W_1-\Psi_{2}(\Delta^{\prime}+L^{\ast})W_2\nonumber\\&+\Psi_{1}\Delta^{\prime}\mathfrak{R}(\Theta^{-1}_{Y})-\Psi_{2}\Delta^{\prime}\mathfrak{I}(\Theta^{-1}_{Y}))+\lambda_{n}^{\prime}\xbar{Z}_{E}]_{E}\nonumber\\&
        -{\Gamma^{\ast}_{EE}}^{-1}\Big[[\mvec({L^{\ast}}^{-1}-(\Delta^{\prime}+L^{\ast})^{-1})]_{E}+\Gamma^{\ast}_{EE}\xbar{\Delta^{\prime}}_{E}\Big]\nonumber.
\end{align}
By definition, the vectorized expression, 
\begin{align*}
    \mvec({L^{\ast}}^{-1}-(\Delta^{\prime}+L^{\ast})^{-1})]_{E} = -\xbar{R(\Delta^{\prime})}_{E}.
\end{align*}
Thus $F(\xbar{\Delta^{\prime}}_{E})$ becomes,
\begin{align}
        &F(\xbar{\Delta^{\prime}}_{E}) = \Big[\overbrace{-(\Gamma^{\ast}_{EE})^{-1}\mvec(\Psi_{1}L^{\ast}W_1-\Psi_{2}L^{\ast}W_2+\lambda_{n}^{\prime}Z)}^{T1}\Big]_{E} \nonumber\\&
       \!+\!\overbrace{\Big[{\Gamma^{\ast}_{EE}}^{-1}\xbar{R}(\Delta^{\prime})}^{T2}\Big]_{E}
   \hspace{-2px}\!-\! \Big[\overbrace{(\Gamma^{\ast}_{EE})^{-1}\mvec(\Psi_{1}\Delta^{\prime}W_{1}\!-\!\Psi_{2}\Delta^{\prime}W_{2})}^{T3}\Big]_{E}\nonumber\\&-\overbrace{\Big[(\Gamma^{\ast}_{EE})^{-1}\mvec(\Psi_{1}\Delta^{\prime}\Theta^{-1}_{Y}-\Psi_{2}\Delta^{\prime}\Theta_{Y}^{-1})}^{T4}\Big]_{E}. 
\end{align}
We now show that $\Vert F(\xbar{\Delta^{\prime}}_{E})\Vert_{\infty}\leq r$ by bounding the $\ell_\infty$-norms of the terms $(T_1)$-$(T_4)$ defined above. Recall that $\nu_{A} = \vertiii{A}_{\infty}\triangleq \max_{j=1,\ldots,p}\sum_{j=1}^{p}\vert A_{ij}\vert$ and it is sub-multiplicative; that is $\vertiii{AB}_\infty\leq \vertiii{A}_\infty\vertiii{B}_\infty$. Notice that this is not true for the max norm ($\ell_\infty$). Recall also that $\Gamma(AB) = (I\otimes AB)$.

\smallskip 
(i) \emph{Upper bound on $\|T_1\|_{\infty}$}: Consider the following chain of inequalities.
\begin{align*}
  &\|T_1\|_{\infty} = \|{\Gamma^{\ast}_{EE}}^{-1}\left[\mvec(\Psi_{1}L^{\ast}W_1-\Psi_{2}L^{\ast}W_2+\lambda_{n}^{\prime}Z)\right]_{E}\|_{\infty} \\
  &\overset{(a)}{\leq} \vertiii{(\Gamma^{\ast}_{EE})^{-1}}_{\infty}\Big[\|\Gamma(\Psi_{1}L^{\ast})\xbar{W}_{1}+\Gamma(\Psi_{2}L^{\ast})\xbar{W}_{2}+\lambda_{n}^{\prime}\xbar{Z}\|_{\infty}  \Big]\\
  &\overset{(b)}{\leq} \nu_{{\Gamma^{\ast}}^{-1}}\Big[\nu_{\Psi_{1}}\nu_{L^{\ast}}\|W_{1}\|_{\infty}+ \nu_{\Psi_{2}}\nu_{L^{\ast}}\|W_{2}\|_{\infty} +\lambda_{n}^{\prime}\Big]\\
  &\overset{(c)}{\leq}2\nu_{{\Gamma^{\ast}}^{-1}}\Big[\nu_{D^{2}}\nu_{L^{\ast}}\|W\|_{\infty}+0.5\lambda_{n}^{\prime}\Big]\overset{(d)}{\leq}r/4,
\end{align*}
where (a) follows because $\|Av\|_{\infty}\leq \vertiii{A}_{\infty}\|v\|_{\infty}$; (b) follows from applying triangle inequality on the element-wise $\ellinf$-norm and $\|Z\|_{\infty}\leq 1$, where $Z$ is the sub-gradient is in Lemma \ref{applma: convexity and uniqueness}; (c) for any complex matrix $\|A\|_{\infty}\geq \max\{\|\mathfrak{R}(A)\|_{\infty},\|\mathfrak{I}(A)\|_{\infty}\}$; and (d) from the definition of radius $r$ in Lemma \ref{applma: Control of Delta}.

(ii) \emph{Upper bound on $\|T_2\|_{\infty}$}: Consider the inequality:
\begin{align*}
    \|T_2\|_{\infty}&=\|(\Gamma^{\ast}_{EE})^{-1}\xbar{R}(\Delta^{\prime})_{E}\|_{\infty}\\
    &\overset{(a)}{\leq} \frac{3}{2}\nu_{{\Gamma^{\ast}}^{-1}}d\nu^{3}_{{L^{\ast}}^{-1}}\|\Delta^{\prime}\|^{2}_{\infty}\\
    &\overset{(b)}{\leq}\frac{3}{2}\nu_{{\Gamma^{\ast}}^{-1}}d\nu^{3}_{{L^{\ast}}^{-1}}r^{2}=\left(\frac{3}{2}\nu_{{\Gamma^{\ast}}^{-1}}d\nu^{3}_{{L^{\ast}}^{-1}}r\right)r\\
    &\overset{(c)}{\leq}\frac{r}{4},
\end{align*}
where inequality (a) follows because Lemma \ref{applma: control remainder} guarantees that $\Vert R(\Delta^{\prime})\Vert_{\infty}\leq (3/2)d\nu_{{L^{\ast}}^{-1}}^{3}\Vert \Delta^{\prime}\Vert^{2}_{\infty}$ whenever $\Vert \Delta^{\prime}\Vert_{\infty}\leq 1/(3d\nu_{{L^{\ast}}^{-1}})$. The latter inequality is a consequence of the hypothesis in Lemma \ref{applma: Control of Delta}; (b) follows by construction $\Delta^{\prime}\in \mathbb{B}_{r}$, and hence, $\Vert\ \Delta^{\prime}\Vert_{\infty}\leq r$; (c) follows by invoking the hypothesis in Lemma \ref{applma: Control of Delta}, where $r$ satisfies $r\leq 1/(6d\nu_{{\Gamma^{\ast}}^{-1}}\nu_{{L^{\ast}}^{-1}}^{3})$. 

\smallskip 
(iii) \emph{Upper bound on $\|T_3\|_{\infty}$}: Consider the inequality:
\begin{align*}
   \|T_{3}\|_{\infty}&= \|-(\Gamma^{\ast}_{EE})^{-1}\left[\mvec(\Psi_{1}\Delta^{\prime}W_{1}-\Psi_{2}\Delta^{\prime}W_{2})\right]_{E}\|_{\infty}\\
   &\leq \nu_{{\Gamma^{\ast}}^{-1}}\|\left[\Gamma(\Psi_{1}\Delta^{\prime})W_1-\Gamma(\Psi_{2}\Delta^{\prime})W_2\right]_{E}\|_{\infty}\\
   &\leq \nu_{{\Gamma^{\ast}}^{-1}}\left[\nu_{\Psi_{1}}\vertiii{\Delta^{\prime}}_{\infty}\|W_1\|_{\infty} +\nu_{\Psi_{2}}\vertiii{\Delta^{\prime}}_{\infty}\|W_2\|_{\infty}\right]\\
   &\overset{(a)}{\leq}2\nu_{{\Gamma^{\ast}}^{-1}}\nu_{D^{2}}\vertiii{\Delta^{\prime}}_{\infty}\|W\|_{\infty}\\
   &\overset{(b)}{\leq}2\nu_{{\Gamma^{\ast}}^{-1}}\nu_{D^{2}}d\|\Delta^{\prime}\|_{\infty}\|W\|_{\infty}\\
   &\overset{(c)}{\leq}2\nu_{{\Gamma^{\ast}}^{-1}}\nu_{D^{2}}\|W\|_{\infty}d\left(\frac{1}{3d\nu_{{L^{\ast}}^{-1}}}\right)\\
   &\overset{(d)}{\leq}\frac{r}{4\nu_{L^{\ast}}}\left(\frac{1}{3\nu_{{L^{\ast}}^{-1}}}\right)=\frac{r}{12}\left(\frac{1}{\nu_{{L^{\ast}}^{-1}}\nu_{L^{\ast}}}\right)\overset{(e)}{\leq}\frac{r}{12}\leq\frac{r}{4},
\end{align*}
where (a) follows since for any complex matrix $\|A\|_{\infty}\geq \max\{\|\mathfrak{R}(A)\|_{\infty},\|\mathfrak{I}(A)\|_{\infty}\}$; (b) follows because by construction $\Delta^{\prime}$ has at-most $d$ non-zeros in every row and that $\vertiii{\Delta^{\prime}}_{\infty}\leq d\Vert\Delta^{\prime}\Vert_{\infty}$; (c) follows because $\Delta^{\prime}$ is a zero-padded matrix of $\Delta$. Hence $\Vert\Delta \Vert=\Vert \Delta^{\prime}\Vert_{\infty}\leq r$, which can be upper bounded by $1/(3d\nu_{{L^{\ast}}^{-1}})$ in light of the hypothesis in Lemma \ref{applma: Control of Delta};  (d) follows from the choice of $r = 8\nu_{{\Gamma^{\ast}}^{-1}}(\nu_{D^{2}}\nu_{L^{\ast}}\Vert W\Vert_{\infty} + \lambda^{\prime}_{n})$ in Lemma \ref{applma: Control of Delta}, which is lower bounded by  $8\nu_{{\Gamma^{\ast}}^{-1}}\nu_{D^{2}}\nu_{L^{\ast}}\Vert W\Vert_{\infty}$, for all $\lambda^{\prime}_{n}\geq 0$. Thus, 
$\Vert W\Vert_{\infty}\leq r/(8\nu_{{\Gamma^{\ast}}^{-1}}\nu_{D^{2}}\nu_{L^{\ast}})$; and finally,
(e) follows because $\nu_{L^{\ast}}\nu_{{L^{\ast}}^{-1}}\geq 1$.

(iv) \emph{Upper bound on $\|T_4\|_{\infty}$}: Consider the inequality: 
\begin{align*}
   &\|T_4\|_{\infty}=\|(\Gamma^{\ast}_{EE})^{-1}\left[ \mvec(\Psi_{1}\Delta^{\prime}\Theta^{-1}_{Y}-\Psi_{2}\Delta^{\prime}\Theta_{Y}^{-1})\right]_{E}\|_{\infty}\\
   &\leq \nu_{{\Gamma^{\ast}}^{-1}}\Big[\|\Gamma(\Psi_{1}\Delta^{\prime})\mathfrak{R}(\Theta_{Y}^{-1})\|_{\infty}+\|\Gamma(\Psi_{2}\Delta^{\prime})\mathfrak{I}(\Theta_{Y}^{-1})\|_{\infty} \Big]\\
   &\leq \nu_{{\Gamma^{\ast}}^{-1}}\Big(d\nu_{D^2}\|\Delta^{\prime}\|_{\infty}\Big(\|\mathfrak{R}(\Theta^{-1}_{Y})\|_{\infty}+\mathfrak{I}(\Theta^{-1}_{Y})\|_{\infty}\Big)\\
   &\leq 2\nu_{{\Gamma^{\ast}}^{-1}}\nu_{D^{2}}dr\|\Theta^{-1}_{Y}\|_{\infty}\\
   &\overset{(a)}{\leq} \frac{r}{2C_{\alpha}}\overset{(b)}{\leq}\frac{r}{4},
\end{align*}
where (a) follows from the condition number of the Hessian $\textbf{[A3]}$; (b) follows since $C_{\alpha} = 1+12/\alpha>2$, for $\alpha\in (0,1]$.
Putting together the pieces, we conclude that $F(\Delta^{\prime}_{E})\Vert_{\infty}\leq \sum_{i=1}^{4}\Vert T_{i}\Vert_{\infty} \leq r$, therefore $F$ is a contraction as claimed.
% \begin{align}
%     \hspace{-2mm}\Vert F(\Delta^{\prime}_{E})\Vert_{\infty}&\leq \sum_{i=1}^{4}\Vert T_{i}\Vert_{\infty} \leq r. 
% \end{align}
% therefore $F$ is a contraction as claimed.
\end{proof}

\setcounter{appendixtheorem}{0}

\begin{appendixtheorem}\label{appthm: Gaussian time series}
    Let the injections $X_{t}$ be a WSS Gaussian time series. Consider a single Fourier frequency $\omega_{j}\in[-\pi,\pi]$. Suppose that assumptions in $\textbf{[A1-A3]}$ hold. Define $\alpha>0$ and $C_{\alpha} =  1+24/\alpha$. Let 
%the regularization parameter be chosen as 
$\lambda_{n} = 96\nu_{D^2}\nu_{L^{\ast}}\delta_{\Theta_{Y}^{-1}}(m,n,p)/\alpha$ and the bandwidth parameter $m\geq \vertiii{\Theta_{Y}^{-1}}_{\infty}^{2}\zeta^{2}d^{2}\log p$, where 
$\zeta = \max\{\nu_{{\Gamma^{\ast}}^{-1}}\nu_{{L^{\ast}}^{-1}}\nu_{L^{\ast}}\nu_{D^{2}}C^{2}_{\alpha},\nu^{2}_{{\Gamma^{\ast}}^{-1}}\nu^{3}_{{L^{\ast}}^{-1}}\nu_{L^{\ast}}\nu_{D^{2}}C^{2}_{\alpha}\}$. 
% where $\alpha >0$. Let the bandwidth  $m\geq \vertiii{\Theta_{Y}^{-1}}_{\infty}^{2}\zeta^{2}d^{2}\log p$ where $\zeta = \max\{\nu_{{\Gamma^{\ast}}^{-1}}\nu_{{L^{\ast}}^{-1}}\nu_{L^{\ast}}\nu_{D^{2}}C^{2}_{\alpha},\nu^{2}_{{\Gamma^{\ast}}^{-1}}\nu^{3}_{{L^{\ast}}^{-1}}\nu_{L^{\ast}}\nu_{D^{2}}C^{2}_{\alpha}\}$ and $C_{\alpha} =  1+24/\alpha$, $\alpha>0$.
If the sample size $n\geq 144\Omega_{n}(\Theta_{Y}^{-1})\zeta md$. Then with probability greater than $1-1/p^{\tau-2}$, for some $\tau>2$, we have 

    \begin{enumerate}[label=(\alph*)]
        \item $\widehat{L}$ exactly recovers the sparsity structure ie. $\widehat{L}_{E^{c}}=0$.
        \item The estimate $\widehat{L}$ which is the solution of \eqref{eq: l_1 regularized Whittle likelihood} satisfies 
        \begin{align}
            \|\widehat{L}-L^{\ast}\|_{\infty} \leq 8\nu^{\prime}\delta_{\Theta_{Y}^{-1}}(m,n,p).
        \end{align}
        \item $\widehat{L}$ satisfies sign consistency if:
        \begin{align}
            |L^{\ast}_{\min}(E)|\geq 8\nu^{\prime}\delta_{\Theta_{Y}^{-1}}(m,n,p),
        \end{align}
    \end{enumerate}
where, $\nu^{\prime} = \nu_{{\Gamma^{\ast}}^{-1}}\nu_{D^{2}}\nu_{L^{\ast}}C_{\alpha}$ and
\begin{align*}
    \delta_{\Theta_{Y}^{-1}}(m,n,p) \!=\!\sqrt{\frac{\log p}{m}} +\frac{m+\frac{1}{2\pi}}{n}\Omega_{n}(\Theta_{Y}^{-1})+\frac{1}{2\pi}L_{n}(\Theta_{Y}^{-1}). 
\end{align*} 
\end{appendixtheorem} 
\begin{proof}\label{app: Gaussian time series}
    Our goal is to derive the sufficient conditions on the tuple $(n,m,p,d)$ to establish support recovery, error norm bound and sign consistency of the estimator $\widehat{L}$. We begin by showing that with an optimal selection of the regularization parameter, the primal solution $\widetilde{L}$ of \eqref{appeq: rest whittle} is equal to $\widehat{L}$ of the original $\ell_{1}$-regularized problem \eqref{appeq: l_1 regularized whittle likelihood} by showing that the primal dual witness construction succeeds with high probability. Towards this, we proceed by verifying the sufficient conditions of the strict dual feasibility. From Lemma \ref{applma: Sufficiency of strict dual feasibility}, the sufficient conditions for strict dual feasibility imply that 
    \begin{align}
    T_1 &= 2\nu_{D^{2}}(d\|\Delta\|_{\infty}+\nu_{L^{\ast}})\|W\|_{\infty}\leq \frac{\alpha\lambda_{n}}{24}\label{appeq: suff condn for strict dual T1}\\
    T_2& = \|R(\Delta)\|_{\infty}\leq \frac{\alpha\lambda_{n}}{24}\label{appeq: suff condn for strict dual T2}\\
T_3&=2\nu_{D^{2}}d\|\Delta\|_{\infty}\|\Theta^{-1}_{Y}\leq \frac{\alpha\lambda_{n}}{24}.\label{appeq: suff condn for strict dual T3}
\end{align}
Let $\mathcal{A}$ denote the event that $\|W\|_{\infty}\leq \delta_{\Theta^{-1}_{Y}}(n,m,p)$, where $W$ is the measure of noise in the averaged periodogram. From this point forward, we will adopt a slight abuse of notation by using $\delta$ instead of $\delta_{\Theta^{-1}_{Y}}$. We condition on the event $\mathcal{A}$ in the analysis that follows. We proceed by first choosing the regularization parameter as $\lambda_{n} = 96\nu_{D^{2}}\nu_{L^{\ast}}\delta/{\alpha}$ , the radius $r$ defined in Lemma \ref{applma: Control of Delta} satisfies the bound
\begin{align}
r&=8\nu_{{\Gamma^\ast}^{-1}}\left[\nu_{D^{2}}\nu_{L^{\ast}}\|W\|_{\infty}+\frac{24\nu_{D^{2}}\nu_{L^\ast}\delta}{\alpha}\right]\nonumber \\
    &\overset{(a)}{\leq}  8\nu_{{\Gamma^\ast}^{-1}}\nu_{D^{2}}\nu_{L^{\ast}}C_{\alpha}\delta,\label{appeq: radius}
\end{align}
where (a) follows from conditioning on event $\mathcal{A}$ and $C_\alpha=1+24/\alpha$. We proceed to select $\delta$ according to the following criterion, which is permissible since $\delta$ can be made arbitrarily small with a sufficient number of samples. The specific conditions on the tuple $(n, m, p)$ to attain such a $\delta$ will be derived subsequently. Choose $\delta$ such that
\begin{align}\label{appeq: choice of delta}
\hspace{-2mm}8\nu_{{\Gamma^\ast}^{-1}}\nu_{D^{2}}\nu_{L^{\ast}}C^{2}_{\alpha}\delta\leq\min\left\{\frac{1}{3\nu_{{L^{\ast}}^{-1}}d},\frac{1}{6\nu_{{\Gamma^\ast}^{-1}}\nu^{3}_{{L^{\ast}}^{-1}}d}\right\}.
\end{align}
Substituting this choice of $\delta$ in \eqref{appeq: radius} we get
\begin{align}\label{appeq: radius constrained}
    r\leq \min\left\{\frac{1}{3\nu_{{L^{\ast}}^{-1}}d},\frac{1}{6\nu_{{\Gamma^\ast}^{-1}}\nu^{3}_{{L^{\ast}}^{-1}}d}\right\}.
\end{align}
We now have the necessary ingredients to verify the sufficient condition for the strict dual feasibility condition \eqref{appeq: suff condn for strict dual T1}-\eqref{appeq: suff condn for strict dual T3}.

\smallskip 
(i) \emph{Upper bound on $T_1$}:
\begin{align}
    T_{1}&=2\nu_{D^{2}}(d\|\Delta\|_{\infty}+\nu_{L^{\ast}})\|W\|_{\infty}\\
    &\overset{(a)}{\leq} 2\nu_{D^{2}}(d\|\Delta\|_{\infty}+\nu_{L^{\ast}})\delta\\
    &\overset{(b)}{\leq} 2\nu_{D^{2}}\nu_{L^{\ast}}\left(\frac{d}{3\nu_{L^{\ast}}\nu_{{L^{\ast}}^{-1}}d}+1\right)\delta\\
    &\overset{(c)}{\leq}2\nu_{D^{2}}\nu_{L^{\ast}}\left(\frac{1}{3}+1\right)\frac{\alpha\lambda_{n}}{96\nu_{D^{2}}\nu_{L^{\ast}}}\\
    &\leq \frac{\alpha\lambda_{n}}{24},
\end{align}
where (a) follows from conditioning on the event $\mathcal{A}$ i.e., $\|W\|_{\infty}\leq \delta$; (b) since the radius $r$ in \eqref{appeq: radius constrained} satisfies condition in Lemma \ref{applma: Control of Delta} therefore $\|\Delta\|_{\infty}\leq r$; (c) follows since $\nu_{L^{\ast}}\nu_{{L^{\ast}}}^{-1}\geq 1$ and substituting for $\delta$ in terms of the regularization $\lambda_{n}$ and $\alpha$.

\smallskip 
(ii)  \emph{Upper bound on $T_2$}:
\begin{align}
    T_2 &= \|R(\Delta)\|_{\infty}\\
    &\overset{(a)}{\leq} \frac{3}{2}d\|\Delta\|^{2}_{\infty}\nu^{3}_{{L^\ast}^{-1}}\\
    &\overset{(b)}{\leq} \frac{3}{2}d r^{2}\nu^{3}_{{L^\ast}^{-1}}\\
    &\overset{(c)}{\leq} \frac{3}{2}d(64\nu^{2}_{{\Gamma^\ast}^{-1}}\nu^{2}_{D^{2}}\nu^{2}_{L^{\ast}}C_{\alpha}^{2}\nu^{3}_{{L^{\ast}}^{-1}}\delta)\delta\\
    &\overset{(d)}{\leq} (2\nu_{D^{2}}\nu_{L^{\ast}})\frac{\alpha\lambda_{n}}{96\nu_{D^{2}}\nu_{L^{\ast}}} \leq \frac{\alpha\lambda_{n}}{24},
\end{align}
where (a) follows since the radius $r$ in \eqref{appeq: radius constrained} satisfies the condition in Lemma \ref{applma: control remainder}; (b) follows since the radius $r$ in \eqref{appeq: radius constrained} satisfies the condition in Lemma \ref{applma: Control of Delta} and therefore $\|\Delta\|_{\infty}\leq r$; (c) follows from substituting for $r$ in \eqref{appeq: radius}; (d) follows from choice of $\delta$ in \eqref{appeq: choice of delta}.

(ii)  \emph{Upper bound on $T_3$}:
\begin{align}
    T_3 &= 2\nu_{D^{2}}d\|\Delta\|_{\infty}\|\Theta^{-1}_{Y}\|_{\infty}\\
    &\overset{(a)}{\leq} 2\nu_{D^{2}}d r\|\Theta^{-1}_{Y}\|_{\infty}\\
    &\overset{(b)}{\leq} 2\nu_{D^{2}}d r\left(\frac{1}{4d\nu_{{\Gamma^\ast}^{-1}}\nu_{D^2}C_{\alpha}}\right)\\
    &\overset{(c)}{\leq} 16d\nu_{{\Gamma^\ast}^{-1}}\nu^{2}_{D^{2}}\nu_{L^{\ast}}C_{\alpha}\delta\left(\frac{1}{4d\nu_{{\Gamma^\ast}^{-1}}\nu_{D^2}C_{\alpha}}\right)\\
    &=\frac{\alpha\lambda_{n}}{24},
\end{align}
where (a) follows since $\|\Delta\|_{\infty}\leq r$; (b) follows from the bounded Hessian condition \textbf{[A3]}; (c) substituting for $r$ in \eqref{appeq: radius}. We therefore have verified the sufficient conditions for strict dual feasibility. It remains to derive the sufficient conditions on the tuple $(n,m,p,d)$ such that the event $\mathcal{A}$ i.e., $\|W\|_{\infty}\leq \delta$ holds with high probability, where $\delta$ is chosen as in \eqref{appeq: choice of delta}. From Lemma \ref{lma: averaged periodogram error for Gaussian time series}, the threshold is 
\begin{align}\label{appeq: expression for delta}
    \delta_{\Theta^{-1}_{Y}}(n,m,p)& = \vertiii{\Theta^{-1}_{Y}}_{\infty}\sqrt{\frac{\tau\log p}{m}}+\frac{m+1/{2\pi}}{n}\Omega_{n}(\Theta^{-1}_{Y})\nonumber\\&+\frac{1}{2\pi}L_{n}(\Theta^{-1}_{Y}),
\end{align}
where $\Omega_{n}(\Theta^{-1}_{Y})$ and $L_{n}(\Theta^{-1}_{Y})$ are defined as in \eqref{eq: model dependent quantities Omega_n} and \eqref{eq: model dependent quantities L_n}. We derive sufficient condition on $(n,m,p,d)$ such that $\delta_{\Theta^{-1}_{Y}}(n,m,p)$ satisfies the criterion in \eqref{appeq: choice of delta}. Towards this, we set the first term in the RHS of \eqref{appeq: expression for delta} to
\begin{align}\label{appeq: bound for first term of delta}
    \vertiii{\Theta^{-1}_{Y}}_{\infty}\sqrt{\frac{\tau\log p}{m}}&\leq \min\Bigg\{ \frac{1}{72\underbrace{\nu_{{\Gamma^\ast}^{-1}}\nu_{D^{2}}\nu_{L^{\ast}}\nu_{{L^{\ast}}^{-1}}C_{\alpha}^{2}}_{\widetilde{\zeta}}d}\nonumber,\\&\frac{1}{144\underbrace{\nu^{2}_{{\Gamma^\ast}^{-1}}\nu_{D^{2}}\nu_{L^{\ast}}\nu^{3}_{{L^{\ast}}^{-1}}C_{\alpha}^{2}}_{\zeta^{\prime}}d} \Bigg\}.
\end{align}
Solving for $m$ we get
\begin{align}\label{appeq: bound for m}
    m\geq (144)^{2}\zeta^{2}\vertiii{\Theta^{-1}_{Y}}^{2}_{\infty}d^{2}\log p,
\end{align}
where $\zeta = \max\{\widetilde{\zeta},\zeta^\prime\}$, similarly we bound the second term in the RHS of \eqref{appeq: expression for delta} as
\begin{align}\label{appeq: bound for second term of delta}
    \frac{m+1/{2\pi}}{n}\Omega_{n}(\Theta^{-1}_{Y})&\leq\min\Bigg\{ \frac{1}{72\underbrace{\nu_{{\Gamma^\ast}^{-1}}\nu_{D^{2}}\nu_{L^{\ast}}\nu_{{L^{\ast}}^{-1}}C_{\alpha}^{2}}_{\widetilde{\zeta}}d}\nonumber,\\&\frac{1}{144\underbrace{\nu^{2}_{{\Gamma^\ast}^{-1}}\nu_{D^{2}}\nu_{L^{\ast}}\nu^{3}_{{L^{\ast}}^{-1}}C_{\alpha}^{2}}_{\zeta^{\prime}}d} \Bigg\}.
\end{align}
For large $n$ we have
\begin{align}
    \frac{m+1/{2\pi}}{n}\Omega_{n}(\Theta^{-1}_{Y}) \approx \frac{m}{n}\Omega_{n}(\Theta^{-1}_{Y}).
\end{align}
Solving for $n$ we get
\begin{align}\label{appeq: bound for n in terms of m}
    n\geq 144 \zeta \Omega_{n}(\Theta^{-1}_{Y})m d.
\end{align}

For large $n$, Assumption \textbf{[A2]} guarantees that
\begin{align}\label{appeq: bound for third term in delta}
    \frac{1}{2\pi}L_{n}(\Theta^{-1}_{Y})&\leq \min\Bigg\{ \frac{1}{72\underbrace{\nu_{{\Gamma^\ast}^{-1}}\nu_{D^{2}}\nu_{L^{\ast}}\nu_{{L^{\ast}}^{-1}}C_{\alpha}^{2}}_{\widetilde{\zeta}}d}\nonumber,\\&\frac{1}{144\underbrace{\nu^{2}_{{\Gamma^\ast}^{-1}}\nu_{D^{2}}\nu_{L^{\ast}}\nu^{3}_{{L^{\ast}}^{-1}}C_{\alpha}^{2}}_{\zeta^{\prime}}d} \Bigg\}.
\end{align}
Combining equations \eqref{appeq: bound for first term of delta},\eqref{appeq: bound for second term of delta}, and \eqref{appeq: bound for third term in delta} we can guarantee that $\delta_{\Theta^{-1}_{Y}}(n,m,p)$ satisfies the criterion in $\eqref{appeq: choice of delta}$.
\end{proof}

\vspace{-2.0mm}
\setcounter{appendixcorollary}{0}
\label{app: frobenius norm bounds for gaussian time series}
\begin{appendixcorollary}\label{appcorr: frobenius norm bounds for gaussian time series}
    Let $s=|\mathcal{E}(L^{\ast})|$ be the cardinality of the edge set $\mathcal{E}(L^{\ast})$. Under the hypothesis as in Theorem \ref{appthm: Gaussian time series}, with probability greater than $1-\frac{1}{p^{\tau-2}}$, the estimator $\widehat{L}$ defined in \eqref{appeq: l_1 regularized whittle likelihood} satisfies
    \begin{align*}
    \Vert \widehat{L}\!-\!L^{\ast}\Vert_{F} & \leq 8\nu^{\prime} (\sqrt{s+p})\delta_{\Theta_{Y}^{-1}}(m,n,p) \quad \text{ and }\\
    \Vert{\widehat{L}\!-\!L^{\ast}}\Vert_{2} &\leq 8\nu^{\prime}\min\{d,\sqrt{s+p}\}\delta_{\Theta_{Y}^{-1}}(m,n,p).
\end{align*}
\end{appendixcorollary}
\begin{proof}
First note the following inequality: 
\begin{align}
    \Vert \widehat{L}-L^{\ast}\Vert_{F}^2&= \sum_{i,j}\left(\widehat{L}_{ij}-L^{\ast}_{ij}\right)^{2}\\&= \sum_{i}\left(\widehat{L}_{ii}-L^{\ast}_{ii}\right)^{2}+\sum_{i\neq j}\left(\widehat{L}_{ij}-L^{\ast}_{ij}\right)^{2}\\
    &\leq p\Vert \widehat{L}-L^{\ast}\Vert^{2}_{\infty} + s\Vert \widehat{L}-L^{\ast}\Vert^{2}_{\infty}\\
    &=(s+p)\Vert \widehat{L}-L^{\ast}\Vert_{\infty}^2.
\end{align}
where the inequality follows because there are at most $p$ non-zero diagonal terms and $s$ non-zero off-diagonal terms in $\widehat{L}-L^\ast$. The latter fact is a consequence of part (a) of Theorem \ref{appthm: Gaussian time series}, which ensures that $\widehat{L}_{E^{c}} = L^{\ast}_{E^{c}}$ with high probability when $n=\Omega(d^{3}\log p)$. We obtain the Frobenius norm bound in the above corollary by upper bounding
$\Vert \widehat{L}-L^{\ast}\Vert_{\infty}$ using part (b) of Theorem \ref{appthm: Gaussian time series}. 

We now establish spectral norm consistency. From matrix norm equivalence conditions \cite{horn2012matrix}, we have,
\begin{align}
    \Vert \widehat{L}-L^{\ast}\Vert_{2}\leq \vertiii{\widehat{L}-L^{\ast}}_{\infty}
    \leq d\Vert \widehat{L}-L^{\ast}\Vert_{\infty}\label{eq: matrix norm equivalence1},
\end{align}
and that
\begin{align}\label{eq: matrix norm equivalence2}
    \Vert \widehat{L}-L^{\ast}\Vert_{2}\leq \Vert \widehat{L}-L^{\ast}\Vert_{F}\leq \sqrt{s+p}\Vert \widehat{L}-L^{\ast}\Vert_{\infty}.
\end{align}
These two bounds can be unified to give
\begin{align}
    \Vert \widehat{L}-L^{\ast}\Vert_{2}\leq \min\{\sqrt{s+p},d\}\Vert \widehat{L}-L^{\ast}\Vert_{\infty}. 
\end{align}
This concludes the proof. 
\end{proof}

\subsection{Linear processes}
\label{subsection: Linear Processes}
 
In this section, we consider a class of WSS processes that are not necessarily Gaussian. Examples include Vector Auto Regressive (VAR($p$)) and Vector Auto Regressive Moving Average (VARMA ($p,q$)) models. Such models, and many others, belong to the family of a linear WSS process with absolute summable coefficients: 
\begin{align}\label{appeq: linear process}
    X_{t} = \sum_{l=0}^{\infty}A_{l}\epsilon_{t-l},
\end{align}
where $A_{l}\in\mathbb{R}^{p\times p}$ is known and $\epsilon_{t}\in\mathbb{R}^{p}$ is a zero mean i.i.d. process with tails possibly heavier than Gaussian tails. The absolute summability $\sum_{l=0}^{\infty}|A_{l}(i,j)|<\infty$ ensures  stationarity for all $i,j\in \{1,\ldots,p\}$ \cite{rosenblatt2012stationary_app}.  We assume that $\epsilon_{kl}$, the $k$-th component of $\epsilon_{l} \in \mathbb{R}^p$, is given by one the distributions below:

\smallskip 
\noindent\textbf{[B1] Sub-Gaussian:} There exists $\sigma>0$ such that for $\eta>0$, we have $\mathbb{P}[|\epsilon_{kl}|>\eta]\leq 2\exp(-\frac{\eta^{2}}{2\sigma^{2}})$. 

\smallskip 

\noindent\textbf{[B2] Generalized sub-exponential with parameter $\rho>0$:} For constants $a$ and $b$, and $\eta>0$: $ \mathbb{P}[|\epsilon_{kl}|>\eta^{\rho}]\leq a\exp(-b\eta)$.      

\smallskip 
\noindent\textbf{[B3] Distributions with finite $4^{\text{th}}$ moment:} There exists a constant $M>0$ such that $\mathbb{E}[\epsilon_{kl}^{4}]\leq M <\infty$.

\medskip 

We need additional notation. Let $n_{k}=\Omega(d^{3}\mathcal{T}_{k})$ represent the family of sample sizes indexed by $k=\{1,2,3\}$, where $\mathcal{T}_{1}=\log p$ correspond to the distribution in [\textbf{B1}], $\mathcal{T}_{2}=(\log p)^{4+4\rho}$ in [\textbf{B2}], and $\mathcal{T}_{3}=p^{2}$ in [\textbf{B3}].  
\begin{appendixtheorem}\label{appthm: Linear Process support recovery}
    Let $X_{t}$ be given by \eqref{appeq: linear process} and $Y_{t} = {L^{\ast}}^{-1}X_{t}$. Fix $\omega_{j}\in[-\pi,\pi]$. Let $n_{k}=\Omega(d^{3}\mathcal{T}_{k})$, where $k=\{1,2,3\}$. Then for some $\tau>2$, with probability greater than $1-1/p^{\tau-2}$:
    \begin{enumerate}[label=(\alph*), itemsep=0pt]
    \item $\widehat{L}$ exactly recovers the sparsity structure ie. $\widehat{L}_{E^{c}}=0$
    \item The $\ellinf$ bound of the error satisfies:
    \begin{align}
        \|\widehat{L}-L^{\ast}\|_{\infty} = \mathcal{O}(\delta_{\Theta_{Y}^{-1}}^{(k)}(n,m,p)).
    \end{align}
    \item $\widehat{L}$ satisfies sign consistency if:
    \begin{align}
        |L^{\ast}_{\min}(E)| = \Omega(\delta_{\Theta_{Y}^{-1}}^{(k)}(n,m,p)),
    \end{align}
    \end{enumerate}
    where $\delta_{\Theta_{Y}^{-1}}^{(k)}(n,m,p)$ for $k=\{1,2,3\}$ is given by,
        \begin{align*}
            \delta_{\Theta_{Y}^{-1}}^{(1)}(n,m,p) &= \vertiii{\Theta_{Y}^{-1}}_{\infty}\frac{(\tau\log p)^{1/2}}{\sqrt{m}}+\bigtriangleup(n,m,\Theta_{Y}^{-1}),\\
            \delta_{\Theta_{Y}^{-1}}^{(2)}(n,m,p) &= \vertiii{\Theta_{Y}^{-1}}_{\infty}\frac{(\tau\log p)^{2+2\rho}}{\sqrt{m}}+\bigtriangleup(n,m,\Theta_{Y}^{-1}),\\
            \delta_{\Theta_{Y}^{-1}}^{(3)}(n,m,p) &= \vertiii{\Theta^{-1}_{Y}}_{\infty}\frac{p^{1+\tau}}{\sqrt{m}}+\bigtriangleup(n,m,\Theta_{Y}^{-1}),
        \end{align*}
        and $\bigtriangleup(n,m,\Theta_{Y}^{-1})=\frac{m+\frac{1}{2\pi}}{n}\Omega_{n}(\Theta_{Y}^{-1})+\frac{1}{2\pi}L_{n}(\Theta_{Y}^{-1})$. 
\end{appendixtheorem} 
\begin{proof}
    The proof follows along the same lines of Theorem \ref{appthm: Gaussian time series}, where the sufficient conditions for $(n_{k},m_{k},p,d)$ are derived for the three families of distribution defined in \textbf{[B1-B3]} using the concentration result in Lemma \ref{lma: averaged periodogram for linear processes}.
\end{proof}

\section{Concentration results on the Averaged Periodogram}\label{sec: Concentration results on the Averaged Periodogram}
This section restates the concentration results for the averaged periodogram of Gaussian time series and linear processes, as originally presented in \cite{sun2018large_app}. Here we use $A\geqsim B$ to denote that there exists a universal constant $c$ that does not depend on the model parameters such that $A\geq cB$.
\renewcommand{\thelemma}{C.\arabic{lemma}}
\newcounter{thelemma}
\setcounter{thelemma}{0}
\begin{lemma}(Gaussian time series)\cite{sun2018large_app}:\label{lma: averaged periodogram error for Gaussian time series}
    Let $\{Z_{t}\}_{t=1}^{n}$, be $n$ observations from a stationary Gaussian time series satisfying assumption $\textbf{[A2]}$. Consider a single Fourier frequency $\omega_{j}\in [-\pi,\pi]$. If $n\geqsim \Omega_{n}(\Theta_{Z}^{-1})\vertiii{\Theta_{Z}^{-1}}^{2}_{\infty}\log p$, then for any $m$ satisfying $m\geqsim \vertiii{\Theta_{Z}^{-1}}^{2}_{\infty}\log p$ and $m\leqsim n/\Omega_{n}(\Theta_{Z}^{-1})$ let $c,c^{\prime}$ and $R$ be universal constants, then choosing a threshold
    \begin{align}
        \delta_{\Theta^{-1}_{Z}}(m,n,p)& = \vertiii{\Theta_{Z}^{-1}}_{\infty}\sqrt{\frac{R \log p}{m}} + \frac{m+1/{2\pi}}{n}\Omega_{n}(\Theta_{Z}^{-1})\nonumber\\&+\frac{1}{2\pi}L_{n}(\Theta_{Z}^{-1}).
    \end{align}
    the error of the averaged periodogram satisfies
    \begin{align*}
        \mathbb{P}[\|P_{Z}(\omega_{j})-\Theta^{-1}_{Z}(\omega_{j})\|_{\infty}\geq \delta_{\Theta^{-1}_{Z}}(m,n,p)]\leq c^{\prime}p^{-(cR-2)}.
    \end{align*}
\end{lemma}

\renewcommand{\thelemma}{C.\arabic{lemma}}
\setcounter{thelemma}{1}
\begin{lemma}(Linear process)\cite{sun2018large_app}:\label{lma: averaged periodogram for linear processes}
    Let $\{Z_{t}\}_{t=1}^{n}$, be $n$ observations from a linear process as defined in \eqref{appeq: linear process} satisfying assumption $\textbf{[A2]}$. Consider a single Fourier frequency $\omega_{j}\in [-\pi,\pi]$. If $n\geqsim \Omega_{n}(\Theta^{-1}_{Y})\mathcal{T}_{k}$, for $k=\{1,2,3\}$, where $\mathcal{T}_{1}=\vertiii{\Theta^{-1}_{Y}}^{2}_{\infty}\log p, \mathcal{T}_{2} = \vertiii{\Theta^{-1}_{Y}}^{2}_{\infty}(\log p)^{4+\rho}$ and $\mathcal{T}_{3}=p^{2}$ for the three families $\textbf{[B1,B2,B3]}$ respectively. Then for $m$ satisfying $m\geqsim \vertiii{\Theta_{Y}^{-1}}_{\infty}\mathcal{T}_{k}$ and $m\leqsim n/\Omega_{n}(\Theta^{-1}_{Y})$ and the following choice of threshold for the family \textbf{[B1,B2,B3]} respectively:
    \begin{enumerate}[label=(\textbf{B\arabic*}), itemsep=0pt]
    \item $\delta^{(1)}_{\Theta^{-1}_{Y}} = \vertiii{\Theta^{-1}_{Y}}_{\infty}\frac{(R \log p)^{1/2}}{\sqrt{m}}+\frac{m+1/{2\pi}}{n}\Omega_{n}(\Theta^{-1}_{Y})+\frac{1}{2\pi}L_{n}(\Theta_{Y}^{-1})$
    \item $\delta^{(2)}_{\Theta^{-1}_{Y}}=\vertiii{\Theta^{-1}_{Y}}_{\infty}\frac{(R\log p)^{2+2\alpha}}{\sqrt{m}}+\frac{m+1/{2\pi}}{n}\Omega_{n}(\Theta^{-1}_{Y})+\frac{1}{2\pi}L_{n}(\Theta_{Y}^{-1})$
    \item $\delta^{(3)}_{\Theta^{-1}_{Y}}=\vertiii{\Theta^{-1}_{Y}}_{\infty}+\frac{m+1/{2\pi}}{n}\Omega_{n}(\Theta^{-1}_{Y})+\frac{1}{2\pi}L_{n}(\Theta_{Y}^{-1})$,
    \end{enumerate}
the error of the averaged periodogram satisfies
\begin{align}
        \mathbb{P}[\|P_{Z}(\omega_{j})-\Theta^{-1}_{Z}(\omega_{j})\|_{\infty}\geq \delta_{\Theta^{-1}_{Z}}(m,n,p)]\leq \mathcal{T}_{k},
    \end{align}
where the tail probability $\mathcal{T}_{k}$ for $k=\{1,2,3\}$ are given by 
\begin{align*}
    \mathcal{T}_{1}&=c_{1}p^{-(c_{2}R-2)}\\
    \mathcal{T}_{2}&=c_{3}p^{-(c_{4}R-2)}\\
    \mathcal{T}_{3}&=c_{5}p^{-2R},
\end{align*}
where $c_{i}$, for $i=1,\ldots,5$, and $R$ are some universal constants.
\end{lemma}

\section{Automated Anatomical Labeling Atlas}\label{app:roi}
\onecolumn
\begin{table}[!ht]
\footnotesize
\centering
\caption{\small A comprehensive overview and abbreviations of the regions of interest (ROIs) from which the functional magnetic resonance imaging (fMRI) observations were recorded. These ROIs are defined according to the widely accepted Automated Anatomical Labeling (AAL) template, which is commonly used in neuroimaging studies to categorize and standardize brain regions for analysis. Although the AAL template offers a reliable and structured framework, it captures only a relatively coarse division of the brain, focusing on larger anatomical areas rather than more detailed substructures. This approach, while effective for many studies, limits the granularity of the analysis to broad regions rather than finer distinctions within the brain. The table also clearly differentiates between the left and right hemispheric divisions of these regions.}
\begin{tabular}{|| c | p{6cm} || c | p{6cm} ||}   
\hline
No. & Name & No. & Name \\ [0.5ex] 
\hline\hline
1 & Left precentral gyrus (PreCG.L) & 2 &  Right precentral gyrus (PreCG.R)\\
3 & Left superior frontal gyrus (SFGdor.L) & 4 & Right superior frontal gyrus (SFGdor.R)\\
5 & Left superior frontal gyrus, orbital part (ORBsup.L) & 6 & Right superior frontal gyrus, orbital part (ORBsup.R)\\  
7 & Left middle frontal gyrus (MFG.L) & 8 & Right middle frontal gyrus (MFG.R)\\
9 & Left middle frontal gyrus, orbital part (ORBmid.L) & 10 & Right middle frontal gyrus, orbital part (ORBmid.R)\\
11 & Left inferior frontal gyrus, pars opercularis (IFGoperc.L) & 12 & Right inferior frontal gyrus, pars opercularis (IFGoperc.R)\\
13 & Left inferior frontal gyrus, pars triangularis (IFGtriang.L) & 14 & Right inferior frontal gyrus, pars triangularis (IFGtriang.R)\\
15 & Left inferior frontal gyrus, pars orbitalis (ORBinf.L) & 16 & Right inferior frontal gyrus, pars orbitalis (ORBinf.R)\\
17 & Left Rolandic operculum (ROL.L) & 18 & Right Rolandic operculum (ROL.R)\\
19 & Left supplementary motor area (SMA.L) & 20 & Right supplementary motor area (SMA.R)\\
21 & Left olfactory cortex (OLF.L) & 22 & Right olfactory cortex (OLF.R)\\
23 & Left medial frontal gyrus (SFGmed.L) & 24 & Right medial frontal gyrus (SFGmed.R)\\
25 & Left medial orbitofrontal cortex (ORBsupmed.L) & 26 & Right medial orbitofrontal cortex (ORBsupmed.R)\\
27 & Left gyrus rectus (REC.L) & 28 & Right gyrus rectus (REC.R)\\
29 & Left insula (INS.L) & 30 & Right insula (INS.R)\\
31 & Left anterior cingulate gyrus (ACG.L) & 32 & Right anterior cingulate gyrus (ACG.R)\\
33 & Left midcingulate area (DCG.L) & 34 & Right midcingulate area (DCG.R)\\
35 & Left posterior cingulate gyrus (PCG.L) & 36 & Right posterior cingulate gyrus (PCG.R)\\
37 & Left hippocampus (HIP.L) & 38 & Right hippocampus (HIP.R)\\
39 & Left parahippocampal gyrus (PHG.L) & 40 & Right parahippocampal gyrus (PHG.R)\\
41 & Left amygdala (AMYG.L) & 42 & Right amygdala (AMYG.R)\\
43 & Left calcarine sulcus (CAL.L) & 44 & Right calcarine sulcus (CAL.R)\\
45 & Left cuneus (CUN.L) & 46 & Right cuneus (CUN.L)\\
47 & Left lingual gyrus (LING.L) & 48 & Right lingual gyrus (LING.R)\\
49 & Left superior occipital (SOG.L) & 50 & Right superior occipital (SOG.R)\\
51 & Left middle occipital gyrus (MOG.L) & 52 & Right middle occipital gyrus (MOG.R)\\
53 & Left inferior occipital cortex (IOG.L) & 54 & Right inferior occipital cortex (IOG.R)\\
55 & Left fusiform gyrus (FFG.L) & 56 & Right fusiform gyrus (FFG.R)\\
57 & Left postcentral gyrus (PoCG.L) & 58 & Rightpostcentral gyrus (PoCG.R)\\
59 & Left superior parietal lobule (SPG.L) & 60 & Right superior parietal lobule (SPG.R)\\
61 & Left inferior parietal lobule (IPL.L) & 62 &  Right inferior parietal lobule (IPL.R)\\
63 & Left supramarginal gyrus (SMG.L) & 64 & Right  supramarginal gyrus (SMG.R)\\
65 & Left angular gyrus (ANG.L) & 66 & Right angular gyrus (ANG.R)\\
67 & Left precuneus (PCUN.L) & 68 & Right precuneus (PCUN.R)\\
69 & Left paracentral lobule (PCL.L) & 70 & Right paracentral lobule (PCL.R)\\
71 & Left caudate nucleus (CAU.L) & 72 & Right caudate nucleus (CAU.R)\\
73 & Left putamen (PUT.L) & 74 & Right putamen (PUT.R)\\
75 & Left globus pallidus (PAL.L) & 76 & Right globus pallidus (PAL.R)\\
77 & Left thalamus (THA.L) & 78 & Right thalamus (THA.R)\\
79 & Left transverse temporal gyrus (HES.L) & 80 & Right transverse temporal gyrus (HES.R)\\
81 & Left superior temporal gyrus (STG.L) & 82 & Right superior temporal gyrus (STG.R)\\
83 & Left superior temporal pole (TPOsup.L) & 84 & Right superior temporal pole (TPOsup.R)\\
85 & Left middle temporal gyrus (MTG.L) & 86 & Right right middle temporal gyrus (MTG.R)\\
87 & Left middle temporal pole (TPOmid.L) & 88 & Right middle temporal pole (TPOmid.R)\\
89 & Left inferior temporal gyrus (ITG.L) & 90 & Right inferior temporal gyrus (ITG.R)\\
\hline
\end{tabular}
\label{tab:roi_name}
\end{table}

In section \ref{subsec: real world brain network}, we conduct experiments to estimate the brain networks for the control and autism groups using fMRI data (obtained under resting-state conditions) from the Autism Brain Imaging Data Exchange (ABIDE) dataset. We observed that the estimate for the control group brain network exhibits greater connectivity than the autism group (see section \ref{subsec: real world brain network}) for more details. The connections between the the ROI's that are specific to \textit{only} the \textit{control} group identified in our experiment in Section \ref{subsec: real world brain network} are the following: MFG(L) $\leftrightarrow$ MFG(R), ROL(R) $\leftrightarrow$ HES(R), HIP(R) $\leftrightarrow$ PHG(R), LING(L) $\leftrightarrow$ CAL(L), MOG(R) $\leftrightarrow$ SOG(L), IOG(R) $\leftrightarrow$ MOG(R), PoCG(L) $\leftrightarrow$ PoCG(R), IPL(L) $\leftrightarrow$ SPG(L), PCUN(L) $\leftrightarrow$ SPG(L), PUT(R) $\leftrightarrow$ PAL(L), STG(L) $\leftrightarrow$ HES(L), STG(R) $\leftrightarrow$ STG(L), MTG(L) $\leftrightarrow$ MTG(R), PreCG(L) $\leftrightarrow$ IFGoperc(L), ORBinf(L) $\leftrightarrow$ ORBinf(R), PCG(L) $\leftrightarrow$ PCG(R). The abbreviations for the above ROI's can be found in \cite{feng2015dynamic} or Table~\ref{tab:roi_name}. In Table \ref{table: Control specific connections}, we validate several connections identified in our experiment that are specific to the control group. These connections (see Appendix \ref{app:roi} for the complete list), absent in the autism group, are associated with cognitive functions such as social interaction, face and image recognition, learning, and working memory. Thus our algorithm effectively extracts well-verified ground truths distinguishing the control and autism groups (see references in Table~\ref{table: Control specific connections} and \cite{yoo2024whole}).

\begin{table*}[t]
    \centering
        \caption{Estimated neural connections specific only to the control group and their functionalities with respect to autism spectrum disorder. The listed connections, absent in the autism group, are validated by prior studies, highlighting their role in functions such as social interaction, language comprehension, and memory.}
    \renewcommand{\arraystretch}{1.5} % Increases the space between rows
    \begin{tabular}{|r|p{6cm}|p{6cm}|}
    \hline
    & \multicolumn{1}{c|}{\textbf{Estimated neural connections specific to control group}} & \multicolumn{1}{c|}{\textbf{Cognitive role \& supporting literature}} \\
    \hline
    1. & Superior temporal gyrus (STG.L $\leftrightarrow$ STG.R).
    & Left STG is critical for speech perception and language comprehension while right STG is important for interpreting speech's emotional tone and intonation. Reduced connectivity between these regions impairs language comprehension, auditory processing, and ability to process prosody \cite{bigler2007superior}. \\
    \hline
    2. & Middle temporal gyrus (MTG.L $\leftrightarrow$ MTG.R)  & Left MTG is involved in the comprehension of semantics (context, sentence meaning) while right MTG is critical for interpreting social and emotional facial cues and is implicated in the theory of mind processing \cite{cheng2015autism}. Reduced MTG connectivity affects understanding of conversational context giving rise to challenges in social communication. \\
    \hline
    3. & Middle frontal gyrus (MFG.L $\leftrightarrow$ MFG.R) & Connection between MFG.L and MFG.R plays a key role in higher-order cognitive functions such as working memory, executive control, and decision-making \cite{khandan2023altered}.  Disrupted connectivity contributes to difficulties in cognitive flexibility and task execution.\\
    \hline
    4. & Right hippocampus (HIP.R) $\leftrightarrow$ Right parahippocampal gyrus (PHG.R) & The right hippocampus is involved in memory formation, spatial navigation, and retrieving autobiographical memory \cite{godwin2015breakdown}. The parahippocampal gyrus supports contextual and spatial memory, linking visual and spatial information with memory processing \cite{courchesne2005frontal}. Deficits in connectivity between HIP.R and PHG.R contribute to memory impairments, affecting spatial awareness and navigation, which are often observed in autism spectrum disorder \cite{godwin2015breakdown}. 
\\
    \hline
    \end{tabular}
    \label{table: Control specific connections}
\end{table*}

{\small \putbib[references_appendix]} 
\end{bibunit}


\begin{thebibliography}{10}
\providecommand{\url}[1]{#1}
\csname url@samestyle\endcsname
\providecommand{\newblock}{\relax}
\providecommand{\bibinfo}[2]{#2}
\providecommand{\BIBentrySTDinterwordspacing}{\spaceskip=0pt\relax}
\providecommand{\BIBentryALTinterwordstretchfactor}{4}
\providecommand{\BIBentryALTinterwordspacing}{\spaceskip=\fontdimen2\font plus
\BIBentryALTinterwordstretchfactor\fontdimen3\font minus \fontdimen4\font\relax}
\providecommand{\BIBforeignlanguage}[2]{{%
\expandafter\ifx\csname l@#1\endcsname\relax
\typeout{** WARNING: IEEEtran.bst: No hyphenation pattern has been}%
\typeout{** loaded for the language `#1'. Using the pattern for}%
\typeout{** the default language instead.}%
\else
\language=\csname l@#1\endcsname
\fi
#2}}
\providecommand{\BIBdecl}{\relax}
\BIBdecl

\bibitem{strogatz2001exploring}
S.~H. Strogatz, ``Exploring complex networks,'' \emph{nature}, vol. 410, no. 6825, pp. 268--276, 2001.

\bibitem{boccaletti2006complex}
S.~Boccaletti, V.~Latora, Y.~Moreno, M.~Chavez, and D.~Hwang, ``Complex networks: Structure and dynamics,'' \emph{Physics reports}, vol. 424, no. 4-5, pp. 175--308, 2006.

\bibitem{van2017modeling}
A.~van~der Schaft, ``Modeling of physical network systems,'' \emph{Systems \& Control Letters}, vol. 101, pp. 21--27, 2017.

\bibitem{bressan2014flows}
A.~Bressan, S.~{\v{C}}ani{\'c}, M.~Garavello, M.~Herty, and B.~Piccoli, ``Flows on networks: recent results and perspectives,'' \emph{EMS Surveys in Mathematical Sciences}, vol.~1, pp. 47--111, 2014.

\bibitem{voss2014searching}
H.~U. Voss and N.~D. Schiff, ``Searching for conservation laws in brain dynamics—bold flux and source imaging,'' \emph{Entropy}, vol.~16, no.~7, pp. 3689--3709, 2014.

\bibitem{podobnik2017biological}
B.~Podobnik, M.~Jusup, Z.~Tiganj, W.-X. Wang, J.~M. Buld{\'u}, and H.~E. Stanley, ``Biological conservation law as an emerging functionality in dynamical neuronal networks,'' \emph{Proceedings of the National Academy of Sciences}, vol. 114, no.~45, pp. 11,826--11,831, 2017.

\bibitem{chung1997spectral}
F.~R. Chung and F.~C. Graham, \emph{Spectral graph theory}.\hskip 1em plus 0.5em minus 0.4em\relax American Mathematical Soc., 1997, no.~92.

\bibitem{shafipour2017network}
R.~Shafipour, S.~Segarra, A.~G. Marques, and G.~Mateos, ``Network topology inference from non-stationary graph signals,'' in \emph{2017 IEEE International Conference on Acoustics, Speech and Signal Processing (ICASSP)}.\hskip 1em plus 0.5em minus 0.4em\relax IEEE, 2017, pp. 5870--5874.

\bibitem{rayas_learning_2022}
A.~Rayas, R.~Anguluri, and G.~Dasarathy, ``Learning the {Structure} of {Large} {Networked} {Systems} {Obeying} {Conservation} {Laws},'' in \emph{Advances in {Neural} {Information} {Processing} {Systems}}, vol.~35, 2022, pp. 14,637--14,650.

\bibitem{DekaTSG2020}
D.~Deka, S.~Talukdar, M.~Chertkov, and M.~V. Salapaka, ``Graphical models in meshed distribution grids: Topology estimation, change detection $\&$ limitations,'' \emph{IEEE Transactions on Smart Grid}, vol.~11, no.~5, pp. 4299--4310, 2020.

\bibitem{anguluri2021grid}
R.~Anguluri, G.~Dasarathy, O.~Kosut, and L.~Sankar, ``Grid topology identification with hidden nodes via structured norm minimization,'' \emph{IEEE Control Systems Letters}, vol.~6, pp. 1244--1249, 2021.

\bibitem{segarra2018network}
S.~Segarra, A.~G. Marques, M.~Goyal, and S.~Rey-Escudero, ``Network topology inference from input-output diffusion pairs,'' in \emph{2018 IEEE Statistical Signal Processing Workshop (SSP)}.\hskip 1em plus 0.5em minus 0.4em\relax IEEE, 2018, pp. 508--512.

\bibitem{park2021learning}
G.~Park, S.~J. Moon, S.~Park, and J.-J. Jeon, ``Learning a high-dimensional linear structural equation model via l1-regularized regression,'' \emph{Journal of Machine Learning Research}, vol.~22, no. 102, pp. 1--41, 2021.

\bibitem{pruttiakaravanich2020convex}
A.~Pruttiakaravanich and J.~Songsiri, ``Convex formulation for regularized estimation of structural equation models,'' \emph{Signal Processing}, vol. 166, 2020.

\bibitem{doddi2022efficient}
H.~Doddi, D.~Deka, S.~Talukdar, and M.~Salapaka, ``Efficient and passive learning of networked dynamical systems driven by non-white exogenous inputs,'' in \emph{International Conference on Artificial Intelligence and Statistics}.\hskip 1em plus 0.5em minus 0.4em\relax PMLR, 2022, pp. 9982--9997.

\bibitem{shafipour2021identifying}
R.~Shafipour, S.~Segarra, A.~G. Marques, and G.~Mateos, ``Identifying the topology of undirected networks from diffused non-stationary graph signals,'' \emph{IEEE Open Journal of Signal Processing}, vol.~2, pp. 171--189, 2021.

\bibitem{mateos2019connecting}
G.~Mateos, S.~Segarra, A.~G. Marques, and A.~Ribeiro, ``Connecting the dots: Identifying network structure via graph signal processing,'' \emph{IEEE Signal Processing Magazine}, vol.~36, no.~3, pp. 16--43, 2019.

\bibitem{mcpherson2001birds}
M.~McPherson, L.~Smith-Lovin, and J.~M. Cook, ``Birds of a feather: Homophily in social networks,'' \emph{Annual review of sociology}, vol.~27, no.~1, pp. 415--444, 2001.

\bibitem{shen2017tensor}
Y.~Shen, B.~Baingana, and G.~B. Giannakis, ``Tensor decompositions for identifying directed graph topologies and tracking dynamic networks,'' \emph{IEEE Transactions on Signal Processing}, vol.~65, no.~14, pp. 3675--3687, 2017.

\bibitem{deb2024regularized}
N.~Deb, A.~Kuceyeski, and S.~Basu, ``Regularized estimation of sparse spectral precision matrices,'' \emph{arXiv preprint arXiv:2401.11128}, 2024.

\bibitem{dallakyan2022time}
A.~Dallakyan, R.~Kim, and M.~Pourahmadi, ``Time series graphical {Lasso} and sparse {VAR} estimation,'' \emph{Computational Statistics \& Data Analysis}, vol. 176, 2022.

\bibitem{basu2015regularized}
S.~Basu and G.~Michailidis, ``Regularized estimation in sparse high-dimensional time series models,'' \emph{Annals of Statistics}, pp. 1535--1567, 2015.

\bibitem{doddi2021learning}
H.~Doddi, D.~Deka, S.~Talukdar, and M.~V. Salapaka, ``Learning networked linear dynamical systems under non-white excitation from a single trajectory.'' \emph{CoRR}, 2021.

\bibitem{doddi2019exact}
H.~Doddi, S.~Talukdar, D.~Deka, and M.~Salapaka, ``Exact topology learning in a network of cyclostationary processes,'' in \emph{2019 American Control Conference (ACC)}.\hskip 1em plus 0.5em minus 0.4em\relax IEEE, 2019, pp. 4968--4973.

\bibitem{ranciati2021fused}
S.~Ranciati, A.~Roverato, and A.~Luati, ``Fused graphical {Lasso} for brain networks with symmetries,'' \emph{Journal of the Royal Statistical Society Series C: Applied Statistics}, vol.~70, no.~5, pp. 1299--1322, 2021.

\bibitem{monti2014estimating}
R.~P. Monti, P.~Hellyer, D.~Sharp, R.~Leech, C.~Anagnostopoulos, and G.~Montana, ``Estimating time-varying brain connectivity networks from functional {MRI} time series,'' \emph{NeuroImage}, vol. 103, pp. 427--443, 2014.

\bibitem{wainwright2009sharp}
M.~J. Wainwright, ``Sharp thresholds for high-dimensional and noisy sparsity recovery using $\ell_{1}$-constrained quadratic programming ({LASSO}),'' \emph{IEEE transactions on information theory}, vol.~55, no.~5, pp. 2183--2202, 2009.

\bibitem{van2008high}
S.~A. Van~de Geer, ``High-dimensional generalized linear models and the {Lasso},'' \emph{The Annals of Statistics}, vol.~36, no.~2, pp. 614--645, 2008.

\bibitem{drton2017structure}
M.~Drton and M.~H. Maathuis, ``Structure learning in graphical modeling,'' \emph{Annual Review of Statistics and Its Application}, vol.~4, no. Volume 4, 2017, pp. 365--393, 2017.

\bibitem{yuan2007model}
M.~Yuan and Y.~Lin, ``Model selection and estimation in the {Gaussian} graphical model,'' \emph{Biometrika}, vol.~94, no.~1, pp. 19--35, 2007.

\bibitem{ravikumar2011high}
P.~Ravikumar, M.~J. Wainwright, G.~Raskutti, and B.~Yu, ``High-dimensional covariance estimation by minimizing $\ell_{1}$-penalized log-determinant divergence,'' \emph{Electronic Journal of Statistics}, vol.~5, pp. 935--980, 2011.

\bibitem{dallakyan2023fused}
A.~Dallakyan and M.~Pourahmadi, ``Fused-{Lasso} regularized {Cholesky} factors of large nonstationary covariance matrices of replicated time series,'' \emph{Journal of Computational and Graphical Statistics}, vol.~32, no.~1, pp. 157--170, 2023.

\bibitem{chang2010estimation}
C.~Chang and R.~S. Tsay, ``Estimation of covariance matrix via the sparse {Cholesky} factor with {Lasso},'' \emph{Journal of Statistical Planning and Inference}, vol. 140, no.~12, pp. 3858--3873, 2010.

\bibitem{tsai2022joint}
K.~Tsai, O.~Koyejo, and M.~Kolar, ``Joint gaussian graphical model estimation: A survey,'' \emph{Wiley Interdisciplinary Reviews: Computational Statistics}, vol.~14, no.~6, 2022.

\bibitem{chen2024estimation}
L.~Chen, ``Estimation of graphical models: An overview of selected topics,'' \emph{International Statistical Review}, vol.~92, no.~2, pp. 194--245, 2024.

\bibitem{ying2020does}
J.~Ying, J.~V. de~Miranda~Cardoso, and D.~P. Palomar, ``Does the $\ell_1$ norm learn a sparse graph under {Laplacian} constrained graphical models?'' \emph{arXiv preprint arXiv:2006.14925}, 2020.

\bibitem{kumar2019structured}
S.~Kumar, J.~Ying, J.~V. de~Miranda~Cardoso, and D.~Palomar, ``Structured graph learning via laplacian spectral constraints,'' \emph{Advances in neural information processing systems}, vol.~32, 2019.

\bibitem{ying2023network}
J.~Ying, X.~Han, R.~Zhou, X.~Wang, and H.~C. So, ``Network topology inference with sparsity and laplacian constraints,'' in \emph{2023 IEEE 11th International Conference on Information, Communication and Networks (ICICN)}.\hskip 1em plus 0.5em minus 0.4em\relax IEEE, 2023, pp. 283--288.

\bibitem{dahlhaus2000graphical}
R.~Dahlhaus, ``Graphical interaction models for multivariate time series,'' \emph{Metrika}, vol.~51, pp. 157--172, 2000.

\bibitem{baek2021local}
C.~Baek, {M.C. D{\"u}ker}, and V.~Pipiras, ``Local {W}hittle estimation of high-dimensional long-run variance and precision matrices,'' \emph{arXiv preprint arXiv:2105.13342}, 2021.

\bibitem{zorzi2015ar}
M.~Zorzi and R.~Sepulchre, ``Ar identification of latent-variable graphical models,'' \emph{IEEE Transactions on Automatic Control}, vol.~61, no.~9, pp. 2327--2340, 2015.

\bibitem{zorzi2019empirical}
M.~Zorzi, ``Empirical bayesian learning in ar graphical models,'' \emph{Automatica}, vol. 109, p. 108516, 2019.

\bibitem{crescente2020learning}
F.~Crescente, L.~Falconi, F.~Rozzi, A.~Ferrante, and M.~Zorzi, ``Learning ar factor models,'' in \emph{2020 59th IEEE Conference on Decision and Control (CDC)}.\hskip 1em plus 0.5em minus 0.4em\relax IEEE, 2020, pp. 274--279.

\bibitem{falconi2023robust}
L.~Falconi, A.~Ferrante, and M.~Zorzi, ``A robust approach to arma factor modeling,'' \emph{IEEE Transactions on Automatic Control}, vol.~69, no.~2, pp. 828--841, 2023.

\bibitem{zorzi2024identification}
M.~Zorzi, ``On the identification of arma graphical models,'' \emph{IEEE Transactions on Automatic Control}, 2024.

\bibitem{alpago2018identification}
D.~Alpago, M.~Zorzi, and A.~Ferrante, ``Identification of sparse reciprocal graphical models,'' \emph{IEEE Control Systems Letters}, vol.~2, no.~4, pp. 659--664, 2018.

\bibitem{DekaTCNS2018}
D.~Deka, S.~Backhaus, and M.~Chertkov, ``Structure learning in power distribution networks,'' \emph{IEEE Transactions on Control of Network Systems}, vol.~5, no.~3, pp. 1061--1074, 2018.

\bibitem{deka2020graphical}
D.~Deka, S.~Talukdar, M.~Chertkov, and M.~V. Salapaka, ``Graphical models in meshed distribution grids: Topology estimation, change detection \& limitations,'' \emph{IEEE Transactions on Smart Grid}, vol.~11, no.~5, pp. 4299--4310, 2020.

\bibitem{deka2023learning}
D.~Deka, V.~Kekatos, and G.~Cavraro, ``Learning distribution grid topologies: A tutorial,'' \emph{IEEE Transactions on Smart Grid}, vol.~15, no.~1, pp. 999--1013, 2023.

\bibitem{grotas2019power}
S.~Grotas, Y.~Yakoby, I.~Gera, and T.~Routtenberg, ``Power systems topology and state estimation by graph blind source separation,'' \emph{IEEE Transactions on Signal Processing}, vol.~67, no.~8, pp. 2036--2051, 2019.

\bibitem{whittle1953estimation}
P.~Whittle, ``Estimation and information in stationary time series,'' \emph{Arkiv f{\"o}r matematik}, vol.~2, no.~5, pp. 423--434, 1953.

\bibitem{brockwell2009time}
P.~J. Brockwell and R.~A. Davis, \emph{Time series: theory and methods}.\hskip 1em plus 0.5em minus 0.4em\relax Springer science \& business media, 2009.

\bibitem{dhillon2008matrix}
I.~S. Dhillon and J.~A. Tropp, ``Matrix nearness problems with {Bregman} divergences,'' \emph{SIAM Journal on Matrix Analysis and Applications}, vol.~29, no.~4, pp. 1120--1146, 2008.

\bibitem{dorfler2012kron}
F.~Dorfler and F.~Bullo, ``Kron reduction of graphs with applications to electrical networks,'' \emph{IEEE Transactions on Circuits and Systems I: Regular Papers}, vol.~60, no.~1, pp. 150--163, 2012.

\bibitem{sun2018large}
Y.~Sun, Y.~Li, A.~Kuceyeski, and S.~Basu, ``Large spectral density matrix estimation by thresholding,'' \emph{arXiv preprint arXiv:1812.00532}, 2018.

\bibitem{fiecas2019spectral}
M.~Fiecas, C.~Leng, W.~Liu, and Y.~Yu, ``{Spectral analysis of high-dimensional time series},'' \emph{Electronic Journal of Statistics}, vol.~13, no.~2, pp. 4079 -- 4101, 2019.

\bibitem{baek2023local}
C.~Baek, M.-C. D{\"u}ker, and V.~Pipiras, ``Local whittle estimation of high-dimensional long-run variance and precision matrices,'' \emph{The Annals of Statistics}, vol.~51, no.~6, pp. 2386--2414, 2023.

\bibitem{basu2023graphical}
S.~Basu and S.~Subba~Rao, ``Graphical models for nonstationary time series,'' \emph{The Annals of Statistics}, vol.~51, no.~4, pp. 1453--1483, 2023.

\bibitem{krampe2025frequency}
J.~Krampe and E.~Paparoditis, ``Frequency domain statistical inference for high-dimensional time series,'' \emph{Journal of the American Statistical Association}, no. just-accepted, pp. 1--22, 2025.

\bibitem{hurvich2002whittle}
C.~Hurvich, ``Whittle’s approximation to the likelihood function,'' \emph{Lecture Notes (New York University Stern School of Business, 2002)}, 2002.

\bibitem{jung2015graphical}
A.~Jung, G.~Hannak, and N.~Goertz, ``Graphical {Lasso} based model selection for time series,'' \emph{IEEE Signal Processing Letters}, vol.~22, no.~10, pp. 1781--1785, 2015.

\bibitem{baek2021thresholding}
C.~Baek, M.-C. D{\"u}ker, and V.~Pipiras, ``Thresholding and graphical local whittle estimation,'' \emph{arXiv preprint arXiv:2105.13342}, 2021.

\bibitem{bunea2007sparsity}
F.~Bunea, A.~Tsybakov, and M.~Wegkamp, ``{Sparsity oracle inequalities for the Lasso},'' \emph{Electronic Journal of Statistics}, vol.~1, pp. 169 -- 194, 2007.

\bibitem{bickel2009simultaneous}
P.~J. Bickel, Y.~Ritov, and A.~B. Tsybakov, ``{Simultaneous analysis of Lasso and Dantzig selector},'' \emph{The Annals of Statistics}, vol.~37, no.~4, pp. 1705 -- 1732, 2009.

\bibitem{negahban2012unified}
S.~N. Negahban, P.~Ravikumar, M.~J. Wainwright, and B.~Yu, ``{A Unified Framework for High-Dimensional Analysis of $M$-Estimators with Decomposable Regularizers},'' \emph{Statistical Science}, vol.~27, no.~4, pp. 538 -- 557, 2012.

\bibitem{hastie2009elements}
T.~Hastie, R.~Tibshirani, J.~H. Friedman, and J.~H. Friedman, \emph{The elements of statistical learning: data mining, inference, and prediction}.\hskip 1em plus 0.5em minus 0.4em\relax Springer, 2009, vol.~2.

\bibitem{dai2018broken}
L.~Dai, K.~Chen, Z.~Sun, Z.~Liu, and G.~Li, ``Broken adaptive ridge regression and its asymptotic properties,'' \emph{Journal of multivariate analysis}, vol. 168, pp. 334--351, 2018.

\bibitem{fan2001variable}
J.~Fan and R.~Li, ``Variable selection via nonconcave penalized likelihood and its oracle properties,'' \emph{Journal of the American statistical Association}, vol.~96, no. 456, pp. 1348--1360, 2001.

\bibitem{loh2017support}
P.-L. Loh and M.~J. Wainwright, ``{Support recovery without incoherence: A case for nonconvex regularization},'' \emph{The Annals of Statistics}, vol.~45, no.~6, pp. 2455 -- 2482, 2017.

\bibitem{zhao2006model}
P.~Zhao and B.~Yu, ``On model selection consistency of {Lasso},'' \emph{The Journal of Machine Learning Research}, vol.~7, pp. 2541--2563, 2006.

\bibitem{cai2011constrained}
T.~Cai, W.~Liu, and X.~Luo, ``A constrained $\ell_1$-minimization approach to sparse precision matrix estimation,'' \emph{Journal of the American Statistical Association}, vol. 106, no. 494, pp. 594--607, 2011.

\bibitem{rothman2008sparse}
A.~J. Rothman, P.~J. Bickel, E.~Levina, and J.~Zhu, ``Sparse permutation invariant covariance estimation,'' \emph{Electronic Journal of Statistics}, vol.~2, pp. 494--515, 2008.

\bibitem{rao2021reconciling}
S.~S. Rao and J.~Yang, ``Reconciling the gaussian and whittle likelihood with an application to estimation in the frequency domain,'' \emph{The Annals of Statistics}, vol.~49, no.~5, pp. 2774--2802, 2021.

\bibitem{rosenblatt2012stationary}
M.~Rosenblatt, \emph{Stationary sequences and random fields}.\hskip 1em plus 0.5em minus 0.4em\relax Springer Science \& Business Media, 2012.

\bibitem{kellogg}
R.~B. Kellogg, {T.Y. Li}, and J.~Yorke, ``A constructive proof of the {Brouwer} fixed-point theorem and computational results,'' \emph{SIAM Journal on Numerical Analysis}, vol.~13, no.~4, pp. 473--483, 1976.

\bibitem{narasimhan2019learning}
{S. Jayadev}, {S. Narasimhan}, and {N. Bhatt}, ``Learning conserved networks from flows,'' \emph{arXiv preprint arXiv:1905.08716}, 2019.

\bibitem{chen2008extended}
J.~Chen and Z.~Chen, ``Extended {Bayesian} information criteria for model selection with large model spaces,'' \emph{Biometrika}, vol.~95, no.~3, pp. 759--771, 2008.

\bibitem{hernadez2016water}
E.~Hernadez, S.~Hoagland, and L.~Ormsbee, ``Water distribution database for research applications,'' in \emph{World Environmental and Water Resources Congress}, 2016, pp. 465--474.

\bibitem{seccamonte2023bilevel}
F.~Seccamonte, ``Bilevel optimization in learning and control with applications to network flow estimation,'' Ph.D. dissertation, UC Santa Barbara, 2023.

\bibitem{vskoch2022human}
A.~{\v{S}}koch, B.~Reh{\'a}k~Bu{\v{c}}kov{\'a}, J.~Mare{\v{s}}, J.~Tint{\v{e}}ra, P.~Sanda, L.~Jajcay, J.~Hor{\'a}{\v{c}}ek, F.~{\v{S}}paniel, and J.~Hlinka, ``Human brain structural connectivity matrices--ready for modelling,'' \emph{Scientific Data}, vol.~9, no.~1, 2022.

\bibitem{hagmann2008mapping}
P.~Hagmann, L.~Cammoun, X.~Gigandet, R.~Meuli, C.~J. Honey, V.~J. Wedeen, and O.~Sporns, ``Mapping the structural core of human cerebral cortex,'' \emph{PLoS biology}, vol.~6, no.~7, 2008.

\bibitem{bassett2006small}
D.~S. Bassett and E.~Bullmore, ``Small-world brain networks,'' \emph{The neuroscientist}, vol.~12, no.~6, pp. 512--523, 2006.

\bibitem{pourahmadi2011covariance}
M.~Pourahmadi, ``{Covariance Estimation: The GLM and Regularization Perspectives},'' \emph{Statistical Science}, vol.~26, pp. 369 -- 387, 2011.

\bibitem{trendafilov2021multivariate}
N.~Trendafilov and M.~Gallo, \emph{Multivariate data analysis on matrix manifolds}.\hskip 1em plus 0.5em minus 0.4em\relax Springer, 2021.

\bibitem{ying2021minimax}
J.~Ying, J.~V. de~Miranda~Cardoso, and D.~Palomar, ``Minimax estimation of {Laplacian} constrained precision matrices,'' in \emph{International Conference on Artificial Intelligence and Statistics}.\hskip 1em plus 0.5em minus 0.4em\relax PMLR, 2021, pp. 3736--3744.

\bibitem{kumar2020unified}
S.~Kumar, J.~Ying, J.~V. d.~M. Cardoso, and D.~P. Palomar, ``A unified framework for structured graph learning via spectral constraints,'' \emph{Journal of Machine Learning Research}, vol.~21, no.~22, pp. 1--60, 2020.

\bibitem{preti2017dynamic}
M.~G. Preti, T.~A. Bolton, and D.~Van De~Ville, ``The dynamic functional connectome: State-of-the-art and perspectives,'' \emph{Neuroimage}, vol. 160, pp. 41--54, 2017.

\bibitem{peng2024financial}
C.~Peng and C.~Simon, ``Financial modeling with geometric brownian motion,'' \emph{Open Journal of Business and Management}, vol.~12, no.~2, pp. 1240--1250, 2024.

\bibitem{engelke2024evy}
S.~Engelke, J.~Ivanovs, and J.~D. Th{\o}stesen, ``L\'evy graphical models,'' \emph{arXiv preprint arXiv:2410.19952}, 2024.

\end{thebibliography}


\begin{thebibliography}{10}
\providecommand{\url}[1]{#1}
\csname url@samestyle\endcsname
\providecommand{\newblock}{\relax}
\providecommand{\bibinfo}[2]{#2}
\providecommand{\BIBentrySTDinterwordspacing}{\spaceskip=0pt\relax}
\providecommand{\BIBentryALTinterwordstretchfactor}{4}
\providecommand{\BIBentryALTinterwordspacing}{\spaceskip=\fontdimen2\font plus
\BIBentryALTinterwordstretchfactor\fontdimen3\font minus \fontdimen4\font\relax}
\providecommand{\BIBforeignlanguage}[2]{{%
\expandafter\ifx\csname l@#1\endcsname\relax
\typeout{** WARNING: IEEEtran.bst: No hyphenation pattern has been}%
\typeout{** loaded for the language `#1'. Using the pattern for}%
\typeout{** the default language instead.}%
\else
\language=\csname l@#1\endcsname
\fi
#2}}
\providecommand{\BIBdecl}{\relax}
\BIBdecl

\bibitem{ravikumar2011high_app}
P.~Ravikumar, M.~J. Wainwright, G.~Raskutti, and B.~Yu, ``High-dimensional covariance estimation by minimizing $\ell_{1}$-penalized log-determinant divergence,'' \emph{Electronic Journal of Statistics}, vol.~5, pp. 935--980, 2011.

\bibitem{deb2024regularized_app}
N.~Deb, A.~Kuceyeski, and S.~Basu, ``Regularized estimation of sparse spectral precision matrices,'' \emph{arXiv preprint arXiv:2401.11128}, 2024.

\bibitem{cai2011constrained_app}
T.~Cai, W.~Liu, and X.~Luo, ``A constrained $\ell_1$-minimization approach to sparse precision matrix estimation,'' \emph{Journal of the American Statistical Association}, vol. 106, no. 494, pp. 594--607, 2011.

\bibitem{rothman2008sparse_app}
A.~J. Rothman, P.~J. Bickel, E.~Levina, and J.~Zhu, ``Sparse permutation invariant covariance estimation,'' \emph{Electronic Journal of Statistics}, vol.~2, pp. 494--515, 2008.

\bibitem{rayas_learning_2022_app}
A.~Rayas, R.~Anguluri, and G.~Dasarathy, ``Learning the {Structure} of {Large} {Networked} {Systems} {Obeying} {Conservation} {Laws},'' in \emph{Advances in {Neural} {Information} {Processing} {Systems}}, vol.~35, 2022, pp. 14,637--14,650.

\bibitem{boyd2004convex_app}
S.~P. Boyd and L.~Vandenberghe, \emph{Convex optimization}.\hskip 1em plus 0.5em minus 0.4em\relax Cambridge university press, 2004.

\bibitem{bauschke2011convex}
H.~H. Bauschke, P.~L. Combettes \emph{et~al.}, \emph{Convex analysis and monotone operator theory in Hilbert spaces}.\hskip 1em plus 0.5em minus 0.4em\relax Springer, 2011, vol. 408.

\bibitem{petersen2008matrix}
K.~B. Petersen, M.~S. Pedersen \emph{et~al.}, ``The matrix cookbook,'' \emph{Technical University of Denmark}, vol.~7, no.~15, 2008.

\bibitem{laub2005matrix}
A.~J. Laub, \emph{Matrix analysis for scientists and engineers}.\hskip 1em plus 0.5em minus 0.4em\relax Siam, 2005, vol.~91.

\bibitem{horn2012matrix}
R.~A. Horn and C.~R. Johnson, \emph{Matrix analysis}.\hskip 1em plus 0.5em minus 0.4em\relax Cambridge university press, 2012.

\bibitem{rosenblatt2012stationary_app}
M.~Rosenblatt, \emph{Stationary sequences and random fields}.\hskip 1em plus 0.5em minus 0.4em\relax Springer Science \& Business Media, 2012.

\bibitem{sun2018large_app}
Y.~Sun, Y.~Li, A.~Kuceyeski, and S.~Basu, ``Large spectral density matrix estimation by thresholding,'' \emph{arXiv preprint arXiv:1812.00532}, 2018.

\bibitem{feng2015dynamic}
G.~Feng, {H.C. Chen}, Z.~Zhu, Y.~He, and S.~Wang, ``Dynamic brain architectures in local brain activity and functional network efficiency associate with efficient reading in bilinguals,'' \emph{Neuroimage}, vol. 119, pp. 103--118, 2015.

\bibitem{yoo2024whole}
S.~Yoo, Y.~Jang, S.~Hong, H.~Park, S.~L. Valk, B.~C. Bernhardt, and B.~Park, ``Whole-brain structural connectome asymmetry in autism,'' \emph{NeuroImage}, vol. 288, p. 120534, 2024.

\bibitem{bigler2007superior}
E.~D. Bigler, S.~Mortensen, E.~S. Neeley, S.~Ozonoff, L.~Krasny, M.~Johnson, J.~Lu, S.~L. Provencal, W.~McMahon, and J.~E. Lainhart, ``Superior temporal gyrus, language function, and autism,'' \emph{Developmental neuropsychology}, vol.~31, no.~2, pp. 217--238, 2007.

\bibitem{cheng2015autism}
W.~Cheng, E.~T. Rolls, H.~Gu, J.~Zhang, and J.~Feng, ``Autism: reduced connectivity between cortical areas involved in face expression, theory of mind, and the sense of self,'' \emph{Brain}, vol. 138, no.~5, pp. 1382--1393, 2015.

\bibitem{khandan2023altered}
Z.~Khandan Khadem-Reza, M.~A. Shahram, and H.~Zare, ``Altered resting-state functional connectivity of the brain in children with autism spectrum disorder,'' \emph{Radiological Physics and Technology}, vol.~16, no.~2, pp. 284--291, 2023.

\bibitem{godwin2015breakdown}
D.~Godwin, R.~L. Barry, and R.~Marois, ``Breakdown of the brain’s functional network modularity with awareness,'' \emph{Proceedings of the National Academy of Sciences}, vol. 112, no.~12, pp. 3799--3804, 2015.

\bibitem{courchesne2005frontal}
E.~Courchesne and K.~Pierce, ``Why the frontal cortex in autism might be talking only to itself: local over-connectivity but long-distance disconnection,'' \emph{Current opinion in neurobiology}, vol.~15, no.~2, pp. 225--230, 2005.

\end{thebibliography}
\end{document}